\documentclass{article}

\PassOptionsToPackage{numbers, compress, sort}{natbib}
 \usepackage[main, final]{neurips_2026}
\usepackage[utf8]{inputenc} % allow utf-8 input
\usepackage[T1]{fontenc}    % use 8-bit T1 fonts
\usepackage{hyperref}       % hyperlinks
\usepackage{url}            % simple URL typesetting
\usepackage{booktabs}       % professional-quality tables
\usepackage{amsfonts}       % blackboard math symbols
\usepackage{nicefrac}       % compact symbols for 1/2, etc.
\usepackage{microtype}      % microtypography
\usepackage{xcolor}         % colors

\usepackage{amsmath}
\usepackage{amssymb}
\usepackage{mathtools}
\usepackage{amsthm}
\usepackage{subcaption}
\usepackage{graphicx}

\usepackage{xcolor}         % colors
\usepackage{wrapfig}
\usepackage{enumitem}
\usepackage{natbib}
\usepackage{algorithm}
\usepackage{algpseudocode}
\usepackage{multirow}
\usepackage{soul}    
\usepackage{kotex}
\usepackage{rotating}

\providecommand{\OURS}{HiSNOT~}

\usepackage{pifont}
\newcommand{\cmark}{\textcolor{green!60!black}{\ding{51}}}
\newcommand{\xmark}{\textcolor{red!70!black}{\ding{55}}}

\usepackage{amsmath}\allowdisplaybreaks
\usepackage{amsfonts,bm}
\usepackage{amssymb}
\usepackage{amsthm}
\usepackage{amsmath}
\usepackage{algorithm}
\usepackage{multirow}
\usepackage{booktabs}
\usepackage{xcolor}

\def\eqref#1{equation~(\ref{#1})}
\def\1{\bf{1}}

\DeclareMathOperator*{\argmax}{arg\,max}
\DeclareMathOperator*{\argmin}{arg\,min}

\makeatletter
\def\Ddots{\mathinner{\mkern1mu\raise\p@
\vbox{\kern7\p@\hbox{.}}\mkern2mu
\raise4\p@\hbox{.}\mkern2mu\raise7\p@\hbox{.}\mkern1mu}}
\makeatother

\makeatletter
\newcommand*{\rom}[1]{\expandafter\@slowromancap\romannumeral #1@}
\makeatother

\theoremstyle{plain}
\newtheorem{theorem}{Theorem}[section]
\newtheorem{proposition}[theorem]{Proposition}
\newtheorem{lemma}[theorem]{Lemma}
\newtheorem{corollary}[theorem]{Corollary}
\theoremstyle{definition}
\newtheorem{definition}[theorem]{Definition}

\theoremstyle{remark}

\title{Neural Optimal Transport in Hilbert Spaces}

\author{%
  Jae-Hwan Choi \\
  The Catholic University of Korea \\
  \texttt{mathjhchoi@catholic.ac.kr} \\
  \And
  Jiwoo Yoon \\
  Sungkyunkwan University \\
  \texttt{yoon6019@skku.edu} \\
  \And
  Dohyun Kwon\footnotemark[1] \\
  Yonsei University / KIAS \\
  \texttt{dohyunkwon@yonsei.ac.kr} \\
  \AND
  Jaewoong Choi\thanks{Corresponding Authors: Jaewoong Choi and Dohyun Kwon.} \\
  Sungkyunkwan University \\
  \texttt{jaewoongchoi@skku.edu} \\
}

\begin{document}

\maketitle

\begin{abstract}
Functional data such as time series, trajectories, and fields are increasingly central in machine learning and are naturally modeled in infinite-dimensional Hilbert spaces.
We study Semi-dual Neural Optimal Transport (SNOT) for learning transport maps between probability measures on separable Hilbert spaces.
A critical challenge in this setting is the spurious solution problem: a neural map can globally optimize the SNOT objective while failing to recover an optimal transport map.
We show that this failure arises from non-regular source measures, which make the inner minimization in the semi-dual objective non-identifiable.
We prove that regular source measures, defined via Gaussian null sets, restore inner-minimizer uniqueness and recover the Monge map.
For singular sources, we introduce Gaussian smoothing and establish a necessary and sufficient kernel condition for smoothing to restore regularity.
We further prove plan-level consistency: as smoothing vanishes, every Wasserstein accumulation point of the induced optimal plans solves the original Kantorovich problem.
Experiments on synthetic singular transports and real-world unpaired time-series imputation benchmarks validate the theory, achieving the best performance on several benchmarks and competitive performance on the remaining ones.
\end{abstract}

\section{Introduction}

Functional observations such as PDE solutions, trajectories, time series, and random fields are naturally modeled in infinite-dimensional spaces \citep{kovachki2023neural,li2020fourier,wang2025optimal}. 
Singular or low-dimensional structure and finite discretization complicate identifiability and consistent approximation.

Neural optimal transport (OT) learns maps and potentials for generative modeling \citep{OTP,fanTMLR,otm,choi2025improving,choi2024analyzing}. Broader OT applications include time-series imputation \citep{wang2025optimal} and functional domain adaptation \citep{zhu2024functional}. 
Semi-dual Neural Optimal Transport (SNOT) uses a max--min objective, studied primarily in finite-dimensional Euclidean spaces. 
We extend it to Hilbert spaces to learn maps between distributions of functions.

We propose \emph{Hilbert Semi-dual Neural Optimal Transport} (HiSNOT) for probability measures on separable Hilbert spaces. 
We give sufficient conditions for Monge-map identifiability and develop Gaussian regularization for singular sources.

\vspace{-0.3cm}
\paragraph{Spurious solution problems.}
A central theoretical challenge is the \emph{spurious solution} problem \citep{OTP,chu2026rate,fanTMLR,otm}.
Even given an optimal Kantorovich potential, the inner minimization in the semi-dual formulation may admit global minimizers that achieve the optimal objective value yet deviate from the true Monge map. 
We interpret this phenomenon as a failure of \emph{inner-minimizer identifiability}. In infinite-dimensional Hilbert spaces, addressing this issue is challenging due to the absence of a standard Lebesgue measure. 
To resolve this, our first contribution is an identifiability theorem. 
Using the classical notion of regular measures defined via Gaussian null sets \citep{ambrosio2005gradient}, we prove that this regularity condition restores inner-minimizer identifiability, guaranteeing that the HiSNOT objective uniquely recovers the true Monge map (Theorem~\ref{thm:uniqueness}).

\vspace{-0.3cm}
\paragraph{Gaussian smoothing for singular functional data.}
Functional sources may be supported on proper low-dimensional subspaces or exact affine constraints and thus fail to be regular in the ambient Hilbert space. We study how Gaussian smoothing removes this obstruction. 
The kernel projection of the source must be regular; the converse holds under the Cameron--Martin compatibility condition in Theorem~\ref{25.12.09.14.29}, and in particular in finite-dimensional spaces. For unrestricted sources in infinite dimensions, injective trace-class covariance is a sufficient choice. 
A regular kernel marginal alone is not sufficient in general, as the counterexample in Appendix~\ref{app:kernel_counterexample} shows.

\vspace{-0.3cm}
\paragraph{Plan-level consistency.}
As smoothing vanishes, optimal maps of the regularized problems induce plans whose Wasserstein accumulation points solve the original singular Kantorovich problem (Theorem~\ref{thm:convergence}). 
This plan-level consistency does not require existence or uniqueness of a limiting Monge map.

\vspace{-0.3cm}
\paragraph{Experimental results.} Finally, we empirically validate the proposed framework through experiments on both synthetic singular transport tasks and real-world unpaired time-series imputation benchmarks (Table~\ref{table:TSI_average}). 
Our results demonstrate that the annealed smoothing approach effectively mitigates spurious solutions in singular environments, allowing HiSNOT to achieve state-of-the-art or highly competitive performance across standard imputation tasks.

Our contributions are summarized as follows:
\begin{itemize}[leftmargin=*, topsep=-1pt, itemsep=0pt]
    \item We identify spurious solutions in Hilbert-space SNOT as a failure of inner-minimizer identifiability and prove that regular source measures restore unique recovery of the Monge map; see Theorem~\ref{thm:uniqueness}.
    
    \item We establish a necessary kernel condition for Gaussian smoothing, a converse under Cameron--Martin compatibility, and universal regularization by injective covariance; see Theorem~\ref{25.12.09.14.29}.
    
    \item We prove plan-level consistency of annealed smoothing: every Wasserstein accumulation point of the smoothed optimal plans solves the original Kantorovich problem; see Theorem~\ref{thm:convergence}.
    
    \item We evaluate HiSNOT on finite-mode singular transports, unpaired time-series imputation, and Exchange generation, and report ablations of smoothing support and magnitude.
\end{itemize}

% \vspace{-0.3cm}
\paragraph{Relation to prior work.}
The spurious solution problem in SNOT has been studied in finite-dimensional Euclidean spaces \citep{fanTMLR,OTP,chu2026rate}. 
\citet{OTP} established that source measures assigning zero mass to sets of Hausdorff dimension at most d-1 ensure identifiability, and proposed Gaussian smoothing with annealing. 
For compact manifold-supported sources and regular targets, \citet{chu2026rate} derive finite-dimensional plan-stability rates and use them to select a terminal statistical noise scale.
We study the corresponding identifiability and smoothing questions on separable Hilbert spaces using Gaussian-null regularity; we do not establish Hilbert-space analogues of the finite-dimensional convergence-rate guarantees.
(For details, please refer to Appendix \ref{app:related_works}.)

For functional data, 
% {\color{red}\citet{zhu2024functional} study regularized transport estimation.}
\citet{zhu2024functional} also learn Hilbert--Schmidt linear maps between Hilbert spaces, balancing a squared $2$-Wasserstein distribution mismatch against a squared Hilbert--Schmidt norm penalty. 
We study standard quadratic OT: when semi-dual inner minimization identifies the Monge map, and how Gaussian source smoothing restores identifiability under suitable conditions.
Appendix~\ref{app:related_works} also compares weak-cost Kernel Neural OT and dynamic transport methods.

\section{Background}
\label{sec:background}

\paragraph{Optimal Transport}
The Optimal Transport (OT) problem is a mathematical framework that investigates the transport problem between the source measure $\mu$ and the target measure $\nu$ \citep{villani, santambrogio}. 
Formulated by \citet{monge1781memoire}, the Monge OT problem defines the optimal transport map as a deterministic measurable function $T$ that minimizes the total transport cost: 
\begin{equation}\label{eq:ot_monge} 
    \mathcal{T}(\mu, \nu) := \inf_{T_\# \mu = \nu}  \left[ \int_{\mathcal{X} } c(x,T(x)) \mu (\mathrm{d}x) \right].
\end{equation}
The pushforward constraint $T_\# \mu = \nu$ ensures that the map $T$ transforms the distribution $\mu$ into $\nu$, \textit{i.e.}, if $x \sim \mu$, then $T(x) \sim \nu$. 
Despite its intuitive nature, the Monge problem is inherently non-convex, and the existence of an optimal transport map generally requires regularity assumptions on the measures $\mu$ and $\nu$ as well as on the cost function $c$.

To address these limitations, \citet{Kantorovich1948} proposed a convex relaxation of the Monge problem by allowing mass splitting via joint probability measures. 
The Kantorovich problem is defined as:
\begin{equation} \label{eq:Kantorovich_finite}
    C(\mu,\nu):=\inf_{\pi \in \Pi(\mu, \nu)} \left[ \int_{\mathcal{X}\times \mathcal{Y}} c(x,y) \pi(\mathrm{d}x,\mathrm{d}y) \right],
\end{equation}
We refer to the joint probability distribution $\pi \in \Pi(\mu, \nu)$ as the \textit{transport plan} between $\mu$ and $\nu$. 
Unlike the Monge OT problem, the optimal transport plan (OT Plan) $\pi^{\star}$ is guaranteed to exist under mild assumptions on $(\mathcal{X}, \mu)$  and $(\mathcal{Y}, \nu)$ and the cost function $c$ \citep{villani}. 
When the optimal transport map $T^{\star}$ exists, the optimal transport plan $\pi^{\star}$ is given by the corresponding deterministic coupling, \textit{i.e.}, $\pi^{\star} = (\mathrm{Id} , T^{\star})_{\#} \mu$.

The Kantorovich problem admits the following semi-dual form under lower-semicontinuity and integrable lower-bound assumptions \citep[Thm.~5.10]{villani}; \citet[Prop.~1.11]{santambrogio} treats the compact, continuous-cost case:
\begin{equation} \label{eq:kantorovich-semi-dual_finite}
     S(\mu,\nu):= \sup_{V\in S_c} \left[ \int_\mathcal{X} V^c(x)\mu(\mathrm{d}x) + \int_\mathcal{Y} V(y) \nu (\mathrm{d}y) \right],
\end{equation}
where $S_c$ denotes the set of $c$-concave functions $V: \mathcal{Y}\rightarrow \mathbb{R}$ and $V^{c}$ denotes the $c$-transform of $V$:
\begin{equation} \label{eq:def_c_transform}
  V^c(x)=\underset{y\in \mathcal{Y}}{\inf}\left[ c(x,y) - V(y) \right].
\end{equation}

The semi-dual formulation optimizes out one of the two Kantorovich dual potentials through the $c$-transform, leaving a single potential $V$. 
When an admissible $ V^\star$ attains the supremum, it is called a Kantorovich potential.
We use globally defined dual representatives satisfying $V^c(x)+V(y)\leq c(x,y)$, with $V\in L^1(\nu)$ and $V^c\in L^1(\mu)$; an optimal potential attains the common primal--dual value. 
All ideal recovery statements assume such attainment.

\paragraph{Neural Optimal Transport in Euclidean Spaces}
Neural OT learns transport maps in finite-dimensional Euclidean spaces. 
Balanced semi-dual methods use Eq.~\ref{eq:kantorovich-semi-dual_finite} \citep{otm,fanscalable,otmICNN}; related unbalanced variants relax the marginal constraints \citep{uotm}.
We parameterize the scalar-valued potential by $V_\phi:\mathcal Y\to\mathbb R$ and the transport map by $T_\theta:\mathcal X\to\mathcal Y$. 
The latter approximates a pointwise minimizer in the $c$-transform of $V_\phi$.
By leveraging a correspondence between the $c$-transform $V_{\phi}^{c}$ and the transport network $T_{\theta}$, we define:
\begin{equation} \label{eq:def_T_finite}
    T_{\theta} (x) \in \operatorname*{arg\,min}_{y \in \mathcal{Y}} \left[c(x, y) - V_{\phi}\left( y \right)\right]  \quad \Leftrightarrow \quad V_{\phi}^c(x)=c\left(x,T_{\theta}(x) \right) - V_{\phi}\left(T_{\theta}(x)\right). 
\end{equation}
This parameterization transforms the semi-dual formulation into a max-min optimization objective $\mathcal{L}(V_{\phi}, T_{\theta})$:
\begin{equation}  \label{eq:snot_finite}
    \sup_{V_{\phi} \in S_c} \inf_{T_{\theta}:\mathcal{X} \rightarrow \mathcal{Y}} 
    \mathcal{L}(V_{\phi}, T_{\theta}) \,\, \quad \,\, \mathcal{L}(V, T) := 
        \int_{\mathcal{X}} c\left(x,T(x)\right)-V \left( T(x) \right) \mu(\mathrm{d}x) + \int_{\mathcal{Y}} V(y)  \nu(\mathrm{d}y).
\end{equation}

\paragraph{Spurious Solutions in Finite Dimensions}
A fundamental challenge in the SNOT approach is the \textbf{spurious solution problem} \citep{otm, fanTMLR, OTP}.
This means that the max-min solution $(V^{\star}, T_{\mathrm{rec}})$ (Eq.~\ref{eq:opt_pair_finite}) of the max-min objective function (Eq. \ref{eq:snot_finite}) may not accurately recover the optimal transport map $T^{\star}$. 
Formally, we define a pair $(V^{\star},T_{\mathrm{rec}})$ to be an \emph{optimal max-min solution} to (Eq. \ref{eq:opt_pair_finite}) if:
\begin{equation} \label{eq:opt_pair_finite}
V^{\star} \in \argmax_{V \in S_c}\; \inf_{T:\mathcal{X} \to \mathcal{Y}}\, \mathcal{L}(V,T), \qquad
T_{\operatorname{rec}} \in \argmin_{T:\mathcal{X} \to \mathcal{Y}}\, \mathcal{L}(V^{\star},T).
\end{equation}
It is well-known that the true pair of Kantorovich potential and optimal transport map $(V^{\star}, T^{\star})$ is an optimal max-min solution, but not every solution recovers the true optimal transport map \citep{otm, fanTMLR}. 

In finite-dimensional Euclidean spaces $\mathbb{R}^{d}$, \citet{OTP} showed that source measures assigning zero mass to sets of Hausdorff dimension at most $d-1$ ensure identifiability. 
A standard sufficient condition is absolute continuity $\mu \ll \mathcal{L}^{d}$ with respect to Lebesgue measure. 
Under this regularity, the spurious solution problem can be resolved, and the neural optimal transport objective correctly recovers the true Monge map. 
Our work generalizes this sufficient condition to \textbf{infinite-dimensional Hilbert spaces}. 
This requires addressing new theoretical challenges, as there is no Lebesgue measure in Hilbert spaces.

\section{Neural Optimal Transport on Hilbert Spaces} \label{sec:SNOT_hilbert}

This section extends neural optimal transport to separable Hilbert spaces. We formulate the \textit{Hilbert-space SNOT} (HiSNOT) objective (Section~\ref{sec:background}) and identify the spurious solution challenge for singular functional data.
We then introduce a regular measure, which is the Hilbert-space analogue of absolute continuity, and prove that regularity restores identifiability (Section~\ref{sec:ana}).

\subsection{HiSNOT: Semi-Dual Neural OT on Hilbert Spaces}

Let $H$ be a separable Hilbert space.
This setting includes the finite-dimensional Euclidean case $H=\mathbb{R}^d$ as a special case, but also covers infinite-dimensional function spaces such as $L^2$ spaces.
Let $\mu,\nu\in\mathcal{P}_2(H)$ denote probability measures with finite second moments, and we consider the quadratic cost $c(x,y):=\frac12\|x-y\|_H^2$. 
Then, the Kantorovich problem and its semi-dual formulation extend naturally to Hilbert spaces:
\begin{align}
    \label{eq:Kantorovich}
    C(\mu,\nu)
    &:=
    \inf_{\pi\in\Pi(\mu,\nu)}
    \frac{1}{2}\int_{H\times H} \|x-y\|_H^2\,\pi(\mathrm dx,\mathrm dy), \\
    \label{eq:kantorovich-semi-dual}
    S(\mu,\nu)
    &:=
    \sup_{V\in S_c}
    \left\{
        \int_H V^c(x)\,\mu(\mathrm dx)
        +
        \int_H V(y)\,\nu(\mathrm dy)
    \right\},
\end{align}
where $S_c$ consists of globally represented Borel $c$-concave potentials with $V\in L^1(\nu)$ and $V^c\in L^1(\mu)$, and the $c$-transform is $V^c(x) := \inf_{y\in H} \left[ c(x,y)-V(y) \right].$

Following the finite-dimensional framework (Section~\ref{sec:background}), our HiSNOT approach parameterizes the potential $V$ by a neural functional $V_\phi$ and represents the pointwise minimizer in the $c$-transform by a transport network $T_\theta$:
\begin{equation}
\label{eq:def_T}
    T_\theta(x)
    \in
    \operatorname*{arg\,min}_{y\in H}
    \left[
        c(x,y)-V_\phi(y)
    \right].
\end{equation}
This leads to the HiSNOT objective
\begin{equation}
\label{eq:snot}
    \sup_{V_\phi\in S_c}
    \inf_{T_\theta:H\to H}
    \mathcal L(V_\phi,T_\theta), \quad
    \mathcal L(V,T)
    :=
    \int_H
        c(x,T(x))-V(T(x))
    \mu(\mathrm dx)
    +
    \int_H V(y)\,\nu(\mathrm dy).
\end{equation}
The HiSNOT formulation is a direct extension of finite-dimensional SNOT because the semi-dual structure and max-min formulation remain unchanged. 
However, as we show next, the spurious solution problem becomes critical for functional data.

\paragraph{Spurious solution in Hilbert spaces}
As discussed in Section~\ref{sec:background}, spurious solutions are not limited to finite-dimensional SNOT; the same obstruction appears in Hilbert spaces.
Even for a Kantorovich potential $V^\star$, the pointwise minimization
\[
    \operatorname*{arg\,min}_{y\in H}
    \left[
        c(x,y)-V^\star(y)
    \right]
\]
may contain multiple minimizers, allowing a recovered map $T_{\mathrm{rec}}$ to achieve the same objective value without coinciding with the Monge map $T^{\star}$; see Appendix~\ref{sec:fur}.
We refer to the uniqueness condition
\[
    \operatorname*{arg\,min}_{y\in H}
    \left[
        c(x,y)-V^\star(y)
    \right]
    =
    \{T^\star(x)\}
    \qquad
    \mu\text{-a.e.}
\]
as \textbf{inner-minimizer identifiability}.
The next subsection proves that regular source measures restore this property.

\subsection{Identifiability under Regular Source Measures}
\label{sec:ana}

We now introduce a Hilbert-space regularity condition that rules out the spurious-solution obstruction described above.
Our approach replaces absolute continuity by regularity with respect to Gaussian null sets.
In finite-dimensional Euclidean spaces, \citet{OTP} rule out spurious solutions under a Hausdorff-dimensional non-charging condition, and a standard sufficient condition is absolute continuity \(\mu\ll\mathcal L^d\).
In Hilbert spaces, there is no infinite-dimensional Lebesgue measure, so we replace absolute continuity by regularity with respect to Gaussian null sets.

\begin{definition}[Regular measure]
\label{26.01.15.13.31}
Let $H$ be a separable Hilbert space. 
A probability measure $\mu\in\mathcal P(H)$ is called \emph{regular} if $\mu(N)=0$ for every Gaussian null set $N\subset H$.
We denote by $\mathcal P^r(H)$ the set of regular probability measures on $H$, and set
\(
    \mathcal P^r_2(H):=\mathcal P^r(H)\cap\mathcal P_2(H).
\)
\end{definition}

The formal definition of Gaussian null sets is given in Appendix~\ref{app:def}.
When $H=\mathbb R^d$, Gaussian null sets coincide with Lebesgue-negligible sets
\citep{ambrosio2005gradient}; hence regularity is equivalent to
$\mu\ll\mathcal L^d$. 
Thus, regular measures provide the Hilbert-space analogue of absolutely continuous
source measures. In particular, if $\mu\in\mathcal P^r_2(H)$ and $\nu\in\mathcal P_2(H)$, then the quadratic OT problem admits a unique Monge map; see Proposition~\ref{25.12.16.14.13}.

A source with finitely many modes and a joint coefficient density is regular in their span but non-regular in the ambient infinite-dimensional space.
Laws on proper closed affine subspaces are also non-regular; nondegenerate Gaussians on $H$ are regular. Appendix~\ref{app:regularity_examples} explains these examples and distinguishes empirical measures from population laws.

Theorem~\ref{thm:uniqueness} establishes that regular source measures eliminate spurious solutions in HiSNOT:
\begin{theorem}[Identifiability of HiSNOT under regularity]
\label{thm:uniqueness}
Let $H$ be a separable Hilbert space.
Assume that $\mu\in\mathcal P^r_2(H)$ and $\nu\in\mathcal P_2(H)$.
Let $T^\star:H\to H$ be the unique Monge map from $\mu$ to $\nu$, and let $V^\star$ be an admissible Kantorovich potential, with $V^\star\in L^1(\nu)$ and $(V^\star)^c\in L^1(\mu)$, represented globally as above.
Then the following statements hold.

\begin{enumerate}[leftmargin=*, topsep=-1pt, itemsep=-1pt, label=(\roman*)]
    \item 
    For $\mu$-almost every $x\in H$,
    \begin{equation}
    \label{eq:argmin}
        \operatorname*{arg\,min}_{y\in  H}
        \left[
            c(x,y)-V^\star(y)
        \right]
        =
        \{T^\star(x)\}.
    \end{equation}

    \item 
    Let $(V^\star,T_{\mathrm{rec}})$ be an optimal max-min solution of the HiSNOT objective (Eq. \ref{eq:snot}), with $V^\star$ a Kantorovich potential. 
    Then every measurable inner minimizer coincides with the Monge map $\mu$-almost everywhere; in particular,
    \(
        T_{\mathrm{rec}}(x)=T^\star(x)
        \text{ for }\mu\text{-a.e. }x\in H.
    \)
\end{enumerate}
\end{theorem}

Theorem~\ref{thm:uniqueness} shows that regularity restores \emph{inner-minimizer identifiability}: the Kantorovich potential uniquely determines the Monge map through the pointwise minimization problem Eq. \ref{eq:argmin}.
Equivalently, under a regular source measure, any optimal max--min solution whose potential is Kantorovich recovers the optimal transport map.
Thus, for the population HiSNOT objective with inner minimization over all measurable maps for which the objective is defined, regular source measures rule out spurious inner minimizers.
Theorem~\ref{thm:uniqueness} gives pointwise identifiability, Theorem~\ref{25.12.09.14.29} specifies valid smoothing conditions, and Theorem~\ref{thm:convergence} gives plan-level consistency.
Their proofs appear in Appendices~\ref{26.05.04.01.17}, \ref{26.05.04.00.24}, and \ref{26.01.29.13.52}, respectively.

% Theorem~\ref{thm:uniqueness} gives pointwise identifiability, Theorem~\ref{25.12.09.14.29} specifies valid smoothing conditions, and Theorem~\ref{thm:convergence} gives plan-level consistency. 
% Their proofs appear in Appendices~\ref{26.05.04.01.17}, \ref{26.05.04.00.24}, and \ref{26.01.29.13.52}, respectively.
% A finite-mode source with a coefficient density is regular in its finite-dimensional span but non-regular in the ambient infinite-dimensional space. Laws on proper closed affine constraint sets are also non-regular; nondegenerate Gaussians are regular. 
% Appendix~\ref{app:regularity_examples} explains these examples and why empirical atomicity does not diagnose population regularity.

\section{Method: HiSNOT with Annealed Smoothing}
\label{sec:method}

\paragraph{Challenge: singularity of functional data.}
Theorem~\ref{thm:uniqueness} shows that regularity of the source measure rules out spurious inner minimizers in the ideal HiSNOT objective.
However, functional data such as paths, time series, and fields often concentrate on low-dimensional or otherwise singular structures in Hilbert space \citep{annurev}.
Thus, directly applying HiSNOT to functional data can lead to inner-minimizer non-identifiability and spurious solutions, as illustrated in Appendix~\ref{sec:fur}. 
This section \textbf{addresses this challenge by introducing Gaussian smoothing that restores regularity} while remaining consistent with the original transport problem. Based on this, we introduce \textbf{HiSNOT with annealed smoothing}.

\subsection{Theory: Regularization via Gaussian Smoothing}
\label{sec:smoothing}

\paragraph{HiSNOT with annealed Gaussian smoothing.}

Let \(Q:H\to H\) be a positive self-adjoint trace-class operator, and let \(\mathcal N(0,Q)\) be the centered Gaussian measure on \(H\) with covariance \(Q\).
For each smoothing level \(\varepsilon>0\), we replace the source measure \(\mu\) in
the HiSNOT objective by
\(
    \mu_\varepsilon
    :=
    \mu \ast \mathcal N(0,\varepsilon^2Q).
\)
Equivalently, if \(X\sim\mu\) and \(G\sim\mathcal N(0,Q)\) is independent of \(X\), then $X_\varepsilon:=X+\varepsilon G
    \sim \mu_\varepsilon$.
For this smoothed source measure, we optimize
\begin{equation}
\label{26.05.04.00.37}
    \sup_{V_\phi\in S_c}
    \inf_{T_\theta:H\to H}
    \mathcal L_\varepsilon(V_\phi,T_\theta),
\end{equation}
where
\[
    \mathcal L_\varepsilon(V,T)
    :=
    \int_H
    \left(
        c(x,T(x))-V(T(x))
    \right)\mu_\varepsilon(\mathrm dx)
    +
    \int_{H}V(y)\,\nu(\mathrm dy).
\]
We then anneal the smoothing amplitude \(\varepsilon\downarrow0\).

This construction has two goals. 
First, for every positive smoothing level, we want $\mu_\varepsilon\in\mathcal P^r_2(H)$, so that the corresponding HiSNOT objective is identifiable by Theorem~\ref{thm:uniqueness}.
Second, as \(\varepsilon\downarrow0\), we want the smoothed problems to remain consistent with the original singular OT problem: if \(T_\varepsilon^\star\) is the Monge map from \(\mu_\varepsilon\) to \(\nu\), then the induced plans $\pi_\varepsilon^\star    :=(\operatorname{Id},T_\varepsilon^\star)_\#\mu_\varepsilon$ should have Wasserstein accumulation points that solve the original Kantorovich problem between \(\mu\) and \(\nu\).

\paragraph{Conditions for successful Gaussian smoothing.}
The unsmoothed component must already be regular. 
The converse holds when the source coefficients in the noisy directions satisfy the weighted square-summability condition below.

\begin{theorem}[Regularity under Gaussian smoothing]
\label{25.12.09.14.29}
Let $Q$ be a positive self-adjoint trace-class operator on a separable Hilbert space $H$, let $\gamma_Q=\mathcal N(0,Q)$, and let $\mu\in\mathcal P(H)$.
Let $\{e_k\}$ be an orthonormal eigenbasis of $Q$, with $Qe_k=\lambda_k e_k$ and $\lambda_k\geq0$, and let $\pi_{\operatorname{Ker}(Q)}$ denote the orthogonal projection onto $\operatorname{Ker}(Q)$.
\begin{enumerate}[label=(\roman*),leftmargin=*,itemsep=1pt,topsep=2pt]
\item \label{thm:kernel_necessity}
If $\gamma_Q*\mu\in\mathcal P^r(H)$, then $(\pi_{\operatorname{Ker}(Q)})_\#\mu\in\mathcal P^r(\operatorname{Ker}(Q))$.
\item \label{thm:kernel_equivalence}
If
\begin{equation}
\label{eq:cm_coefficient_condition}
\sum_{k:\lambda_k>0}\frac{|\langle x,e_k\rangle_H|^2}{\lambda_k}<\infty
\qquad\text{for $\mu$-almost every $x\in H$},
\end{equation}
then
\begin{equation}
\label{eq:kernel_characterization}
\gamma_Q*\mu\in\mathcal P^r(H)
\quad\Longleftrightarrow\quad
(\pi_{\operatorname{Ker}(Q)})_\#\mu\in\mathcal P^r(\operatorname{Ker}(Q)).
\end{equation}
The additional assumption is automatic when $Q$ has finite rank.
\item \label{thm:kernel_injective}
If $\operatorname{Ker}(Q)=\{0\}$, then, without the additional assumption in part~\ref{thm:kernel_equivalence}, $\gamma_Q*\mu\in\mathcal P^r(H)$ for every $\mu\in\mathcal P(H)$.
\end{enumerate}
\end{theorem}

The condition in~Eq. \ref{eq:cm_coefficient_condition} requires finite weighted source energy in positive-variance directions, with no restriction on zero-variance coefficients.
It is sufficient for the converse, but not necessary for regularity.
Appendix~\ref{26.05.04.00.24} proves the theorem and gives a counterexample to the unrestricted converse.

Under either sufficient condition above, Theorem~\ref{thm:uniqueness} gives identifiability of exact inner minimizers at positive smoothing levels when an admissible Kantorovich potential is attained.
% {\color{red}A summable positive spectrum, optionally combined with a finite-rank FPCA-guided term, supplies such an injective covariance. FPCA is an allocation heuristic rather than a regularity test; see Appendix~\ref{app:covariance_choice}.}
In particular, any covariance with a summable, strictly positive spectrum is injective and hence regularizes every source. In practice, such a covariance may be augmented with a finite-rank term along FPCA directions, which allocates smoothing variance but plays no role in the regularity guarantee; see Appendix~\ref{app:covariance_choice}.

\paragraph{Plan-level consistency as smoothing vanishes.}

The second goal is to ensure that the solutions of the smoothed HiSNOT problems remain consistent with the original singular OT problem as the smoothing amplitude vanishes.
Since the limiting source measure may be singular, the correct convergence notion is at the level of transport plans rather than maps: the original problem may fail to admit a Monge map, or optimal plans may be non-unique.

\begin{theorem}[Plan-level consistency of HiSNOT]
\label{thm:convergence}
Let $\{\mu_k\}_{k\in\mathbb N}\subset \mathcal P_2^r(H)$ and let
$\mu,\nu\in\mathcal P_2(H)$ satisfy
$$
    \lim_{k\to\infty} W_2(\mu_k,\mu)=0.
$$
For each $k$, let $T_k^\star$ be the unique optimal Monge map from $\mu_k$ to $\nu$, and define the induced optimal plan $\pi_k^\star
    :=
    (\operatorname{Id},T_k^\star)_\#\mu_k$.
Then the sequence $\{\pi_k^\star\}_{k\in\mathbb N}$ is relatively compact in $\mathcal P_2(H\times H)$ with respect to the $W_2$-topology, and every $W_2$-accumulation point of $\{\pi_k^\star\}_{k\in\mathbb N}$ is an optimal transport plan in $\Pi(\mu,\nu)$.
\end{theorem}

We apply Theorem~\ref{thm:convergence} to an annealed sequence of smoothed source measures
\[
    \mu_{\varepsilon_k}
    =
    \mu\ast\mathcal N(0,\varepsilon_k^2Q),
    \qquad
    \varepsilon_k\downarrow0.
\]
Since $W_2^2(\mu_{\varepsilon_k},\mu)
    \leq
    \varepsilon_k^2\operatorname{Tr}(Q)
    \to0$, we have \(\mu_{\varepsilon_k}\to\mu\) in \(W_2(H)\).
Under a sufficient condition of Theorem~\ref{25.12.09.14.29} (for example, injective $Q$), each $\mu_{\varepsilon_k}$ belongs to $\mathcal P_2^r(H)$.
Hence each positive-smoothing HiSNOT problem admits a unique Monge map $T_{\varepsilon_k}^\star$, and the induced plans $\pi_{\varepsilon_k}^\star
    :=
    (\operatorname{Id},T_{\varepsilon_k}^\star)_\#\mu_{\varepsilon_k}$
are relatively compact in $\mathcal P_2(H\times H)$.
Every Wasserstein accumulation point of these plans is an optimal Kantorovich plan for the original problem between $\mu$ and $\nu$.
Thus, HiSNOT with annealed smoothing is consistent at the plan level.
This is the natural form of convergence in singular settings, where a limiting Monge map may not exist, and optimal plans may not be unique.

\subsection{Implementation Details}
We describe the practical implementation of HiSNOT with annealed smoothing.
The theoretical objective and consistency guarantees were introduced in Section~\ref{sec:smoothing}; here we focus on the sample-based training procedure.
At smoothing level $\varepsilon>0$, we draw samples
$x_i\sim\mu$, $y_i\sim\nu$, and independent perturbations
$G_i\sim\mathcal N(0,Q)$, and form
\(
    \widetilde x_i
    =
    x_i+\varepsilon G_i
    \sim
    \mu_\varepsilon.
\)
We train the potential network $V_\phi$ and the transport network $T_\theta$ using the empirical smoothed objective
\[
    \widehat{\mathcal L}_\varepsilon(\phi,\theta)
    =
    \frac1m\sum_{i=1}^m
    \left[
        \frac12\|\widetilde x_i-T_\theta(\widetilde x_i)\|_H^2
        -
        V_\phi(T_\theta(\widetilde x_i))
    \right]
    +
    \frac1m\sum_{i=1}^m V_\phi(y_i).
\]
Training alternates between minimizing this objective with respect to $\theta$ and maximizing it with respect to $\phi$.
The potential $V_\phi$ is parameterized as a scalar functional on $H$, and $T_\theta$ is parameterized as an $H$-valued neural operator.
In experiments, we use Fourier Neural Operator architectures \citep{li2020fourier,kovachki2023neural}.

\paragraph{Practical Gaussian smoothing via spectral augmentation.}

For functional data, we generate Gaussian perturbations by adding noise to spectral coefficients.
Let $\{e_k\}_{k\in\mathbb N}$ be an orthonormal basis of $H$, and choose summable eigenvalues $\{\lambda_k\}_{k\in\mathbb N}\subset[0,\infty)$ defining $Qe_k=\lambda_k e_k$.
We sample
\[
    G
    =
    \sum_{k=1}^{\infty}
    \sqrt{\lambda_k}\,\xi_k e_k,
    \qquad
    \xi_k\sim\mathcal N(0,1)\ \text{i.i.d.}
\]
For $x\sim\mu$, where $\mu\in\mathcal{P}_2(H)$ and $x$ is independent of $G$, the smoothed sample is $
x_\varepsilon=x+\varepsilon G =\sum_k(\langle x,e_k\rangle_H +\varepsilon\sqrt{\lambda_k}\xi_k)e_k$,
with $\mu_\varepsilon=\operatorname{Law}(x_\varepsilon)$.
In practice, we retain $M$ noise modes, using $G_M=\sum_{k=1}^M\sqrt{\lambda_k}\xi_k e_k$ and perturbing each sample as $x+\varepsilon G_M$.
For fixed $\varepsilon>0$ and $Q$, coupling this sample with $x_\varepsilon$ using the same $x$ and Gaussian coefficients gives
\[
W_2^2\!\left(
    \operatorname{Law}(x+\varepsilon G_M),\mu_\varepsilon
\right)
\leq
\varepsilon^2\sum_{k>M}\lambda_k
\xrightarrow[M\to\infty]{}0.
\]
Because $Q$ is trace class, the total variance $\sum_{k>M}\lambda_k$ of the omitted noise modes tends to zero.
This justifies approximating the full-space smoothed source $\mu_\varepsilon$ by finite-mode perturbations; Theorem~\ref{25.12.09.14.29} gives conditions for regularity of $\mu_\varepsilon$.
Appendix~\ref{app:hilbert_gaussian} gives the proof, the retained-coefficient density, and the additional error when the source is also projected onto the retained modes.
The values $\lambda_k$ determine the smoothing direction; directions with $\lambda_k=0$ are not perturbed. (For details, please refer to Appendix \ref{app:implementation_detail}.)

\paragraph{Annealing schedule.}

We train with a nonincreasing smoothing schedule $\varepsilon_0\geq\varepsilon_1\geq\cdots\geq\varepsilon_{N-1}>0$.
At each stage, the parameters learned at level $\varepsilon_n$ are used to initialize training at level $\varepsilon_{n+1}$.
Following \citet{OTP}, we use decreasing schedules with a positive terminal noise floor; the synthetic interpolation and the clipped imputation schedule are specified separately in Appendix~\ref{app:implementation_detail}.
Further implementation details and hyperparameters are provided in
Appendix~\ref{app:implementation_detail}. 
% {\color{blue}
Appendix~\ref{app:smoothing_cost} details the additional computational and memory costs of Gaussian smoothing and the assumptions underlying these estimates.
% }

% {\color{red}For a batch of $m$ samples and $M$ retained modes, each perturbation costs $O(mM)$ coefficient operations and temporary storage. 
% One potential and $K_T$ transport updates add $O((K_T+1)mM)$ operations per outer iteration, plus any grid synthesis, without changing neural evaluation counts. Appendix~\ref{app:smoothing_cost} gives the assumptions; no wall-clock comparison is claimed.}
% {\color{blue}
% Appendix~\ref{app:smoothing_cost} details the additional computational and memory costs of Gaussian smoothing and the assumptions underlying these estimates.}

\section{Experiments} \label{sec:experiments}
% {\color{red}Synthetic experiments isolate identifiability and directional smoothing in finite-mode spaces.
% Imputation and the additional Exchange generation experiment evaluate functional/operator models using discretized observations. 
% These roles do not establish an advantage caused by infinite dimensionality alone.}
% {\color{blue}
% Synthetic experiments use controlled finite-mode examples to illustrate spurious solutions and the effects of Gaussian smoothing along different spectral directions.
% Imputation benchmarks, and the additional Exchange generation experiment (Appendix~\ref{app:exchange_generation}) evaluate the practical performance of functional/operator models on discretely observed time series.
% The main experiments are organized as follows:
% }
We empirically validate our theoretical analysis on Semi-dual Neural OT on Hilbert spaces through the following experiments:
\begin{itemize}[leftmargin=*, topsep=-1pt, itemsep=-1pt]
	\item In Sec \ref{sec:exp_synthetic_data}, we evaluate the effectiveness of Gaussian smoothing in overcoming spurious solutions.
	\item In Sec \ref{sec:exp_smoothing_characterization}, we empirically test our characterization theorem for smoothing operators (Theorem~\ref{25.12.09.14.29}) by comparing theoretically appropriate operators and inappropriate operators.
	\item In Sec \ref{sec:exp_timeseries}, we assess our SNOT model on time-series imputation benchmarks, which are real-world applications for high-dimensional function space datasets.
\end{itemize}
For implementation details of experiments, please refer to Appendix \ref{app:implementation_detail}.

\subsection{Effect of Gaussian Smoothing on Spurious Solutions} \label{sec:exp_synthetic_data}
We compare vanilla HiSNOT with its annealed-smoothing variant on four synthetic singular transport examples: \textit{Perpendicular}, \textit{Parallel}, \textit{One-to-Many}, and \textit{Grid}.
These examples illustrate the spurious minimizers discussed in Appendix~\ref{sec:fur}.
Specifically, we compare the vanilla \OURS model (Eq. \ref{eq:snot}) and the regularized version utilizing Gaussian smoothing (Eq. \ref{26.05.04.00.37}). In these settings, the vanilla model is exposed to ill-posedness and spurious solution problems. 
In contrast, the regularized model introduces principled Gaussian smoothing with a noise level $\varepsilon$. 
This leads to a well-posed learning objective that eliminates spurious solutions.
The evaluation is conducted on synthetic datasets, including Perpendicular (Example 1) and One-to-Many (Example 2) (See Appendix \ref{app:example_details} for details). Here, we utilize two metrics: the transport cost error $D_{cost} = \left| W^2_2 (\mu, \nu) - \int_H \Vert T_{\theta}(x) - x \Vert_H^2 \mu(\mathrm{d}x) \right|$ and the target distribution error $D_{target} = W^2_2 (T_{\theta \#} \mu, \nu)$.
$D_{cost}$ measures how closely the model approximates the theoretically optimal transport cost, while $D_{target}$ measures how accurately the model generates the target distribution.

The results are provided in Table \ref{tab:synthetic} and Figure \ref{fig:synthetic}. 
While the baseline \OURS exhibits spurious solutions and fails to accurately reconstruct the target measure, our regularized model substantially reduces target-distribution error in the singular examples.
Quantitatively, our model is significantly more accurate in the Perpendicular, One-to-Many, and Grid tasks. 
Specifically, in the Perpendicular data, the introduction of Gaussian smoothing significantly reduces the transport cost error $D_{cost}$ from 0.101 to 0.038 and the target distribution error $D_{target}$ from 0.249 to 0.006. 
Interestingly, in the Parallel data, the unregularized model performs well. 
Thus, both models achieve results with competitive accuracy.
\begin{figure*}[t]
     \centering
     \begin{subfigure}[b]{0.23\textwidth} 
         \centering
         \includegraphics[width=\textwidth]{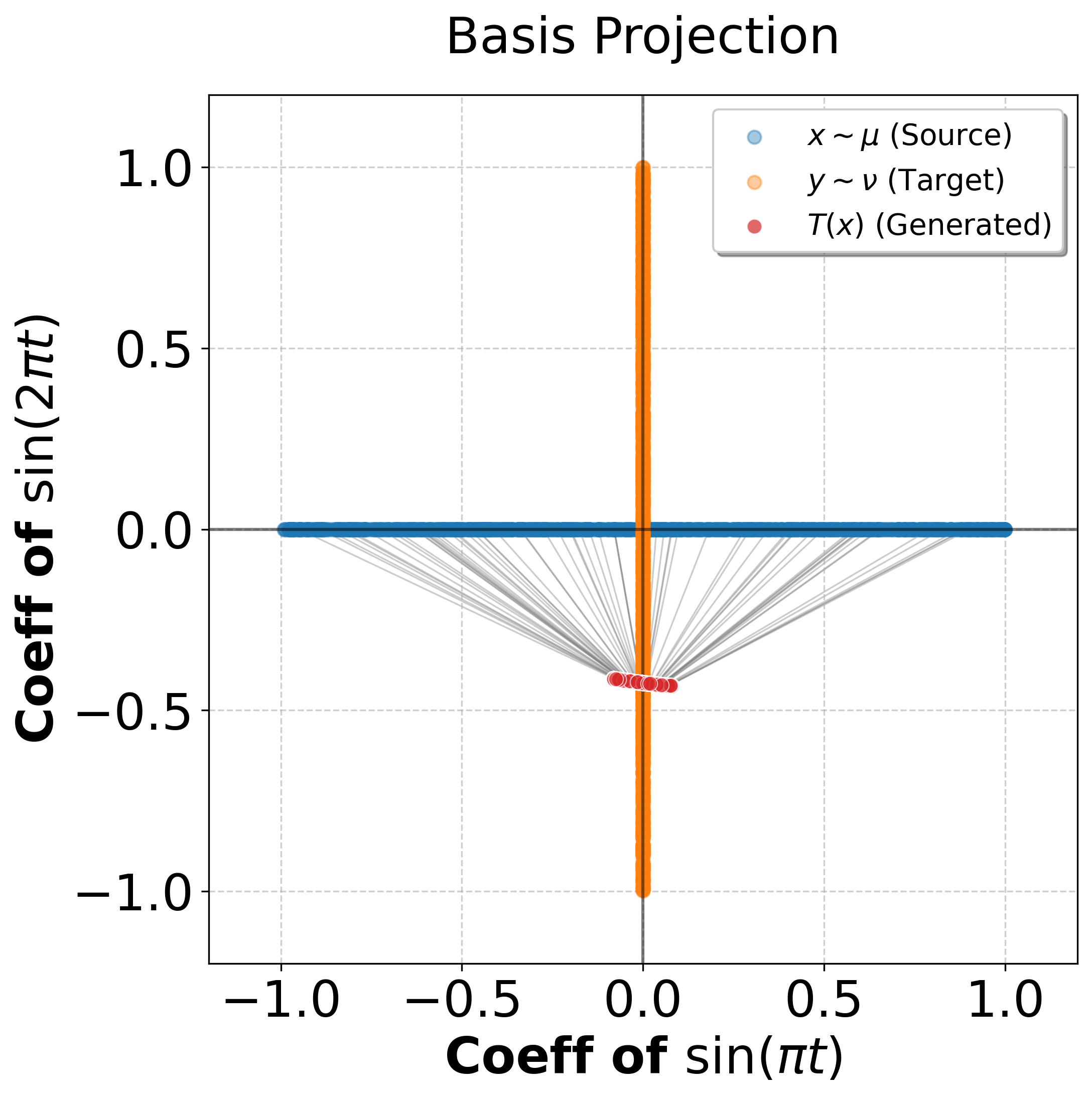}
     \end{subfigure}
     \hfill 
     \begin{subfigure}[b]{0.23\textwidth}
         \centering
         \includegraphics[width=\textwidth]{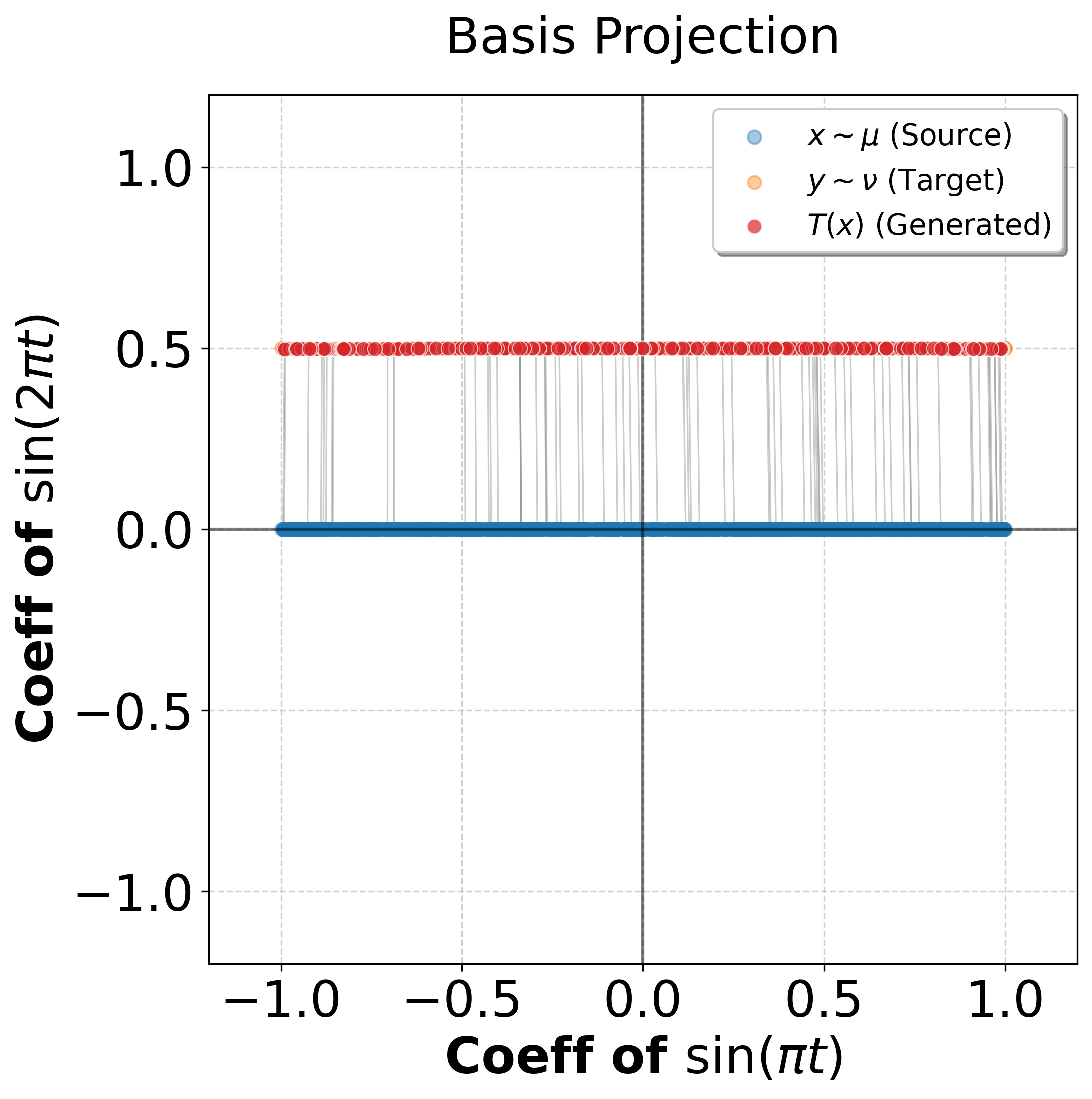}
     \end{subfigure}
     \hfill
     \begin{subfigure}[b]{0.23\textwidth}
         \centering
         \includegraphics[width=\textwidth]{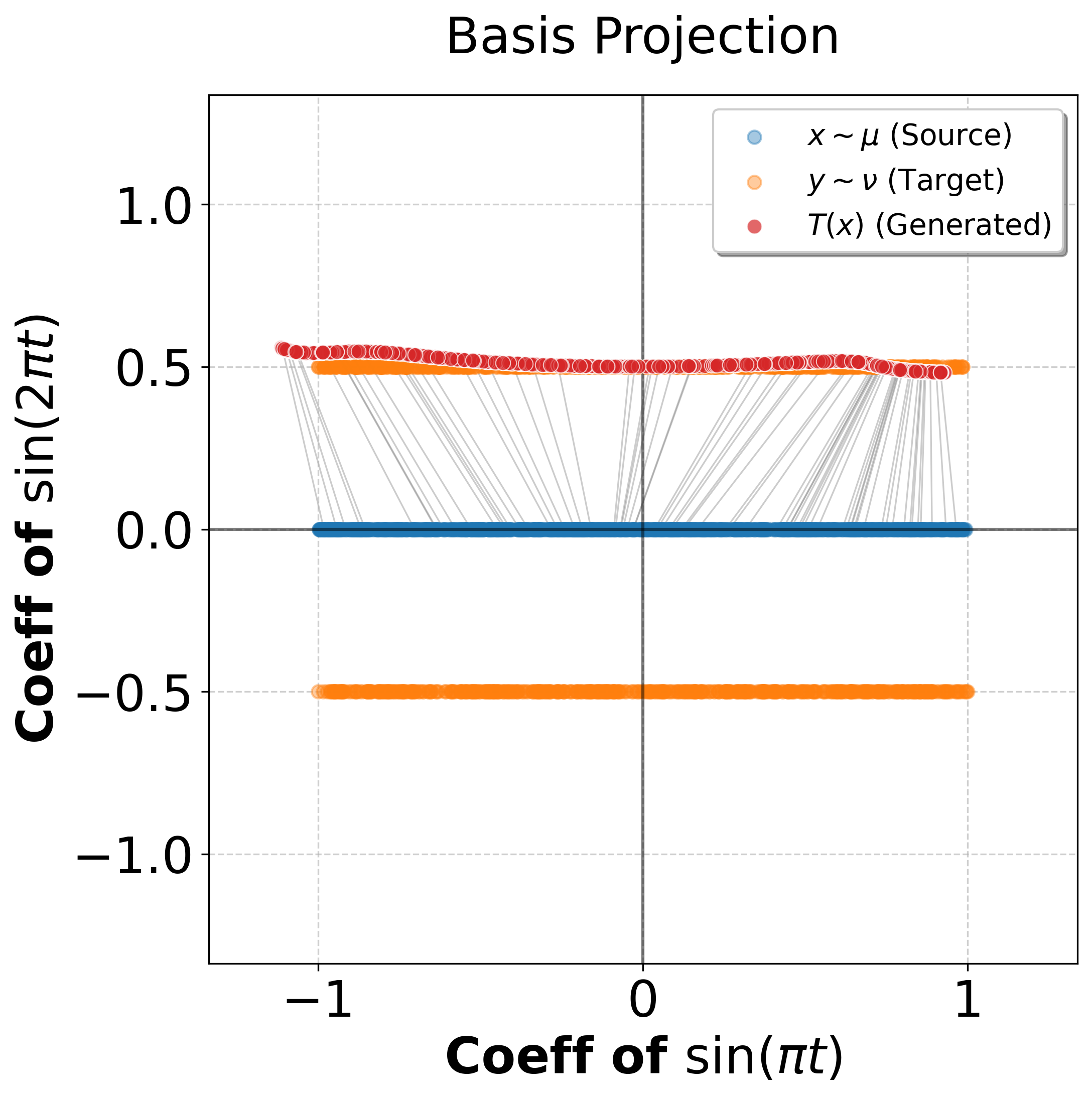}
     \end{subfigure}
     \hfill
     \begin{subfigure}[b]{0.23\textwidth}
         \centering
         \includegraphics[width=\textwidth]{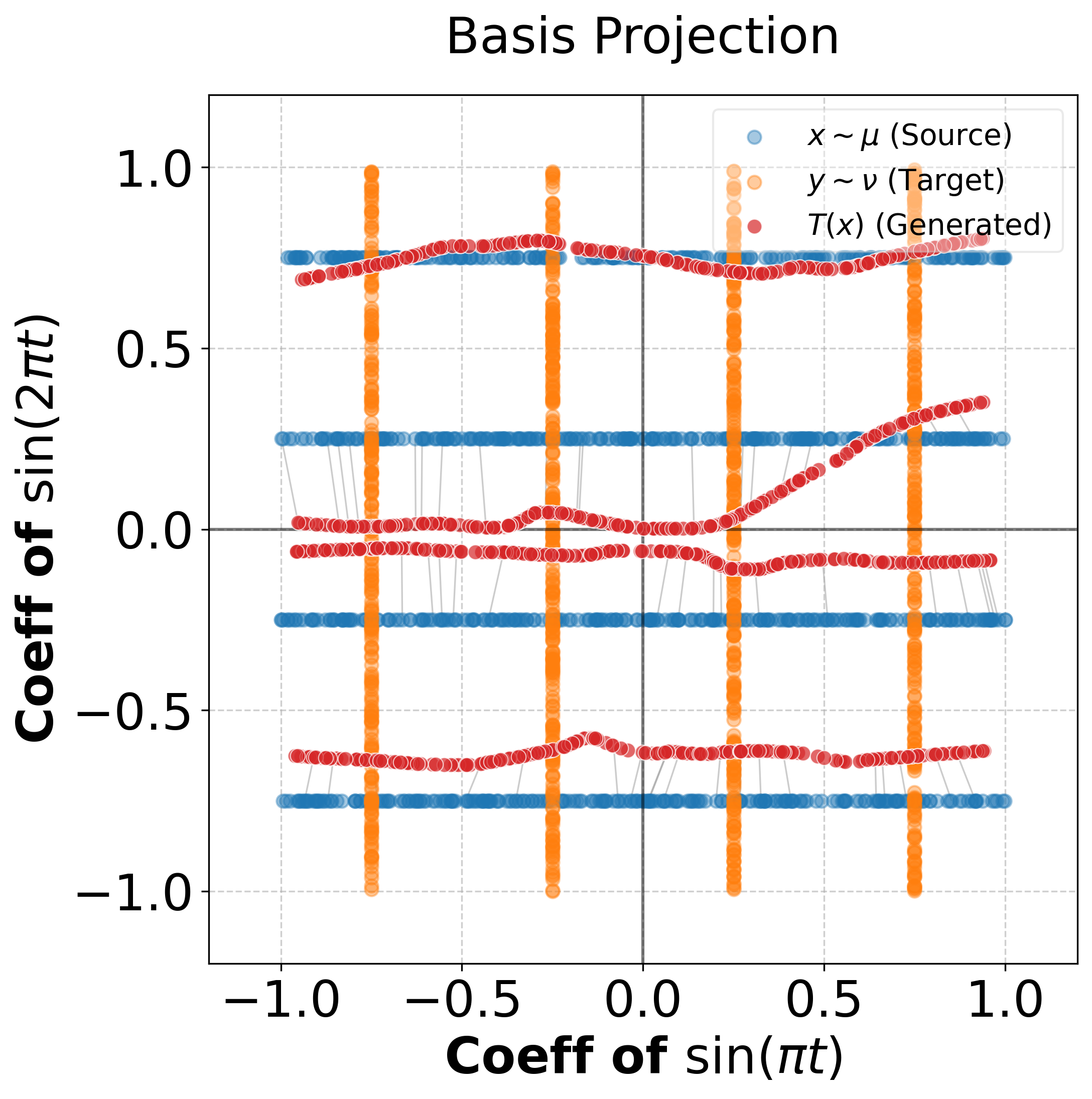}
     \end{subfigure}

     \medskip

     \begin{subfigure}[b]{0.23\textwidth}
         \centering
         \includegraphics[width=\textwidth]{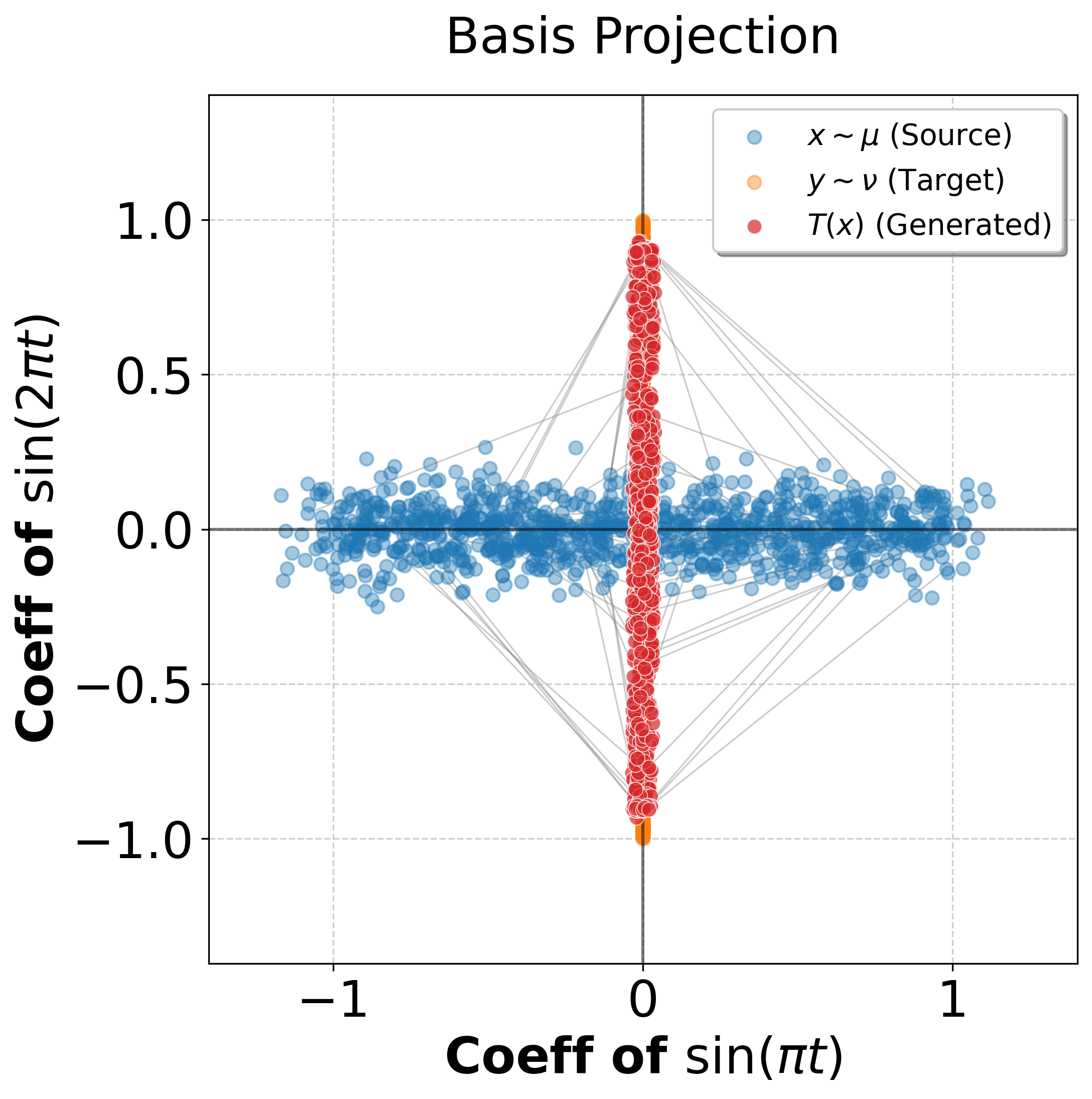}
         \caption{Perpendicular}
     \end{subfigure}
     \hfill
     \begin{subfigure}[b]{0.23\textwidth}
         \centering
         \includegraphics[width=\textwidth]{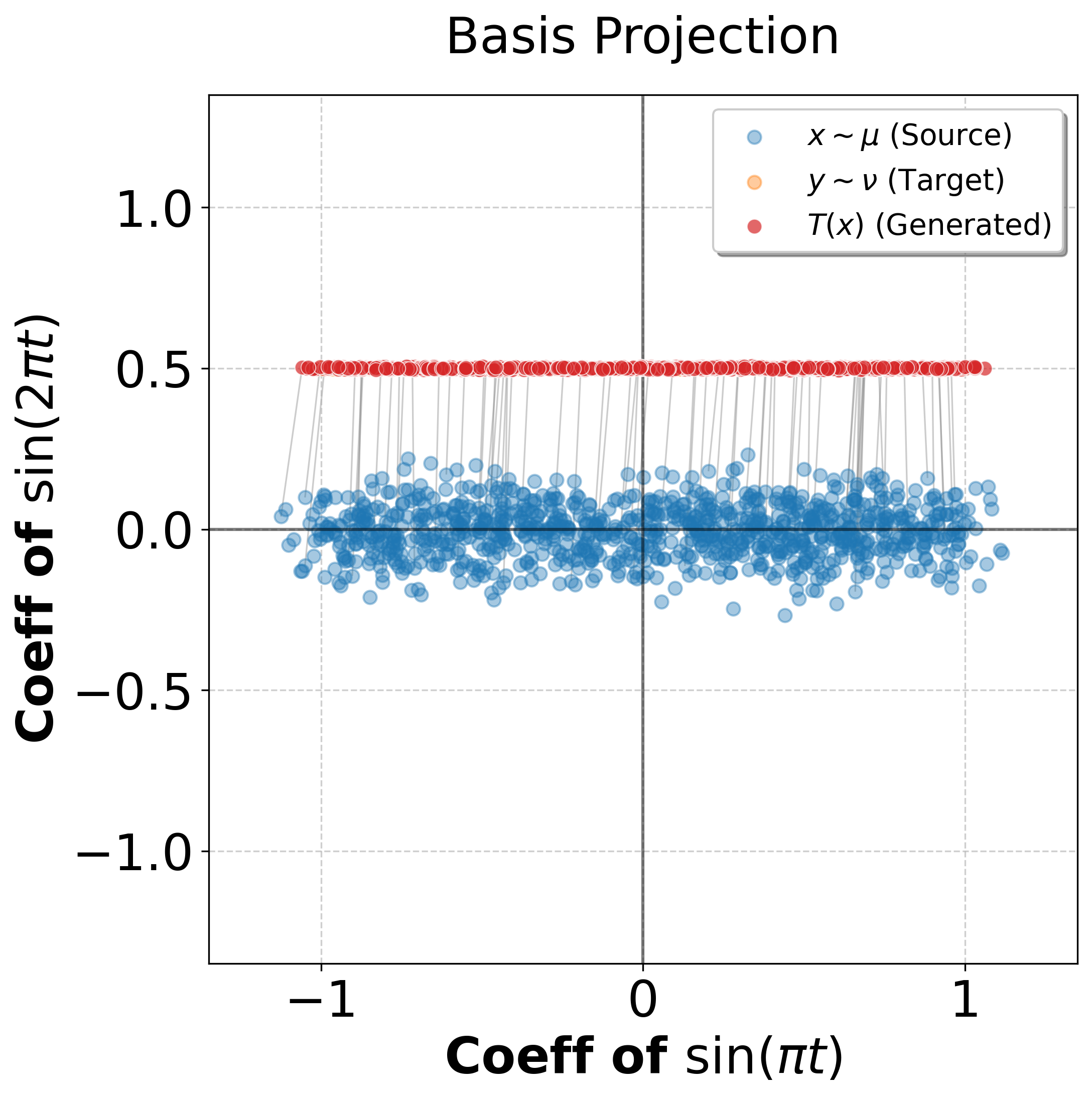}
         \caption{Parallel}
     \end{subfigure}
     \hfill
     \begin{subfigure}[b]{0.23\textwidth}
         \centering
         \includegraphics[width=\textwidth]{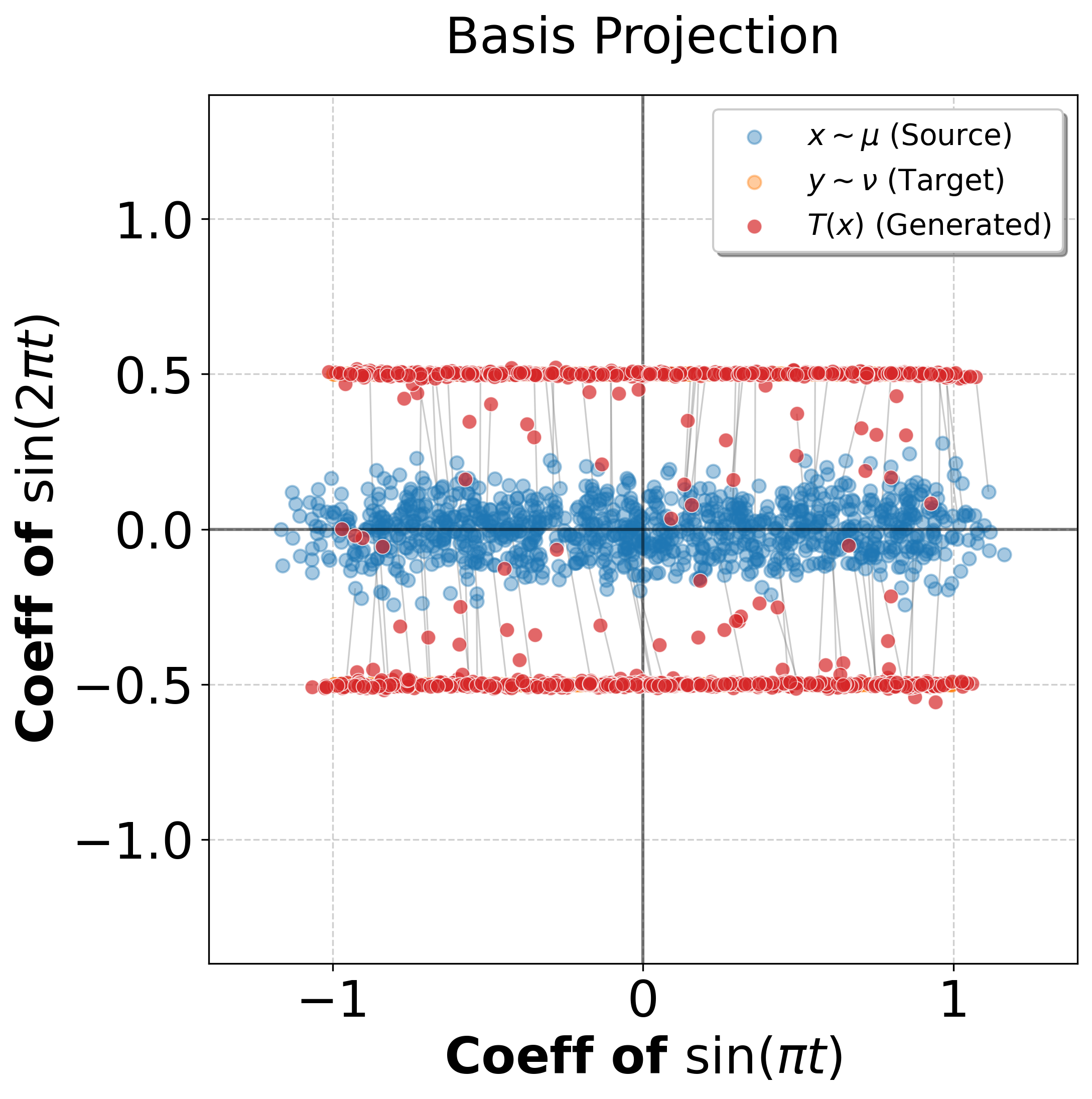}
         \caption{One-to-Many}
     \end{subfigure}
     \hfill
     \begin{subfigure}[b]{0.23\textwidth}
         \centering
         \includegraphics[width=\textwidth]{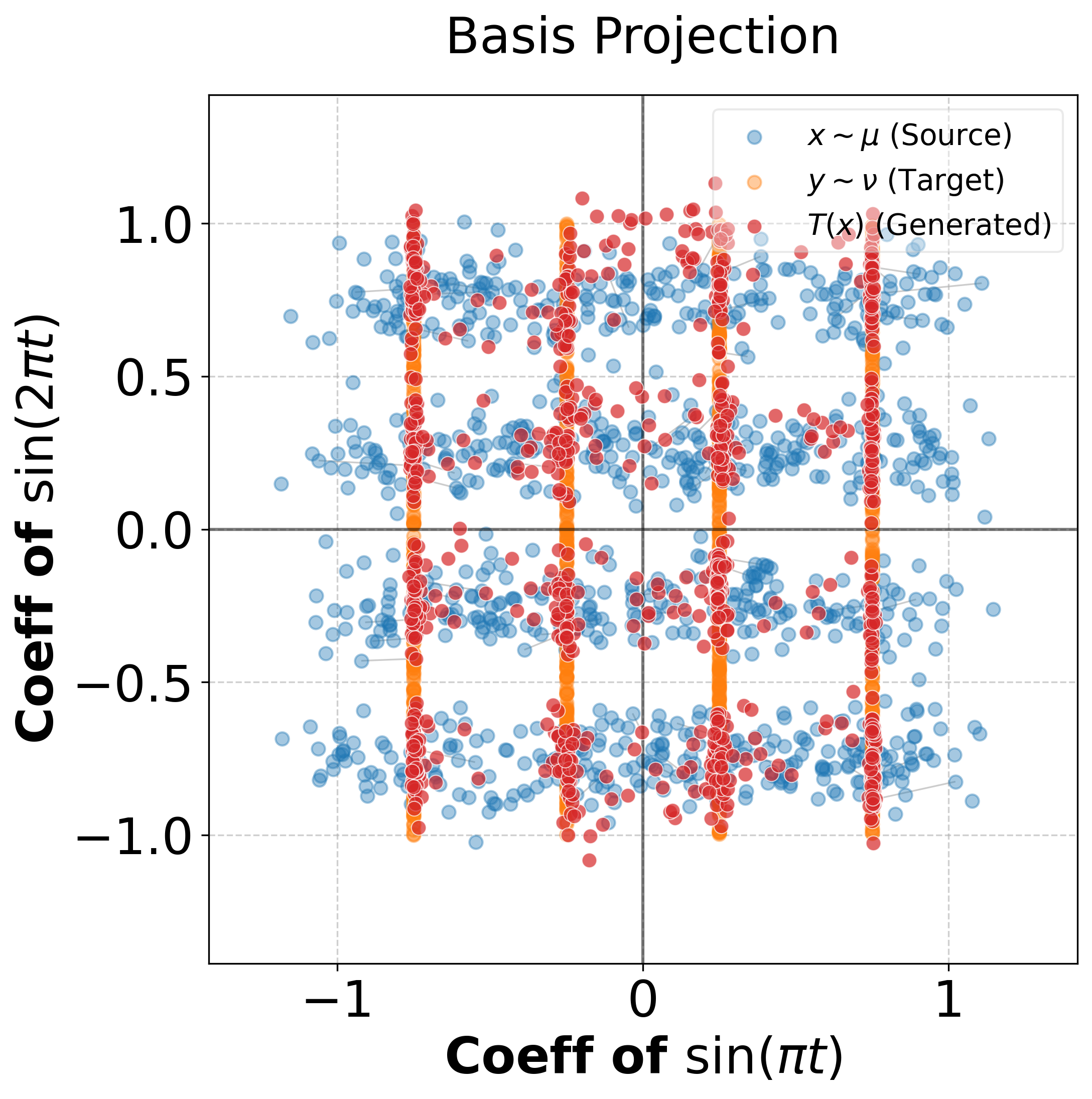}
         \caption{Grid}
     \end{subfigure}

     \caption{\textbf{Comparison of transport maps} learned by vanilla \OURS (top row) and \OURS with annealed smoothing (bottom row). The vanilla model exhibits spurious solutions in singular cases (Perpendicular, One-to-Many, Grid), while smoothing improves agreement with the target distribution in these finite-mode experiments.}
     \label{fig:synthetic}
     % \vspace{-8pt}
\end{figure*}

\begin{table*}[h]
\centering
\small
\caption{
\textbf{Quantitative comparison of \OURS performance with and without Gaussian smoothing} across synthetic scenarios. We report the transport cost error $D_{\text{cost}}(\downarrow)$ and target distribution error $D_{\text{target}}(\downarrow)$. 
Smoothing (\cmark) reduces the reported errors in three of the four tested scenarios.}
\label{tab:synthetic}
\scalebox{1.0}{
\begin{tabular}{lcccccccc}
\toprule
\multirow{2}{*}{Smoothing} & \multicolumn{2}{c}{Perpendicular} & \multicolumn{2}{c}{Parallel} & \multicolumn{2}{c}{One-to-Many} & \multicolumn{2}{c}{Grid} \\
\cmidrule(lr){2-3} \cmidrule(lr){4-5} \cmidrule(lr){6-7} \cmidrule(lr){8-9}
& $D_{cost}$ & $D_{target}$ & $D_{cost}$ & $D_{target}$ & $D_{cost}$ & $D_{target}$ & $D_{cost}$ & $D_{target}$ \\
\midrule
\xmark    & 0.101 & 0.249 & \textbf{0.001} & \textbf{0.000} & 0.051 & 0.288 & 0.007 & 0.031 \\
\cmark & \textbf{0.038} & \textbf{0.006} & 0.003 & 0.001 & \textbf{0.001} & \textbf{0.004} & \textbf{0.007} & \textbf{0.008} \\
\bottomrule
\end{tabular}
}
% \vspace{-8pt}
\end{table*}

\subsection{Characterization of Smoothing Operator} \label{sec:exp_smoothing_characterization}
% {\color{red}Here kernel conditions refer to the effective two-mode Hilbert space $H_2=\operatorname{span}\{\sin(\pi t),\sin(2\pi t)\}$. 
% Finite-dimensional sufficiency in Theorem~\ref{25.12.09.14.29} applies in $H_2$, not to a finite-rank perturbation in the full ambient $L^2$ space.
% Within $H_2$, Theorem~\ref{25.12.09.14.29} gives an equivalence: the smoothed source is regular if and only if its joint projection onto the covariance kernel is regular. 
% In general, regularity of separate one-dimensional marginals is not a substitute for this joint condition.}

% {\color{red}Here kernel conditions refer to the effective two-mode Hilbert space $H_2=\operatorname{span}\{\sin(\pi t),\sin(2\pi t)\}$.
% In this space, Theorem~\ref{25.12.09.14.29} states that smoothing yields a regular source exactly when its projection onto the covariance kernel is regular.
% The same theorem applies when these finite-rank covariances are extended by zero on $H_2^\perp$ to the full $L^2$ space; their kernels then also include $H_2^\perp$.
% Appendix~\ref{sec:revisit_example} explains the resulting full-space condition and the joint-projection condition.}

Theorem~\ref{25.12.09.14.29} states that smoothing yields a regular source exactly when its projection onto the covariance kernel is regular. This provides a criterion for choosing the smoothing directions that make the source measure regular. We assess how this criterion affects HiSNOT in the effective two-mode space $H_2=\operatorname{span}\{\sin(\pi t),\sin(2\pi t)\}$. The theorem also applies when these finite-rank covariances are extended by zero on $H_2^\perp$ to the full $L^2$ space; their kernels then additionally contain $H_2^\perp$. Appendix~\ref{sec:revisit_example} explains the resulting full-space condition and the joint-projection condition.

To empirically validate this, we compare (1) theoretically appropriate Gaussian smoothing against (2) inappropriate Gaussian smoothing using the Perpendicular dataset. 
In this setting, the source measure $\mu$ has a regular marginal along $\sin(\pi t)$ and is concentrated at zero along $\sin(2\pi t)$. 
We evaluate two Gaussian measures:
\begin{itemize}[leftmargin=*, itemsep=-1pt]
    \item \textbf{Appropriate smoothing} ($\gamma_{1}$) such that $\text{Ker}(Q_{1}) = \text{span}\{ \sin (\pi t) \}$. This operator avoids smoothing only in the direction where $\mu$ is already regular, thereby satisfying the theorem.
    \item \textbf{Inappropriate smoothing} ($\gamma_{2}$) such that $\text{Ker}(Q_{2}) = \text{span}\{ \sin (2\pi t) \}$. This operator fails to apply smoothing to a direction where $\mu$ is singular, thus failing to produce a regular measure.
\end{itemize}

Figure \ref{fig:smoothing_comparison} presents the qualitative comparison of these two smoothings. With appropriate smoothing (left), the \OURS model successfully captures the target data distribution, as observed in the experiments with smoothing on all retained modes (Figure \ref{fig:synthetic}). 
However, under the inappropriate smoothing (right), the \OURS exhibits the spurious solution patterns. 
These results illustrate the finite-dimensional kernel criterion in Theorem~\ref{25.12.09.14.29}. 
They do not test an unrestricted infinite-dimensional converse.
For these diagonal two-mode examples, smoothing must include the unsmoothed singular direction.

\begin{figure}[t]
    \centering

    \begin{subfigure}[b]{0.33\linewidth}
        \centering
        \includegraphics[width=\textwidth]{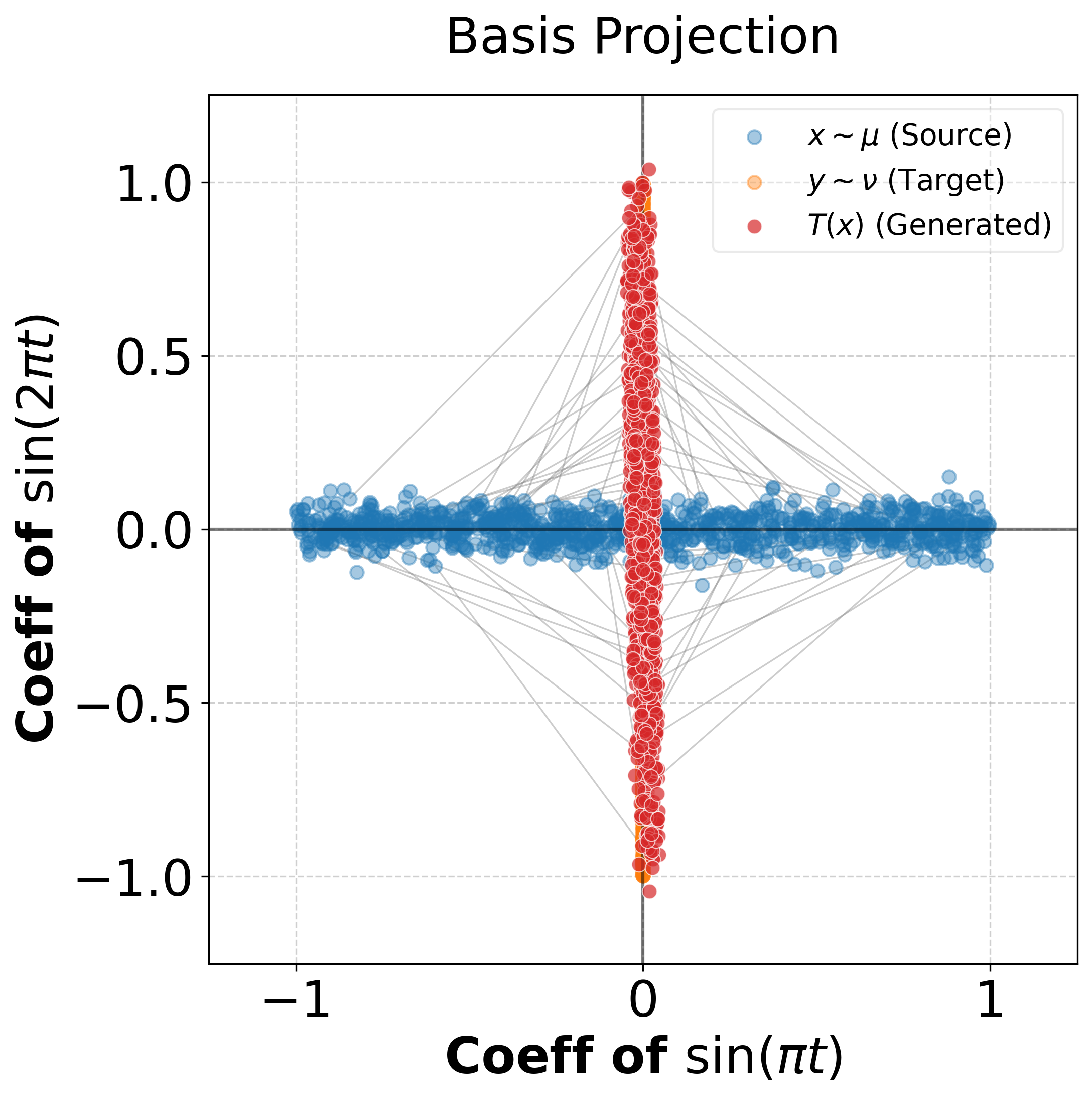}
    \end{subfigure}
    \qquad \qquad
    \begin{subfigure}[b]{0.33\linewidth}
        \centering
        \includegraphics[width=\textwidth]{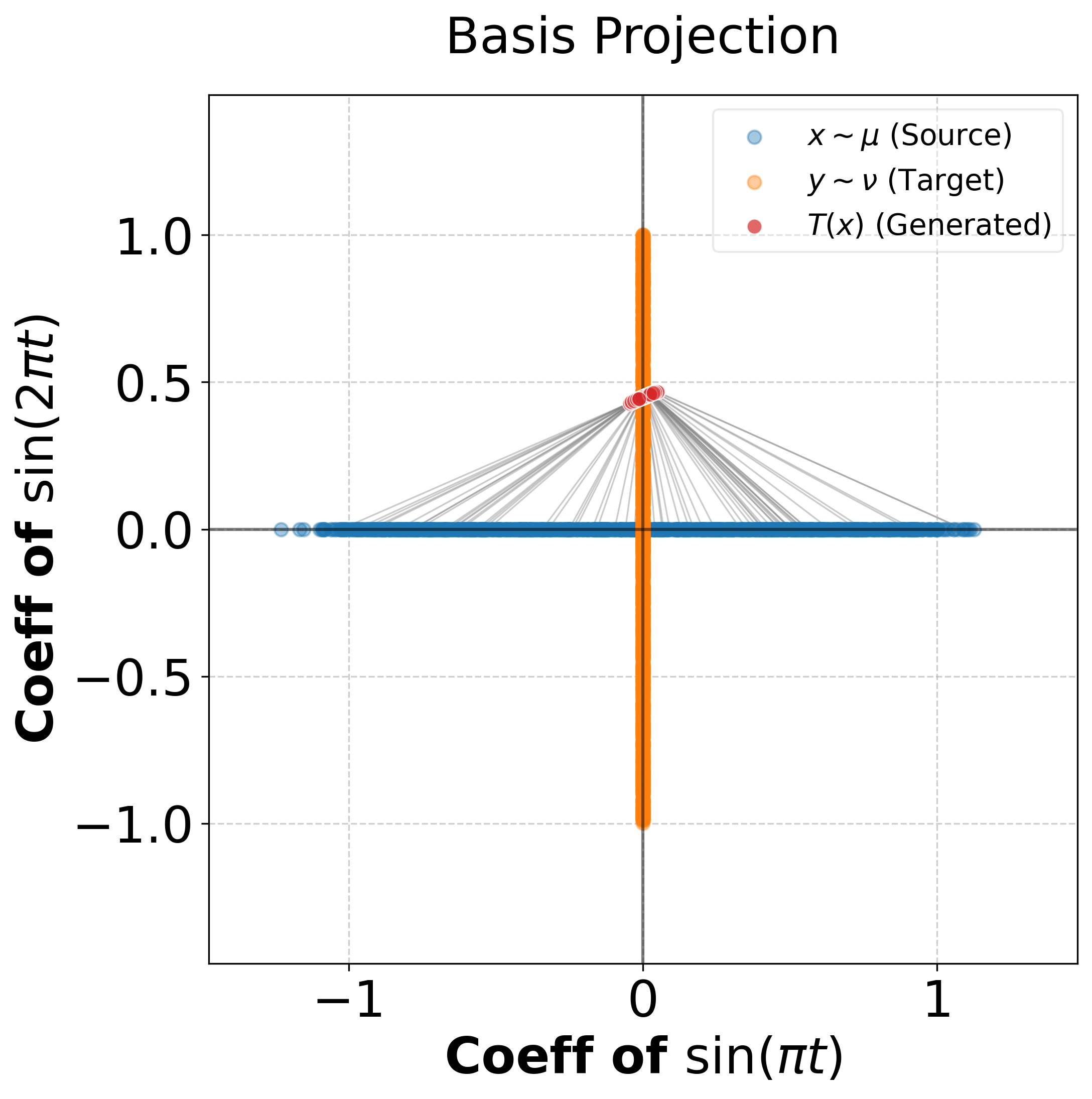}
    \end{subfigure}
    
    \caption{\textbf{Smoothing directions in the effective two-mode space.} \textbf{Left:} Smoothing the second mode makes the source absolutely continuous in $H_2$. \textbf{Right:} Smoothing only the first mode leaves a singular kernel marginal; the trained map exhibits a spurious-solution pattern.}
    \label{fig:smoothing_comparison}
    \vspace{-8pt}
\end{figure}

\subsection{Time Series Imputation Benchmarks} \label{sec:exp_timeseries}

\begin{table*}[t]
\centering
\small
\caption{
\textbf{Unpaired time-series imputation performance (MSE and MAE)} across five datasets. Results are averaged over missing ratios of $0.5$ and $0.7$. The \textbf{bold} and \underline{underlined} values denote the best and second-best results. All baseline results are taken from \citet{wang2025optimal}.
}
\label{table:TSI_average}
\vspace{-2pt}
\scalebox{0.98}{
\begin{tabular}{l cc cc cc cc cc}
\toprule
\textbf{Datasets} & \multicolumn{2}{c}{\textbf{ETTh1}} & \multicolumn{2}{c}{\textbf{ETTh2}} & \multicolumn{2}{c}{\textbf{ETTm1}} & \multicolumn{2}{c}{\textbf{ETTm2}} & \multicolumn{2}{c}{\textbf{Exchange}} \\
\cmidrule(r){2-3} \cmidrule(r){4-5}  \cmidrule(r){6-7} \cmidrule(r){8-9} \cmidrule(r){10-11}  
\textbf{Methods} & \textbf{MSE} & \textbf{MAE} & \textbf{MSE} & \textbf{MAE} & \textbf{MSE} & \textbf{MAE} & \textbf{MSE} & \textbf{MAE} & \textbf{MSE} & \textbf{MAE} \\
\midrule
Transformer \citep{vaswani2017attention}   & 0.338 & 0.410 & 0.282 & 0.383 & 0.085 & 0.196 & 0.056 & 0.161 & 0.320 & 0.264 \\
DLinear \citep{zeng2023transformers}  & \underline{0.187} & 0.307 & 0.138 & 0.262 & 0.131 & 0.250 & 0.107 & 0.227 & 0.288 & 0.238 \\
TimesNet \citep{wu2022timesnet}      & 0.343 & 0.415 & 0.147 & 0.278 & 0.086 & 0.209 & 0.093 & 0.220 & 0.372 & 0.308 \\
FreTS \citep{freTS}        & 0.233 & 0.355 & 0.159 & 0.278 & 0.061 & 0.168 & 0.040 & 0.136 & 0.229 & 0.189 \\
PatchTST \citep{patchTST}     & 0.216 & 0.338 & 0.144 & 0.272 & \textbf{0.055} & 0.153 & 0.033 & 0.122 & 0.227 & 0.188 \\
SCINet \citep{scinet}       & 0.195 & 0.321 & 0.172 & 0.288 & 0.088 & 0.204 & 0.085 & 0.211 & 0.326 & 0.269 \\
iTransformer \citep{itransformer}  & 0.206 & 0.321 & 0.125 & 0.237 & 0.068 & 0.173 & 0.043 & 0.142 & 0.087 & 0.070 \\
SAITS \citep{saits}        & 0.286 & 0.358 & 0.236 & 0.303 & 0.079 & 0.188 & 0.059 & 0.160 & 1.007 & 0.832 \\
CSDI \citep{csdi}         & \underline{0.187} & 0.303 & 0.107 & 0.271 & 0.156 & 0.208 & 0.161 & 0.150 & 0.134 & 0.150 \\
Sinkhorn \citep{sinkhorn}  & 0.936 & 0.685 & 0.877 & 0.644 & 0.965 & 0.717 & 0.927 & 0.690 & 0.783 & 0.648 \\
TDM \citep{TDM}          & 0.991 & 0.749 & 0.998 & 0.745 & 1.003 & 0.756 & 0.998 & 0.748 & 0.969 & 0.801 \\
PSW-I \citep{wang2025optimal}        & \textbf{0.160} & \textbf{0.262} & \underline{0.053} & \underline{0.155} & \underline{0.057} & \textbf{0.146} & \underline{0.025} & \underline{0.103} & \underline{0.036} & \textbf{0.030} \\
\midrule
HiSNOT (Ours)  & 0.227 & \underline{0.296} & \textbf{0.050} & \textbf{0.149} & 0.067 & \underline{0.152} & \textbf{0.022} & \textbf{0.092} & \textbf{0.006} & \underline{0.048} \\
\bottomrule
\end{tabular}
}
\vspace{-8pt}
\end{table*}

We evaluate \OURS with annealed smoothing on unpaired time series imputation benchmarks. Time-series data is one of the most widely used benchmarks for machine learning models operating on functional data \citep{ramsay2005functional, jiang2025deepfrc}.
For time-series data, (possibly irregular) observations are regarded as discrete samples from underlying continuous processes defined in a Hilbert space.
Here, we treat the imputation task as an optimal transport problem between the distribution of partially observed sequences (source) and fully observed paths (target). 

% {\color{red}Low-rank structure and incomplete observations motivate smoothing, but the manifold hypothesis alone does not establish population non-regularity for these datasets. }
% {\color{red}Low-rank models and incomplete observations motivate smoothing of the source distribution. 
% This experiment assesses practical imputation performance on discretized functional observations;
% Appendix~\ref{app:regularity_examples} discusses population regularity.}
% This experiment assesses practical performance on discretized functional observations.
% We report mean squared error (MSE) and mean absolute error (MAE), following \citet{wang2025optimal}.

Real-world time series are commonly modeled as lying near a low-dimensional manifold, and partial observation further restricts the data to a lower-dimensional structure. The source distribution of partially observed sequences is therefore expected to be singular, which makes standard SNOT prone to spurious solutions. HiSNOT with annealed smoothing is designed for exactly this regime. This experiment assesses its practical utility on high-dimensional, real-world functional data given as discretized observations; a rigorous discussion of population regularity for such data is deferred to Appendix~\ref{app:regularity_examples}. We report mean squared error (MSE) and mean absolute error (MAE), following \citet{wang2025optimal}.

Table~\ref{table:TSI_average} shows the lowest MSE and MAE on ETTh2 and ETTm2, and the lowest MSE on Exchange. On Exchange MAE, PSW-I performs better (0.030 versus 0.048). Performance on the other datasets depends on the metric.
Exchange MSE is 0.006 compared with 0.036 for the next best baseline, PWS-I, a sixfold reduction in the reported values.
Compared to classical OT baselines such as Sinkhorn and TDM, \OURS provides a significant improvement in accuracy, achieving better or comparable to previous state-of-the-art methods \citep{wang2025optimal}. Detailed results for each missing ratio (0.5 and 0.7) are reported in Appendix~\ref{app:TSI}.

We additionally evaluate Gaussian-to-Exchange function-level generation using sliding windows and compare $1{,}000$ generated windows with $1{,}000$ target windows through $W_2$. 
HiSNOT has the lower reported central value in five of six window-size/stride settings, while SNOT has the lower value for window size $720$ and stride $1$. 
The complete results are in Appendix~\ref{app:exchange_generation}. This finite-observation experiment does not establish resolution-independent numerical convergence.

Additional ablations in Appendix~\ref{app:smoothing_ablations} test smoothing support and both schedule endpoints. 
In a controlled $H_5$ example, omitting a required smoothing direction gives $D_{\mathrm{target}}=0.2063$, while all three tested supports containing both required directions give $0.0249$--$0.0300$. 
The initial- and terminal-level ablations yield $D_{\mathrm{target}}<0.01$ at every tested positive setting. 
These results support robustness within the tested ranges; the regularity comparison for the support ablation is relative to $H_5$.

\section{Conclusion}
\label{sec:limitations_conclusion}
We introduced HiSNOT, a Hilbert-space extension of semi-dual neural optimal transport for functional data.
Regular source measures restore pointwise inner-minimizer identifiability for admissible optimal potentials. 
Gaussian smoothing obeys a necessary kernel condition, with a converse under Cameron--Martin compatibility; injective trace-class covariance regularizes arbitrary sources. The induced exact optimal plans are consistent as smoothing vanishes.
Finite-mode synthetic experiments, imputation benchmarks, and Exchange generation show the empirical utility and setting-dependent benefits of the approach.

\paragraph{Limitations.}
The analysis concerns static quadratic OT on separable Hilbert spaces and exact population optimizers. 
It establishes plan-level accumulation-point consistency, not convergence to a unique limiting map. 
A kernel marginal alone is insufficient for arbitrary infinite-dimensional degenerate smoothing. 
Finite-mode discretization, finite-sample neural optimization, statistical estimation of smoothing directions, and optimal noise schedules require separate guarantees. 
We do not prove resolution-independent convergence, dynamic OT or gradient-flow results, or extensions to general Banach or metric spaces. 
Appendix~\ref{app:revision_context} gives operational interpretations and scope.

% Acknowledgements should only appear in the accepted version.

\begin{ack}
J.-H. Choi is partially supported by the National Research Foundation of Korea (NRF) grant funded by the Korea government (MSIT) (No. RS-2026-25488663). J. Yoon and J. W. Choi were partially supported by the National Research Foundation of Korea (NRF) grant funded by the Korea government (MSIT) (RS-2024-00349646, RS-2026-25506577), National Institute for Mathematical Sciences (NIMS) grant funded by the Korea government (MSIT) (No. S26120000), and BK21 FOUR Project (X-AI Driven Immersive Learning Innovation Research and Education Group). D. Kwon is partially supported by the National Research Foundation of Korea (NRF) grant funded by the Korea government (MSIT) (No. RS-2023-00252516, No. RS-2024-00408003,  No. RS-2026-25488663, and No. RS-2026-25613008), the POSCO Science Fellowship of POSCO TJ Park Foundation, and the Korea Institute for Advanced Study. This research was supported by the Yonsei University Research Fund of 2026-22-0251. J. W. Choi and D. Kwon thank the Center for Advanced Computation in KIAS for providing computing resources.

\paragraph{Competing interests.}
Other than the funding and support disclosed above, the authors report no financial relationships that could be perceived as influencing this work.
\end{ack}

\bibliographystyle{plainnat}
\bibliography{my_bib}
%%%%%%%%%%%%%%%%%%%%%%%%%%%%%%%%%%%%%%%%%%%%%%%%%%%%%%%%%%%%

\newpage
\appendix

\section{Gaussian Null Sets and G\^ateaux Differentiability}
\label{app:def}
We record the definitions used in Section~\ref{sec:ana}.

\begin{definition}[Gaussian null set]
\label{def:gaussian_null_set}
Let $H$ be a separable Hilbert space, and let $\mathcal P(H)$ denote the set of Borel probability measures on $H$.

\begin{enumerate}[leftmargin=*, label=(\roman*)]
    \item 
    A measure $\gamma\in\mathcal P(H)$ is called a \emph{Gaussian measure} if, for every $h\in H$, the pushforward measure
    \[
        (\langle h,\cdot\rangle_H)_\#\gamma
    \]
    is a Gaussian measure on $\mathbb R$.

    \item 
    A Gaussian measure $\gamma$ is called \emph{nondegenerate} if, for every $h\in H\setminus\{0\}$, the one-dimensional Gaussian measure
    \[
        (\langle h,\cdot\rangle_H)_\#\gamma
    \]
    has strictly positive variance. Equivalently, for every $h\in H\setminus\{0\}$, there exist $m_h\in\mathbb R$ and $\sigma_h>0$ such that, for every interval $(a,b)\subset\mathbb R$,
    \[
    \gamma\bigl(\{x\in H:\ a<\langle h,x\rangle_H<b\}\bigr)
    =
    \frac{1}{\sqrt{2\pi\sigma_h^2}}
    \int_a^b
    \exp\left(
        -\frac{(t-m_h)^2}{2\sigma_h^2}
    \right)
    \,\mathrm dt .
    \]

    \item 
    A Borel set $N\in\mathcal B(H)$ is called a \emph{Gaussian null set} if
    \[
        \gamma(N)=0
    \]
    for every nondegenerate Gaussian measure $\gamma$ on $H$.
\end{enumerate}
\end{definition}

A probability measure $\mu\in\mathcal P(H)$ is called regular if it assigns zero mass to every Gaussian null set; see Definition~\ref{26.01.15.13.31}.

Because standard calculus is insufficient for infinite-dimensional spaces, we employ the \textit{G\^ateaux differential}. This generalizes the directional derivative to Banach spaces, providing the necessary mathematical tools
to rigorously capture functional variations.

\begin{definition}[G\^ateaux differentiability]
\label{def:gateaux_derivative}
Let $H$ be a Hilbert space, and let $J:H\to\mathbb R$ be a real-valued functional.
For $x,y\in H$, the directional derivative of $J$ at $x$ in the direction $y$ is defined by
\[
    \partial_y J(x)
    :=
    \lim_{\varepsilon\to0}
    \frac{J(x+\varepsilon y)-J(x)}{\varepsilon},
\]
provided the limit exists.

We say that $J$ is \emph{G\^ateaux differentiable} at $x$ if there exists a continuous linear functional
$DJ(x)\in H^*$ such that
\[
    DJ(x)[y]
    =
    \partial_y J(x)
    \qquad
    \text{for every }y\in H.
\]
Since $H$ is a Hilbert space, we identify $DJ(x)$ with its Riesz representative, denoted by
$\nabla J(x)\in H$, so that
\[
    DJ(x)[y]
    =
    \langle \nabla J(x),y\rangle_H
    \qquad
    \text{for every }y\in H.
\]
\end{definition}

\section{Proof of Theorem~\ref{thm:uniqueness}}
\label{26.05.04.01.17}

\paragraph{Admissible dual potentials.}
\label{def:admissible_dual_potential}
Throughout this proof, a dual potential is represented by a Borel
$c$-concave function $V:H\to[-\infty,\infty)$ with
$V\in L^1(\nu)$ and $u:=V^c\in L^1(\mu)$, where
\[
 u(x)=\inf_{y\in H}\{c(x,y)-V(y)\},
 \qquad c(x,y)=\tfrac12\|x-y\|_H^2.
\]
The transform is defined using this representative on all of $H$.
A Kantorovich potential is an admissible potential attaining
$\int_H u\,\mathrm d\mu+\int_H V\,\mathrm d\nu=C(\mu,\nu)$.
The objective $\mathcal L(V,T)$ is interpreted with
$c(x,T(x))-V(T(x))$ as a single integrand. Its negative part is
integrable because it is bounded below by $u\in L^1(\mu)$; thus
$\mathcal L(V,T)\in\mathbb R\cup\{+\infty\}$ is well defined for
every measurable map $T:H\to H$.

\begin{proposition}[Optimal maps and dual attainment]
\label{25.12.16.14.13}
Let $H$ be a separable Hilbert space, $\mu\in\mathcal P_2^r(H)$,
and $\nu\in\mathcal P_2(H)$.
For the quadratic cost, there is a unique optimal plan,
\[
 \pi^\star=(\operatorname{Id},T^\star)_\#\mu,
\]
where $T^\star\in L^2(H,\mu;H)$ is unique $\mu$-almost everywhere.
An admissible Kantorovich potential exists. For every such potential
$V^\star$, with $u=(V^\star)^c$, one has
\begin{equation}
\label{25.12.10.11.02}
 u(x)+V^\star(T^\star(x))=c(x,T^\star(x))
 \quad\text{for }\mu\text{-almost every }x.
\end{equation}
\end{proposition}
\begin{proof}
Existence and uniqueness of the optimal map follow from
\citet[Theorem~6.2.10]{ambrosio2005gradient}.
Dual attainment follows from
\citet[Theorem~6.1.5]{ambrosio2005gradient}, whose integrability
conditions hold because
$c(x,y)\leq\|x\|_H^2+\|y\|_H^2$ and both marginals have finite
second moments. For any dual maximizer, the nonnegative slack
$c(x,y)-u(x)-V^\star(y)$ has integral zero under $\pi^\star$.
Consequently it vanishes $\pi^\star$-almost everywhere, which gives
the stated identity. In particular,
\begin{equation}
\label{26.01.26.12.35}
 T^\star(x)\in\operatorname*{arg\,min}_{y\in H}
 \{c(x,y)-V^\star(y)\}
 \quad\text{for }\mu\text{-almost every }x.
\end{equation}
\end{proof}

\subsection{Uniqueness of the pointwise minimizer}
\label{26.01.15.16.21}
\begin{proof}[Proof of Theorem~\ref{thm:uniqueness}(i)]
Write $q(x)=\tfrac12\|x\|_H^2$, $u=(V^\star)^c$, and
\[
 \Phi(x):=q(x)-u(x)
 =\sup_{y:\,V^\star(y)>-\infty}
 \{\langle x,y\rangle_H-q(y)+V^\star(y)\}.
\]
Choose $x_0,y_0\in H$ such that $u(x_0)$ and $V^\star(y_0)$ are
finite; such points exist by admissibility. For every integer
$n>\|y_0\|_H$, define the truncated convex function
\[
 \Phi_n(x):=
 \sup_{\substack{\|y\|_H\leq n\\V^\star(y)>-\infty}}
 \{\langle x,y\rangle_H-q(y)+V^\star(y)\}.
\]
The term corresponding to $y_0$ is a finite affine lower bound.
For every candidate $y$, the affine term is bounded above by
$\Phi(x_0)+n\|x-x_0\|_H$. Hence $\Phi_n$ is finite everywhere.
Moreover, all its affine slopes have norm at most $n$, so
\[
 |\Phi_n(x)-\Phi_n(z)|\leq n\|x-z\|_H
 \qquad(x,z\in H).
\]
By \citet[Theorem~6.2.3]{ambrosio2005gradient}, $\Phi_n$ is
G\^ateaux differentiable outside a Gaussian null set $N_n$.
The countable union $N:=\bigcup_{n>\|y_0\|_H}N_n$ is Gaussian null,
and therefore $\mu(N)=0$.

Fix $x\notin N$ for which $u(x)$ is finite and
Eq.~\ref{25.12.10.11.02} holds. Suppose that $y_1,y_2$ minimize
$c(x,y)-V^\star(y)$. Choose an integer $n>\|y_0\|_H$ containing
both $y_1$ and $y_2$ in its closed ball. Each of their affine terms
equals $\Phi(x)$, so $\Phi_n(x)=\Phi(x)$. Thus, for $i=1,2$,
\[
 \Phi_n(z)\geq\Phi_n(x)+\langle z-x,y_i\rangle_H
 \qquad(z\in H).
\]
Differentiating along positive and negative multiples of each
direction shows that $y_i=\nabla\Phi_n(x)$.
Therefore $y_1=y_2$. Since $T^\star(x)$ is a minimizer by
Eq.~\ref{26.01.26.12.35}, the argmin is exactly $\{T^\star(x)\}$.
This establishes the claim without requiring bounded target support
or differentiability of the untruncated potential $u$.
\end{proof}

\subsection{Uniqueness of the integrated inner minimizer}
\begin{proof}[Proof of Theorem~\ref{thm:uniqueness}(ii)]
For every measurable $T:H\to H$, define the nonnegative slack
\[
 d_T(x):=c(x,T(x))-V^\star(T(x))-(V^\star)^c(x)
\]
on the full-measure set where $(V^\star)^c$ is finite.
Admissibility and dual optimality give
\begin{equation}
\label{eq:lower_bound}
 \mathcal L(V^\star,T)-C(\mu,\nu)
 =\int_H d_T(x)\,\mu(\mathrm dx)\geq0.
\end{equation}
Equation~\ref{25.12.10.11.02} implies that $T^\star$ attains this
lower bound. If $T_{\mathrm{rec}}$ is an inner minimizer, then
$\int_H d_{T_{\mathrm{rec}}}\,\mathrm d\mu=0$.
Hence $d_{T_{\mathrm{rec}}}=0$ $\mu$-almost everywhere, so
$T_{\mathrm{rec}}(x)$ is a pointwise minimizer there.
Part~(i) yields $T_{\mathrm{rec}}=T^\star$ $\mu$-almost everywhere.
\end{proof}

\paragraph{Scope of the ideal objective.}
The argument optimizes over the admissible dual class and measurable
transport maps containing $T^\star$. It is an identifiability result
for the distributional objective; it does not assert exact recovery
by a fixed neural architecture or an empirical optimization algorithm.
For completeness, $(V^\star,T^\star)$ also satisfies the saddle
inequalities in this ideal class: the preceding proof gives
$\mathcal L(V^\star,T)\geq C(\mu,\nu)$ for every $T$, whereas
$T^\star_\#\mu=\nu$ implies
$\mathcal L(V,T^\star)=C(\mu,\nu)$ for every admissible $V$.

\section{Proof of Theorem~\ref{25.12.09.14.29}}
\label{26.05.04.00.24}

Use the orthonormal eigenbasis $(e_k)$ of $H$ from
Theorem~\ref{25.12.09.14.29}, with $Qe_k=\lambda_ke_k$,
$\lambda_k\geq0$, and $\sum_k\lambda_k<\infty$.
All sums run over the available eigenbasis indices, a finite set when $H$ is finite-dimensional.
For this appendix section only, abbreviate
$K=\operatorname{Ker}(Q)$, $S=K^\perp$,
$\pi_K=\pi_{\operatorname{Ker}(Q)}$, and $\mu_K=(\pi_K)_\#\mu$;
$\pi_S$ denotes the orthogonal projection onto $S$.
Then $H=S\oplus K$, and the Gaussian $\gamma_Q=\mathcal N(0,Q)$ is
concentrated on $S$ and nondegenerate when considered as a measure on $S$.
We identify $H$ with $S\times K$ through the isometry $(h,k)\mapsto h+k$.
Its Cameron--Martin space is
\[
 H_Q=\operatorname{Ran}(Q^{1/2})
 =\left\{h\in S:
       \sum_{k:\lambda_k>0}\frac{|\langle h,e_k\rangle_H|^2}{\lambda_k}<\infty
   \right\},
 \qquad
 \|h\|_{H_Q}^2=\sum_{k:\lambda_k>0}
       \frac{|\langle h,e_k\rangle_H|^2}{\lambda_k}.
\]
This is a Borel subspace of $H$ and a Hilbert space in its own norm.
Its closure in the ambient norm is $S$; this closure must not be
substituted for $H_Q$ in a translation argument.
For zero-dimensional factors, the unique probability measure is understood
to be Gaussian and nondegenerate, consistently with the definition using
nonzero continuous linear functionals.

\begin{lemma}[Cameron--Martin translations]
\label{25.12.08.19.00}
Let $(\gamma_Q)_h(A)=\gamma_Q(A-h)$ for $h\in H$.
Then $(\gamma_Q)_h$ and $\gamma_Q$ are equivalent if $h\in H_Q$,
and are mutually singular if $h\notin H_Q$.
In particular,
$ (\gamma_Q)_h\ll\gamma_Q $ if and only if $h\in H_Q$.
\end{lemma}
\begin{proof}
This is the Cameron--Martin theorem; see
\citet[Theorem~4.44]{hairer2009introduction}.
For completeness, its spectral argument also shows why the ambient closure
is insufficient.
If $h\notin S$, the measures are supported on the disjoint affine spaces
$S+h$ and $S$, so they are singular.
For $h\in H_Q$, let $h_k=\langle h,e_k\rangle_H$ and consider, for
$x\sim\gamma_Q$,
\[
 M_h(x)=\lim_{n\to\infty}\sum_{\substack{k\leq n\\\lambda_k>0}}
                      \frac{h_k\langle x,e_k\rangle_H}{\lambda_k}.
\]
The sum converges in $L^2(\gamma_Q)$ and almost surely, since its independent
centered Gaussian terms have summable variances.
Its distribution is $\mathcal N(0,\|h\|_{H_Q}^2)$.
The positive density
\[
 R_h(x)=\exp\bigl(M_h(x)-\tfrac12\|h\|_{H_Q}^2\bigr)
\]
has integral one. Computing its integrals against finite-coordinate
characteristic functions shows that $R_h\gamma_Q$ has the same coordinate
laws as $(\gamma_Q)_h$. Those coordinates generate the Borel sigma-algebra
on $S$, so $R_h\gamma_Q=(\gamma_Q)_h$. Positivity gives equivalence.

Finally suppose $h\in S\setminus H_Q$. Then
$a_n=\sum_{\substack{k\leq n\\\lambda_k>0}}h_k^2/\lambda_k\to\infty$.
For $a_n>0$, the functional
\[
 \ell_n(x)=a_n^{-1/2}\sum_{\substack{k\leq n\\\lambda_k>0}}
                        \frac{h_k\langle x,e_k\rangle_H}{\lambda_k}
\]
has distribution $\mathcal N(0,1)$ under $\gamma_Q$ and
$\mathcal N(\sqrt{a_n},1)$ under $(\gamma_Q)_h$.
For $B_n=\{x:\ell_n(x)>\sqrt{a_n}/2\}$, therefore,
$\gamma_Q(B_n)\to0$ and $(\gamma_Q)_h(B_n)\to1$.
Choose a subsequence along which both error probabilities are summable.
The Borel--Cantelli lemma implies that the limit inferior of these events
has measure zero under $\gamma_Q$ and one under $(\gamma_Q)_h$.
This proves singularity.
\end{proof}

\begin{lemma}[Gaussian sections and a regular factor]
\label{lem:gaussian_regular_product}
Let $\gamma$ be a nondegenerate Gaussian probability measure on $S$ and
$\eta\in\mathcal P^r(K)$.
Then $\gamma\otimes\eta$ is regular on $S\oplus K$.
More precisely, for any Gaussian-null Borel set $A\subset S\oplus K$,
the set
\[
 N_A=\{k\in K:\gamma(A_k)>0\},
 \qquad A_k=\{h\in S:(h,k)\in A\},
\]
is Gaussian null in $K$.
\end{lemma}
\begin{proof}
The function $k\mapsto\gamma(A_k)$ is Borel measurable.
For every nondegenerate Gaussian $\Gamma_K$ on $K$,
$\gamma\otimes\Gamma_K$ is a nondegenerate Gaussian on $S\oplus K$.
Hence, by the definition of Gaussian nullity and Fubini's theorem,
\[
 0=(\gamma\otimes\Gamma_K)(A)
   =\int_K\gamma(A_k)\,\Gamma_K(\mathrm dk).
\]
Nonnegativity gives $\Gamma_K(N_A)=0$.
Since $\Gamma_K$ was arbitrary, $N_A$ is Gaussian null in $K$.
Regularity of $\eta$ then yields $\eta(N_A)=0$, and another application
of Fubini gives $(\gamma\otimes\eta)(A)=0$.
\end{proof}

\begin{lemma}[Disintegration after Gaussian convolution]
\label{25.12.08.19.47}
Let $\mu\in\mathcal P(S\oplus K)$ and $\mu_K=(\pi_K)_\#\mu$.
There is a Borel probability kernel $k\mapsto\mu_k\in\mathcal P(S)$
such that $\mu(\mathrm dh,\mathrm dk)=\mu_k(\mathrm dh)\mu_K(\mathrm dk)$.
With $\gamma_Q$ regarded as a probability measure on $S$,
\[
 (\gamma_Q*\mu)(\mathrm dh,\mathrm dk)
     =(\gamma_Q*\mu_k)(\mathrm dh)\mu_K(\mathrm dk).
\]
\end{lemma}
\begin{proof}
Existence of the disintegration follows because $S$ and $K$ are Polish spaces;
see \citet[Theorem~5.3.1]{ambrosio2005gradient}.
Let $X=(X_S,X_K)\sim\mu$ and $G\sim\gamma_Q$ be independent.
Since $G$ takes values in $S$, adding it leaves the $K$-coordinate unchanged:
\[
 \pi_K(X+G)=X_K\quad\text{almost surely},
 \qquad (\pi_K)_\#(\gamma_Q*\mu)=\mu_K.
\]
Thus the smoothed measure has the same kernel marginal as $\mu$.

For a bounded Borel function $f$ on $S\times K$, independence and
disintegration give
\[
 \mathbb E f(X_S+G,X_K)
 =\int_K\int_S\int_S f(h+g,k)\,
            \gamma_Q(\mathrm dg)\mu_k(\mathrm dh)\mu_K(\mathrm dk).
\]
For each $k$, the two inner integrals are exactly integration against the
convolution $\gamma_Q*\mu_k$:
\[
 \int_S\int_S f(h+g,k)\,\gamma_Q(\mathrm dg)\mu_k(\mathrm dh)
 =\int_S f(h,k)\,(\gamma_Q*\mu_k)(\mathrm dh).
\]
Substituting this equality above shows that the conditional law of the
smoothed $S$-coordinate given $X_K=k$ is $\gamma_Q*\mu_k$ for
$\mu_K$-almost every $k$. Since the identity holds for every bounded Borel
$f$, it proves the claimed disintegration with the unchanged marginal $\mu_K$.
\end{proof}

\begin{lemma}[Nondegenerate smoothing]
\label{25.12.09.11.49}
If $\operatorname{Ker}(Q)=\{0\}$, then
$\gamma_Q*\mu\in\mathcal P^r(H)$ for every $\mu\in\mathcal P(H)$.
\end{lemma}
\begin{proof}
For each $x\in H$, the translated measure $\mathcal N(x,Q)$ is a
nondegenerate Gaussian on the whole space $H$.
If $A$ is Gaussian null, then $\mathcal N(x,Q)(A)=0$ for every $x$, and
\[
 (\gamma_Q*\mu)(A)=\int_H\mathcal N(x,Q)(A)\,\mu(\mathrm dx)=0.
\]
This argument requires no equivalence between translated Gaussian measures.
\end{proof}

\begin{proof}[Proof of Theorem~\ref{25.12.09.14.29}]
\textbf{Part~\ref{thm:kernel_necessity}: necessity without an additional assumption.}
Suppose $\gamma_Q*\mu$ is regular.
If $B\subset K$ is Gaussian null, then $A=\pi_K^{-1}(B)$ is Gaussian null
in $H$: the projection of every nondegenerate Gaussian on $H$ is
nondegenerate on $K$.
Since the added noise has zero $K$-coordinate,
\[
 \mu_K(B)=(\gamma_Q*\mu)(\pi_K^{-1}(B))=0.
\]
Thus $\mu_K$ is regular, proving necessity.

\textbf{Part~\ref{thm:kernel_equivalence}: equivalence under Cameron--Martin compatibility.}
Assume the weighted coefficient condition in~\eqref{eq:cm_coefficient_condition}.
For every $x\in H$, the projection $\pi_Sx$ lies in $S$ and has the same
coefficients as $x$ along every $e_k$ with $\lambda_k>0$. The spectral
description of $H_Q$ therefore gives the exact equivalence
\[
 \pi_Sx\in H_Q
 \quad\Longleftrightarrow\quad
 \sum_{k:\lambda_k>0}
       \frac{|\langle x,e_k\rangle_H|^2}{\lambda_k}<\infty.
\]
Thus the assumption is precisely $\mu(\pi_S^{-1}(H_Q))=1$.
It imposes no restriction on the source coefficients with $\lambda_k=0$
and requires only almost-sure finiteness, not a finite expectation of the sum.
The forward implication in~\eqref{eq:kernel_characterization} is
part~\ref{thm:kernel_necessity}.
For the reverse implication, suppose $\mu_K\in\mathcal P^r(K)$.
The spectral description of $H_Q$ above shows that it is Borel.
Disintegration therefore gives
\[
 1=\mu(\pi_S^{-1}(H_Q))
   =\int_K\mu_k(H_Q)\,\mu_K(\mathrm dk),
\]
so
$\mu_k(H_Q)=1$ for $\mu_K$-almost every $k$.
Let $A\subset H$ be Gaussian null.
Lemma~\ref{lem:gaussian_regular_product}, applied to $\gamma_Q$ on $S$,
gives $\gamma_Q(A_k)=0$ for $\mu_K$-almost every $k$.
Thus the Borel set
\[
 K_0=\{k\in K:\mu_k(H_Q)=1\text{ and }\gamma_Q(A_k)=0\}
\]
has $\mu_K(K_0)=1$.
For each $k\in K_0$, Lemma~\ref{25.12.08.19.00} gives
$\gamma_Q(A_k-h)=0$ for every $h\in H_Q$.
Consequently Lemma~\ref{25.12.08.19.47} yields
\[
 \begin{aligned}
 (\gamma_Q*\mu)(A)
   &=\int_K\int_S\gamma_Q(A_k-h)\,
                    \mu_k(\mathrm dh)\mu_K(\mathrm dk)\\
   &=0.
 \end{aligned}
\]
Indeed, for each $k\in K_0$ the inner integrand vanishes for every
$h\in H_Q$, and $\mu_k(H_Q)=1$; the complement of $K_0$ is
$\mu_K$-null. Hence $\gamma_Q*\mu$ is regular, proving the reverse
implication and therefore the equivalence.
The conditional source laws are used only on a $\mu_K$-full set; no statement
about their versions under an unrelated Gaussian measure is required.
If $Q$ has finite rank, $S$ is finite-dimensional and $H_Q=S$, so the
additional source condition is automatic.

\textbf{Part~\ref{thm:kernel_injective}: injective covariance.}
For $K=\{0\}$, Lemma~\ref{25.12.09.11.49} gives regularity for every
$\mu$ by integrating the nondegenerate Gaussian translates
$\mathcal N(x,Q)$. This does not invoke the additional assumption in
part~\ref{thm:kernel_equivalence}.
\end{proof}

The Cameron--Martin compatibility assumption in
part~\ref{thm:kernel_equivalence} is a sufficient hypothesis for the
converse, not a necessary condition for regularity of the convolution.

\begin{corollary}[Finite-rank covariance]
\label{cor:finite_rank_kernel}
If $Q$ has finite rank, then for every $\mu\in\mathcal P(H)$,
\[
 \gamma_Q*\mu\in\mathcal P^r(H)
 \quad\Longleftrightarrow\quad
 (\pi_{\operatorname{Ker}(Q)})_\#\mu
          \in\mathcal P^r(\operatorname{Ker}(Q)).
\]
In particular, this equivalence holds for any covariance on a
finite-dimensional ambient space.
\end{corollary}
\begin{proof}
Apply part~\ref{thm:kernel_equivalence} of the theorem with
$\operatorname{Ran}(Q^{1/2})=S$.
\end{proof}

\paragraph{Positive smoothing amplitudes.}
For every $\varepsilon>0$,
$\operatorname{Ker}(\varepsilon^2Q)=K$ and
$\operatorname{Ran}((\varepsilon^2Q)^{1/2})=H_Q$.
Thus the same sufficient assumptions apply to
$\mu*\mathcal N(0,\varepsilon^2Q)$ at every positive amplitude.
When $\mu\in\mathcal P_2(H)$, finite trace also ensures that the smoothed
measure has finite second moment.

\subsection{Why kernel marginal regularity alone is insufficient}
\label{app:kernel_counterexample}
The additional assumption in the degenerate-noise converse cannot be
omitted for general infinite-rank covariance, even when the kernel is
one-dimensional and the source has finite moments of every order.

Let $H=\ell^2(\mathbb N)\oplus\mathbb R$, write points as $(x,u)$, and set
\[
 v=(1/n)_{n\geq1},\qquad
 Q(x,u)=((x_n/n^2)_{n\geq1},0).
\]
This is a positive self-adjoint trace-class covariance with
$\operatorname{Tr}(Q)=\sum_n n^{-2}$ and
$K=\{0\}\oplus\mathbb R$.
For $U\sim\mathcal N(0,1)$, let
$X=(U^2v,U)$ and $\mu=\operatorname{Law}(X)$.
Then $\mu_K=\mathcal N(0,1)$ on $K$ is regular, and
\[
 \mathbb E\|X\|_H^2=3\sum_n n^{-2}+1<\infty.
\]
Let $(\xi_n)$ be independent standard Gaussian variables, independent of
$U$. The random element $G=((\xi_n/n)_{n\geq1},0)$ is $H$-valued almost
surely because its expected squared norm is finite, and has law $\gamma_Q$.
Hence $\rho=\gamma_Q*\mu$ is the law of
\[
 X+G=\bigl(((U^2+\xi_n)/n)_{n\geq1},U\bigr).
\]

For $N\geq1$ define the continuous linear functional
$L_N(x,u)=N^{-1}\sum_{n=1}^N nx_n$ and the Borel set
\[
 A=\{(x,u):\lim_{N\to\infty}L_N(x,u)=u^2\}.
\]
The strong law of large numbers gives $\rho(A)=1$.
We show that $A$ is Gaussian null, which proves that $\rho$ is not regular.

Let $\Gamma$ be any nondegenerate Gaussian on $H$, with arbitrary mean and
covariance, and let $(Z,V)\sim\Gamma$.
The last coordinate $V$ is a nondegenerate real Gaussian.
Set $Y_N=L_N(Z,V)$, so that each $(Y_N,V)$ is jointly Gaussian.
Suppose, toward a contradiction, that $\Gamma(A)>0$.
On $A$ the sequence $Y_N$ converges to $V^2$ and is eventually bounded.
Taking a countable union of possible bounds and initial indices shows that
there are $M>0$, $\delta>0$ and $N_0$ such that
\[
 \Gamma\bigl(A\cap\{\sup_{N\geq N_0}|Y_N|\leq M\}\bigr)\geq\delta.
\]
In particular, $\Gamma(|Y_N|\leq M)\geq\delta$ for every $N\geq N_0$.
Write $m_N=\mathbb E_\Gamma Y_N$ and
$\sigma_N^2=\operatorname{Var}_\Gamma(Y_N)$.
When $\sigma_N>0$, the Gaussian density bound gives
\[
 \delta\leq\frac{2M}{\sqrt{2\pi}\sigma_N},
 \qquad
 \sigma_N\leq C:=\frac{2M}{\delta\sqrt{2\pi}}.
\]
For $|m_N|>M$, Chebyshev's inequality further gives
$\delta\leq\sigma_N^2/(|m_N|-M)^2$.
The same resulting bounds hold directly for a degenerate $Y_N$.
Thus $|m_N|\leq M+C/\sqrt\delta$, and the tail sequence $(Y_N)$ is bounded
in the Hilbert space $L^2(\Gamma)$.

By the Hilbert-space Banach--Saks property, a subsequence $(Y_{N_j})$
has arithmetic means
$B_J=J^{-1}\sum_{j=1}^JY_{N_j}$ converging in $L^2(\Gamma)$ to a random
variable $Y$.
Each pair $(B_J,V)$ is jointly Gaussian.
Convergence in $L^2$ implies convergence of its mean vector and covariance
matrix and of its joint distribution, so $(Y,V)$ is jointly Gaussian too,
possibly degenerate.
On $A$, Cesaro convergence gives $B_J\to V^2$ pointwise.
An almost surely convergent subsequence of the $L^2$ convergence therefore
implies $Y=V^2$ on $A$ up to a $\Gamma$-null set, and hence
$\Gamma(Y=V^2)>0$.

Gaussian regression writes $Y=aV+b+R$, where $R$ is a centered Gaussian
independent of $V$.
If $\operatorname{Var}(R)>0$, conditioning on $V$ gives
$\Gamma(Y=V^2)=0$ because the conditional law of $Y$ is continuous.
If $\operatorname{Var}(R)=0$, equality requires $V^2-aV-b=0$.
This polynomial has at most two roots, to which the nondegenerate Gaussian
$V$ assigns probability zero.
Both cases are contradictory. Thus $\Gamma(A)=0$ for every nondegenerate
Gaussian $\Gamma$ on $H$, proving the claimed Gaussian nullity.

In this example,
\[
 H_Q=\{(h,0):h\in\ell^2,\ \sum_n n^2h_n^2<\infty\},
 \qquad (v,0)\notin H_Q,\quad (v,0)\in\overline{H_Q}=S.
\]
The projected source $\pi_SX=(U^2v,0)$ therefore lies outside the
Cameron--Martin space almost surely. Equivalently, the weighted sum in
\eqref{eq:cm_coefficient_condition} is
$\sum_n n^2(U^2/n)^2=\infty$ whenever $U\neq0$.
Although the kernel marginal is regular, a nonlinear relation with that
coordinate survives the smoothing in the infinitely many noisy coordinates.

\section{Plan-Level Consistency: Proof of Theorem~\ref{thm:convergence}}
\label{26.01.29.13.52}
We use narrow convergence with respect to the norm topology of $H$
and the product norm on $H\times H$ throughout this section.

\begin{lemma}[Convergence of Gaussian smoothing]
\label{lem:gaussian_smoothing}
Let $\mu\in\mathcal P_2(H)$, and let $Q$ be a positive
self-adjoint trace-class covariance operator. Then
\[
 W_2^2\bigl(\mu*\mathcal N(0,\varepsilon^2Q),\mu\bigr)
 \leq\varepsilon^2\operatorname{Tr}(Q)
 \longrightarrow0
 \qquad(\varepsilon\downarrow0).
\]
\end{lemma}
\begin{proof}
For independent $X\sim\mu$ and $G\sim\mathcal N(0,Q)$,
the pair $(X+\varepsilon G,X)$ is an admissible coupling.
Its expected squared displacement is
$\varepsilon^2\mathbb E\|G\|_H^2
=\varepsilon^2\operatorname{Tr}(Q)$.
\end{proof}

\begin{lemma}[Compactness and stability of couplings]
\label{26.04.04.12.17}
Suppose that $\mu_k\to\mu$ in $W_2(H)$ and
$\mu_k,\mu,\nu\in\mathcal P_2(H)$.
Then $\bigcup_{k\in\mathbb N}\Pi(\mu_k,\nu)$ is relatively compact
in $\mathcal P_2(H\times H)$ for the $W_2$ topology.
If $\pi_k\in\Pi(\mu_k,\nu)$ converges narrowly to $\pi$, then
$\pi\in\Pi(\mu,\nu)$ and $\pi_k\to\pi$ in $W_2(H\times H)$.
\end{lemma}
\begin{proof}
The $W_2$ convergence of $\mu_k$ implies that the family
$\{\mu_k\}_{k\in\mathbb N}$ is tight and has uniformly integrable
second moments; see
\citet[Proposition~7.1.5]{ambrosio2005gradient}.
The same properties hold for the singleton $\{\nu\}$.
Given compact sets $K,L\subset H$, every
$\pi\in\Pi(\mu_k,\nu)$ satisfies
\[
 \pi\bigl((K\times L)^c\bigr)
 \leq\mu_k(K^c)+\nu(L^c).
\]
Thus the entire union of coupling sets is tight.

For $a,b\geq0$ and $R>0$,
\[
 (a+b)\mathbf1_{\{a+b>R^2\}}
 \leq2a\mathbf1_{\{a>R^2/2\}}
      +2b\mathbf1_{\{b>R^2/2\}}.
\]
Indeed, if just one threshold is exceeded, that variable is the
larger one; if both are exceeded, the inequality is immediate.
Applying this with $a=\|x\|_H^2$ and $b=\|y\|_H^2$ gives,
uniformly over $\pi\in\Pi(\mu_k,\nu)$,
\begin{align*}
 &\int_{\{\|(x,y)\|_{H\times H}>R\}}
       \|(x,y)\|_{H\times H}^2\,\pi(\mathrm dx,\mathrm dy)\\
 &\quad\leq
 2\int_{\{\|x\|_H>R/\sqrt2\}}\|x\|_H^2\,\mu_k(\mathrm dx)
 +2\int_{\{\|y\|_H>R/\sqrt2\}}\|y\|_H^2\,\nu(\mathrm dy).
\end{align*}
Taking the supremum over $k$ and letting $R\to\infty$ proves
uniform integrability of the product second moments.
Tightness and this bound imply relative $W_2$ compactness by the
same proposition.

If $\pi_k$ converges narrowly to $\pi$, testing against bounded
continuous functions of either coordinate gives the marginals
$\mu$ and $\nu$. The uniform moment bound just established and
narrow convergence give $W_2(\pi_k,\pi)\to0$.
\end{proof}

\begin{proof}[Proof of Theorem~\ref{thm:convergence}]
The preceding lemma gives relative $W_2$ compactness of the optimal
plans $\pi_k^\star$. Let a subsequence converge in $W_2$ to $\pi$.
Its marginals are $\mu$ and $\nu$.
The continuous cost $c(x,y)=\tfrac12\|x-y\|_H^2$ has at most
quadratic growth, so convergence in $W_2(H\times H)$ implies
convergence of its integrals. Moreover, the triangle inequality gives
\[
 |W_2(\mu_k,\nu)-W_2(\mu,\nu)|\leq W_2(\mu_k,\mu)\longrightarrow0.
\]
Along the chosen subsequence we therefore obtain
\[
 \int_{H\times H}c(x,y)\,\pi(\mathrm dx,\mathrm dy)
 =\lim_k\int_{H\times H}c(x,y)\,\pi_k^\star(\mathrm dx,\mathrm dy)
 =\lim_k\tfrac12W_2^2(\mu_k,\nu)
 =\tfrac12W_2^2(\mu,\nu).
\]
Thus $\pi$ is optimal. This argument applies to every $W_2$
accumulation point.
\end{proof}

\paragraph{Application to smoothing.}
If $\mu_{\varepsilon_k}:=\mu*\mathcal N(0,\varepsilon_k^2Q)$ is
regular for every $k$ and $\varepsilon_k\downarrow0$, then
Lemma~\ref{lem:gaussian_smoothing} verifies the required marginal
convergence. Proposition~\ref{25.12.16.14.13} supplies the unique
optimal maps for these regularized sources, and the theorem applies
to their induced plans. Regularity is used for this map
representation; the compactness and optimal-plan stability argument
itself requires only the stated $W_2$ convergence of the marginals.

\section{Examples of Spurious Solutions in Non-Regular Settings}
\label{sec:fur}
These examples illustrate how the ideal semi-dual objective can have
spurious inner minimizers when the source is non-regular. Regularity
is a sufficient condition for identifiability, not a necessary
condition for every individual transport problem. We use globally
defined admissible potentials, with the inner minimization taken over
all of $H$. Smoothing restores the stated guarantees under a
sufficient condition of Theorem~\ref{25.12.09.14.29}.

\subsection{Setup and two singular transport problems}
\label{26.01.25.17.05}
Let $H=L^2((-1,1))$ and choose distinct positive integers $k_1,k_2$.
Set $e_1(t)=\sin(k_1\pi t)$ and $e_2(t)=\sin(k_2\pi t)$, which are
orthonormal in $H$. Write $U_i=\operatorname{span}\{e_i\}$,
$I=[-1,1]$, and let $m$ be the uniform probability measure on $I$.
The source is $\mu=\operatorname{Law}(Ae_1)$ with $A\sim m$.
It is non-regular in $H$ because it is concentrated on the proper
closed subspace $U_1$, a Gaussian null set.

In the perpendicular example, the target is
$\nu_1=\operatorname{Law}(Be_2)$ with $B\sim m$.
In the one-to-many example, it is
$\nu_2=\operatorname{Law}(Be_1+Re_2)$, where $B\sim m$ and the
independent sign $R$ is uniform on $\{-1,1\}$.
The first problem has many optimal Monge maps. The second has an optimal Kantorovich plan but no optimal deterministic transport map.
Valid deterministic transports do exist in the second problem, and their costs can approach the Kantorovich optimum.

\begin{figure}[t]
    \centering
    \includegraphics[width=.45\linewidth]{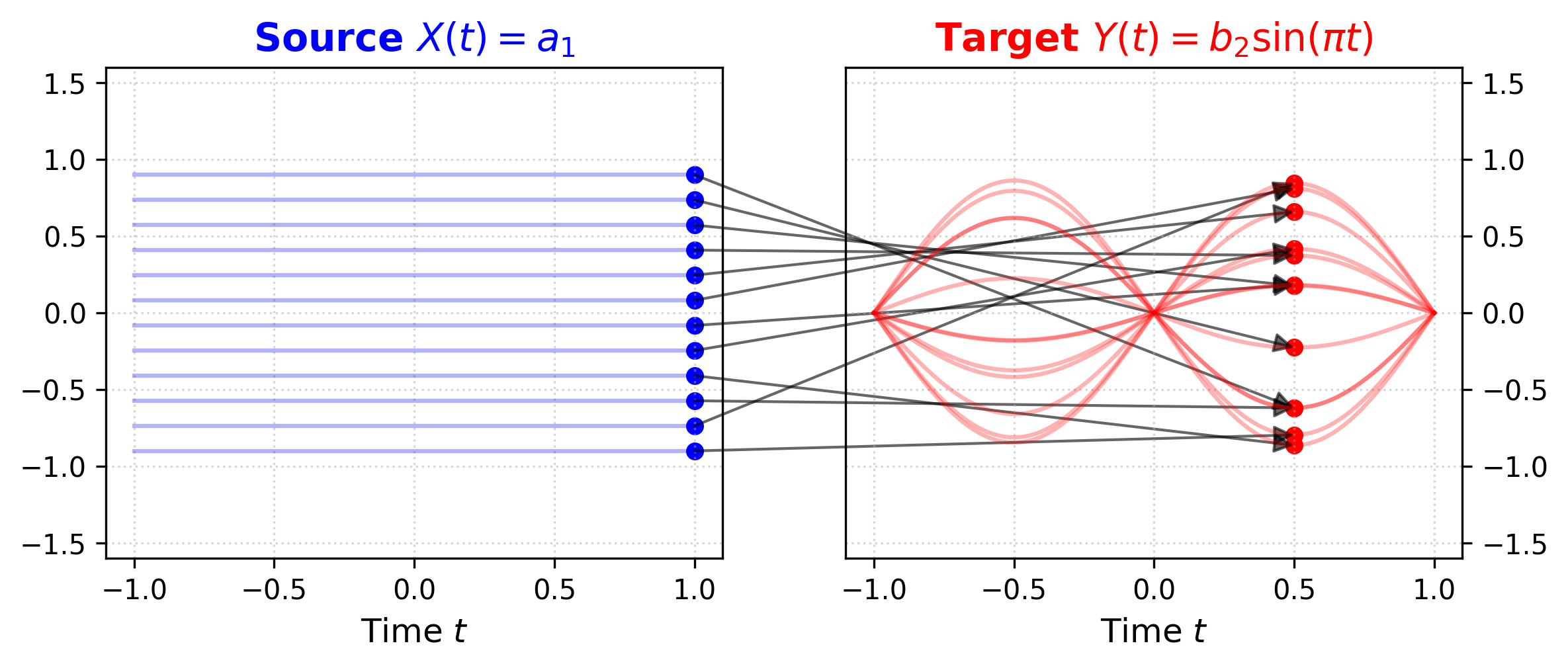}
    \ \ 
    \includegraphics[width=.45\linewidth]{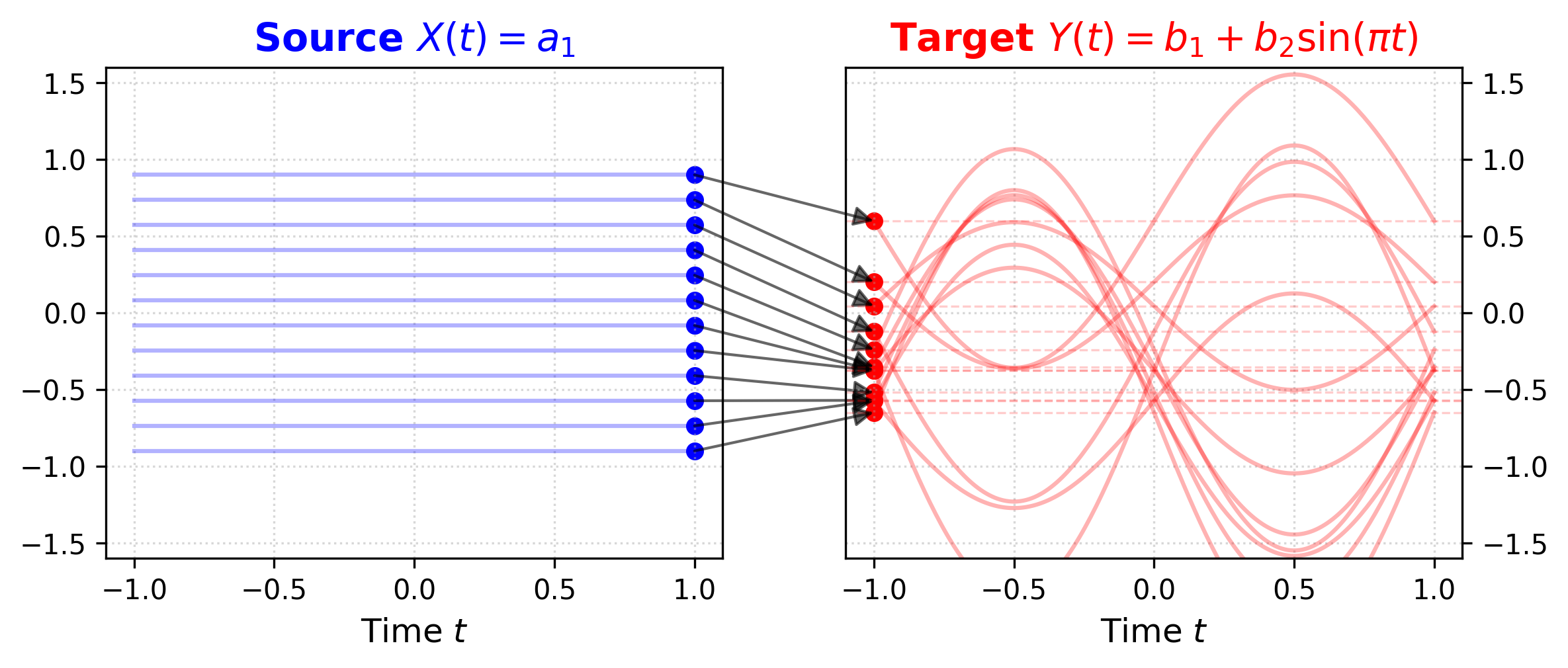}
    \caption{Illustrations of non-unique optimal transport maps in the perpendicular example.}
    \label{fig:example}
\end{figure}

\subsection{Details on Examples 1 and 2}
\label{app:example_details}
\subsubsection{Example 1: Perpendicular transport}
For $x=ae_1$ and $y=be_2$, orthogonality gives
$c(x,y)=\tfrac12(a^2+b^2)$. Consequently,
\begin{equation}
\label{eq:app_kantorovich}
 C(\mu,\nu_1)=\tfrac12\int_I a^2\,m(\mathrm da)
              +\tfrac12\int_I b^2\,m(\mathrm db)=\tfrac13,
\end{equation}
and every coupling is optimal. Each measure-preserving map
$\psi:I\to I$ gives an optimal Monge map
$T_\psi(ae_1)=\psi(a)e_2$, so the optimal map is not unique.

To exhibit spurious inner minimizers using the full-space objective,
write $q(y)=\tfrac12\|y\|_H^2$, let $\pi_{U_2^\perp}$ denote the orthogonal projection onto $U_2^\perp$, and choose
\[
 V_1(y):=q(y)-2\|\pi_{U_2^\perp}y\|_H\qquad(y\in H).
\]
This is $c$-concave: if
$D=\{z\in U_2^\perp:\|z\|_H\leq2\}$ and
$w(z)=q(z)$ on $D$, $w(z)=-\infty$ elsewhere, then
\[
 w^c(y)=q(y)-\sup_{z\in D}\langle z,y\rangle_H
       =q(y)-2\|\pi_{U_2^\perp}y\|_H=V_1(y).
\]
For $x=ae_1$, $a\in I$, and $p=\pi_{U_2^\perp}y$,
\[
 c(x,y)-V_1(y)=q(x)-\langle x,p\rangle_H+2\|p\|_H
 \geq q(x)+(2-|a|)\|p\|_H.
\]
Equality at the minimum holds exactly when $p=0$. Thus
\[
 V_1^c(x)=q(x),\qquad
 \operatorname*{arg\,min}_{y\in H}\{c(x,y)-V_1(y)\}=U_2
 \qquad(x\in Ie_1).
\]
Since $V_1=q$ on the target line, both dual terms are integrable
and their sum is $1/3$. Hence $V_1$ is an admissible Kantorovich
potential. Every measurable map with $T(x)\in U_2$ for
$\mu$-almost every $x$ minimizes
\begin{equation}
\label{26.01.29.17.45}
 \inf_{T:H\to H}\mathcal L(V_1,T),
\end{equation}
with value $C(\mu,\nu_1)$.
Many of these maps fail the marginal constraint; for example,
$T(x)=0$ is an inner minimizer but $T_\#\mu=\delta_0\ne\nu_1$.
The flat family therefore consists of line-valued maps, rather than
arbitrary $H$-valued maps.

\subsubsection{Example 2: One-to-many transport}
For any coupling between $x=ae_1$ and $y=be_1+re_2$,
\[
 c(x,y)=\tfrac12(a-b)^2+\tfrac12.
\]
The plan which keeps $b=a$ and independently chooses $r=\pm1$
with equal probabilities has cost $1/2$, the global lower bound.
Thus $C(\mu,\nu_2)=1/2$.

An optimal deterministic transport would have to preserve the first
coordinate, hence have the form
$T(ae_1)=ae_1+s(a)e_2$ with $s(a)\in\{-1,1\}$.
Let $E=\{a:s(a)=1\}$. The target marginal would require, for every
Borel $B\subset I$,
\[
 m(E\cap B)=\tfrac12m(B).
\]
Taking $B=I$ yields $m(E)=1/2$, while taking $B=E$ yields
$m(E)=\tfrac12m(E)$, a contradiction. No optimal deterministic
transport exists.

This is a failure of attainment, not a positive gap between the
Monge and Kantorovich infima. To see this, partition $I$ into $n$
equal cells. On the first half of each cell, map the first
coordinate affinely onto the whole cell and choose sign $+1$;
on the second half, do the same and choose sign $-1$.
This defines a valid transport $T_n$ to $\nu_2$. Its first-coordinate
displacement is at most $1/n$, so
\[
 \tfrac12<C(T_n):=\int_Hc(x,T_n(x))\,\mu(\mathrm dx)
 \leq\tfrac12+\frac{1}{2n^2}.
\]
Thus valid maps can approximate the optimal cost even though no
one of them attains it.

For the semi-dual problem, choose the global potential
\[
 V_2(y):=\tfrac12\|\pi_{U_2}y\|_H^2\qquad(y\in H).
\]
It is $c$-concave: $V_2=w^c$ for $w=0$ on $U_2^\perp$ and
$w=-\infty$ elsewhere. For $x\in U_1$,
\[
 c(x,y)-V_2(y)=\tfrac12\|x-\pi_{U_2^\perp}y\|_H^2,
 \qquad V_2^c(x)=0.
\]
On the target support $V_2=1/2$, so it is an admissible
Kantorovich potential with dual value $1/2$. Its full-space argmin is
\[
 \operatorname*{arg\,min}_{y\in H}\{c(x,y)-V_2(y)\}=x+U_2.
\]
Every measurable selector from these lines is an inner minimizer.
In particular, selectors $T(ae_1)=ae_1+s(a)e_2$ with
$s(a)\in\{-1,1\}$ lie on the target support and attain the
optimal semi-dual value, yet fail its marginal constraint as shown
above. This is a failure of exact inner-minimizer identifiability;
it does not rule out approximation by other deterministic maps or
learning a conditional plan with an explicitly stochastic model.

\subsection{Revisiting the examples with smoothing}
\label{sec:revisit_example}
Let $Q$ be an injective nonnegative self-adjoint trace-class
covariance on $H$. For every $\varepsilon>0$, the smoothed source
$\mu_\varepsilon=\mu*\mathcal N(0,\varepsilon^2Q)$ is regular by
Theorem~\ref{25.12.09.14.29}. Proposition~\ref{25.12.16.14.13}
and Theorem~\ref{thm:uniqueness} therefore give a unique optimal
map and a unique inner minimizer for each admissible optimal
potential. Theorem~\ref{thm:convergence} concerns the induced plans
as $\varepsilon\downarrow0$. In the one-to-many example, these
plans can approach the splitting optimal plan without the maps
converging to an optimal deterministic map for the original source.

The finite-mode experiments must be interpreted in their retained
space. In $H_2=\operatorname{span}\{e_1,e_2\}$, noise along $e_2$
leaves a regular source marginal on $U_1$ and restores regularity
in $H_2$. Noise only along $e_1$ leaves the unsmoothed $e_2$
marginal atomic and does not restore regularity there. A perturbation
supported on $H_2$ alone is not a regularizer of this source in the
entire infinite-dimensional $H$.

\subsection{Comparison with finite-dimensional results}
In finite dimensions, the source non-charging condition studied by
\citet{OTP} is more general than absolute continuity. Our
Gaussian-null regularity assumption supplies a sufficient
Hilbert-space condition; the examples above do not establish its
necessity for all transport problems. Extensions to weaker geometric
source assumptions are beyond the present analysis.

\section{Related Works} \label{app:related_works}
Neural Optimal Transport (Neural OT) learns the optimal transport maps between the source and target distributions using neural networks. A prominent family of approaches, called Semi-dual Neural OT (SNOT), builds on the semi-dual formulation of the OT problem \cite{otm, uotm, otmICNN, fanTMLR}. These methods differ in their transport map parameterization \cite{otm, fanTMLR}, the strict convexity constraints on the potential via ICNNs \cite{otmICNN}, and the OT formulation \cite{uotm}. SNOT has been successfully applied to various machine learning tasks in finite-dimensional Euclidean spaces, including generative modeling \citep{otm, fanTMLR, uotm, sjko} and inverse problems \citep{gazdieva2025optimal, lee2025isum}. 

Despite these successes, SNOT suffers from a fundamental limitation known as the \textbf{spurious solution problem}: not all solutions to the max-min objective $\mathcal{L}(V, T)$ (Eq. \ref{eq:snot}) recover the correct optimal transport map. Although the ground-truth Kantorovich potential and transport map $(V^{\star}, T^{\star})$ are always the max-min solution, the converse does not hold in general \citep{otm, fanTMLR}. In other words, even when the max-min objective is solved exactly, the recovered map $T_{rec}$ may fail to push the source measure $\mu$ to the target $\nu$, leading to incorrect transport maps.

\citet{fanTMLR} showed that, when the finite-dimensional source measure $\mu$ is atomless, the saddle point solution $(V_{sad}, T_{sad})$ of $\mathcal{L}_{V, T}$ satisfying 
\begin{equation} \label{eq:saddle_def}
    % \begin{aligned}
        V_{sad} \in \arg\max_{V} \mathcal{L} (V, T_{sad}), \qquad
        T_{sad} \in \arg\min_{T} \mathcal{L} (V_{sad}, T).    
    % \end{aligned}
\end{equation} 
is equivalent to the Kantorovich optimal potential and optimal transport map. Note that the saddle point condition is a stronger assumption than the max-min solution, and this recovery is not guaranteed for the generic max-min solution of $\mathcal{L}_{V, T}$.

\citet{OTP} established a sufficient condition for the max-min solution to recover optimal transport map in \textbf{finite-dimensional Euclidean spaces}. The sufficient condition is that the source measure $\mu$ must not assign positive mass to measurable sets of Hausdorff dimension at most $d-1$, where $d$ denotes the ambient dimension. To enforce this condition in practice, they proposed Gaussian smoothing of the source measure with an annealing scheme, and proved weak convergence of the annealed optimal transport plan $\pi_{\varepsilon}^{\star}$. \citet{chu2026rate} strengthened this weak convergence analysis by deriving quantitative convergence rates. Based on this, they proposed a convergence-rate-optimal noise annealing schedule. They also analyzed necessary smoothing directions in finite dimensions; our Hilbert-space analysis distinguishes the necessary kernel condition from the additional hypotheses needed for a converse.

We study inner-minimizer identifiability for regular Hilbert-space sources (Theorem~\ref{thm:uniqueness}), $W_2$ accumulation-point consistency of exact optimal plans (Theorem~\ref{thm:convergence}), and necessary and sufficient smoothing conditions with their precise scope (Theorem~\ref{25.12.09.14.29}). 
Finite-dimensional directional and rate analyses are complementary \citep{chu2026rate}.

Zhu et al.\citet{zhu2024functional} estimate Hilbert--Schmidt linear maps using the squared $2$-Wasserstein distance between the pushforward source and target distributions and a squared Hilbert--Schmidt norm penalty. 
Their finite-basis empirical formulation jointly optimizes the map coefficients and a coupling between source and target samples \citep[Eq.~(12)]{zhu2024functional}.
Our focus is the identifiability of the quadratic semi-dual inner minimizer and regularization of singular Hilbert-valued sources. These are population statements; the numerical models use discretized functional observations.

Kernel Neural Optimal Transport \citep{korotinkernel} studies weak costs $C(x,\pi(\cdot\mid x))$ and represents conditional laws with a stochastic map $T(x,z)$. A weak quadratic objective can fail to identify a target-preserving optimal conditional plan; a continuous characteristic kernel, under the paper's hypotheses, restores conditional-law identifiability. Its Hilbert feature space distinguishes distributions and need not be the data space. In contrast, we retain the standard pairwise cost $\frac12\|x-y\|_H^2$ on a Hilbert data space and analyze the deterministic pointwise selector $T(x)$. We regularize the source rather than replacing the transport cost. The weak-OT formulation is stated on general spaces, with Euclidean constructions and image experiments; it should not be described as excluding all non-Euclidean spaces.

Dynamic OT \citep{benamou2000computational} and Wasserstein gradient flows \citep{ambrosio2005gradient,mokrov2021large} concern evolving measures. The present theory solves a static endpoint problem and does not analyze that evolution; see Appendix~\ref{app:dynamic_scope}.

\section{Algorithm and Implementation Details} \label{app:implementation_detail}

This appendix collects the implementation details omitted from the main text.
We begin with a rigorous construction of the Hilbert-valued Gaussian perturbation used for spectral augmentation.
We then give examples of orthonormal bases for functional data, state the full HiSNOT training algorithm, and describe the experimental setups for the synthetic transport examples and time-series imputation benchmarks.

\subsection{Hilbert-Valued Gaussian Perturbations}
\label{app:hilbert_gaussian}
We provide the formal statement and proof for the Hilbert-valued Gaussian perturbation used in the spectral augmentation step.

\begin{proposition}[Hilbert-valued Gaussian perturbation]
\label{prop:hilbert_Gaussian_random_app}
Let $\{e_k\}_{k\in\mathbb N}$ be an orthonormal basis of $H$, and let
$\{\lambda_k\}_{k\in\mathbb N}\subset[0,\infty)$ satisfy
$$
    \sum_{k=1}^{\infty}\lambda_k<\infty.
$$
Let $\{\xi_k\}_{k\in\mathbb N}$ be i.i.d. standard Gaussian random variables. 
Then
\begin{equation}
\label{25.11.29.22.58}
    G
    :=
    \sum_{k=1}^{\infty}\sqrt{\lambda_k}\,\xi_k e_k
\end{equation}
defines an $H$-valued centered Gaussian random variable. 
Its covariance is the trace-class operator defined by $Qe_k=\lambda_k e_k$.
In particular, $G\sim\mathcal N(0,Q)$.
\end{proposition}

\begin{proof}
    Let $G_N := \sum_{k=1}^{N} \sqrt{\lambda_k} \xi_k e_k$ be the partial sum. We first show that $\{G_N\}_{N \in \mathbb{N}}$ is a Cauchy sequence in $L^2(\Omega; H)$. For $M > N$, we have:
    \begin{align*}
        \mathbb{E}\left[ \|G_M - G_N\|_H^2 \right] 
        &= \mathbb{E}\left[ \left\| \sum_{k=N+1}^{M} \sqrt{\lambda_k} \xi_k e_k \right\|_H^2 \right] \\
        &= \sum_{k=N+1}^{M} \lambda_k \mathbb{E}[\xi_k^2] \|e_k\|_H^2 \\
        &= \sum_{k=N+1}^{M} \lambda_k.
    \end{align*}
    Since $\{\lambda_k\}_{k\in\mathbb{N}}$ is summable (trace-class condition), the tail sum converges to 0 as $N, M \to \infty$. 
    Thus, $G_N$ converges to a limit $G$ in $L^2(\Omega; H)$. 
    Hence, there exists $G\in L^2(\Omega;H)$ such that $G_N\to G$ in $L^2(\Omega;H)$.
    
    To verify that $G$ is Gaussian, consider an arbitrary linear functional $\ell(h) = \langle h, u \rangle_H$ for $u \in H$. Then
    $$
    \langle G, u \rangle_H = \sum_{k=1}^{\infty} \sqrt{\lambda_k} \xi_k \langle e_k, u \rangle_H.
    $$
    For each $N$, $\langle G_N,u\rangle_H$ is a centered Gaussian random variable, and it converges in $L^2(\Omega)$ to $\langle G,u\rangle_H$. 
    Since the class of Gaussian laws is closed under convergence in distribution, $\langle G,u\rangle_H$ is Gaussian. 
    Since the variance 
    $$
    \sum_{k=1}^{\infty} \lambda_k \langle e_k, u \rangle^2 \le \left(\max_{k\in\mathbb{N}} \lambda_k\right) \|u\|_H^2 < \infty,
    $$ the limit is a univariate Gaussian random variable. Thus, $G$ is a Gaussian random vector.
    
    Finally, we identify the covariance operator $Q$. For any $u, v \in H$,
    \begin{align*}
        \mathbb{E}[\langle G, u \rangle \langle G, v \rangle] 
        &= \mathbb{E}\left[ \left(\sum_{j=1}^{\infty} \sqrt{\lambda_j} \xi_j \langle e_j, u \rangle\right) \left(\sum_{k=1}^{\infty} \sqrt{\lambda_k} \xi_k \langle e_k, v \rangle\right) \right] \\
        &= \sum_{j,k=1}^{\infty} \sqrt{\lambda_j \lambda_k} \delta_{jk} \langle e_j, u \rangle \langle e_k, v \rangle \\
        &= \sum_{k=1}^{\infty} \lambda_k \langle e_k, u \rangle \langle e_k, v \rangle \\
        &= \langle Q u, v \rangle,
    \end{align*}
    where $Q$ is the operator defined by $Q e_k = \lambda_k e_k$. 
    This completes the proof.
\end{proof}

\paragraph{Finite-mode coefficients and source approximation}

Fix $\varepsilon>0$, $\mu\in\mathcal P_2(H)$, and the trace-class covariance $Q$ and orthonormal basis of Proposition~\ref{prop:hilbert_Gaussian_random_app}.
Let $x\sim\mu$ be independent of the Gaussian coefficients, set $x_\varepsilon=x+\varepsilon G$, and write $\mu_\varepsilon=\operatorname{Law}(x_\varepsilon)$.
With $G_M=\sum_{k=1}^M\sqrt{\lambda_k}\xi_k e_k$, the noise truncation used in Algorithm~\ref{alg:hiSNOT} keeps the source signal and forms $x+\varepsilon G_M$. Coupling this sample with $x+\varepsilon G$ gives
$$
\begin{aligned}
W_2^2\!\left(\operatorname{Law}(x+\varepsilon G_M),\mu_\varepsilon\right)
&\leq\mathbb E\|\varepsilon(G-G_M)\|_H^2\\
&=\varepsilon^2\sum_{k>M}\lambda_k\longrightarrow0.
\end{aligned}
$$
The equality follows from orthonormality and the unit variances of the Gaussian coefficients; the tail vanishes because $Q$ is trace class.
Thus increasing $M$ approximates the smoothed source law for the same fixed $Q$ and $\varepsilon$. 
Under the sufficient conditions of Theorem~\ref{25.12.09.14.29}, including injective $Q$, this limiting law is regular on $H$.

For a representation that also projects the source onto $H_M=\operatorname{span}\{e_1,\ldots,e_M\}$, the retained coefficient vector and reconstructed function are
$$
\begin{aligned}
\mathbf c_{\varepsilon,M}
&:=\bigl(\langle x,e_k\rangle_H
       +\varepsilon\sqrt{\lambda_k}\xi_k\bigr)_{k=1}^M
  \in\mathbb R^M,\\
x_{\varepsilon,M}
&:=\sum_{k=1}^M(\mathbf c_{\varepsilon,M})_k e_k
  =\sum_{k=1}^M\langle x_\varepsilon,e_k\rangle_H e_k
  \in H_M.
\end{aligned}
$$
These constructions agree when $x\in H_M$ almost surely.
If $\lambda_k>0$ for all $k\leq M$, the added coefficient vector is an independent Gaussian with positive-definite covariance $\varepsilon^2\operatorname{diag}(\lambda_1,\ldots,\lambda_M)$.
Convolution therefore gives a joint density $f_{\varepsilon,M}$ for $\mathbf c_{\varepsilon,M}$ with respect to Lebesgue measure on $\mathbb R^M$:
$$
\begin{aligned}
\mathbb P(\mathbf c_{\varepsilon,M}\in A)
&=\int_A f_{\varepsilon,M}(z)\,\mathrm dz,
\qquad A\subseteq\mathbb R^M\ \text{Borel},\\
\operatorname{Law}(x_{\varepsilon,M})
&=\Bigl(z\mapsto\sum_{k=1}^M z_ke_k\Bigr)_\#
  \bigl(f_{\varepsilon,M}(z)\,\mathrm dz\bigr)
  \in\mathcal P_2^r(H_M).
\end{aligned}
$$
In infinite-dimensional $H$, this law gives unit mass to the Gaussian-null subspace $H_M$ and hence is non-regular as a law on $H$.

The reconstruction still approximates the full-space smoothed source in Wasserstein distance. 
Using the same $x$ and Gaussian coefficients couples $x_{\varepsilon,M}$ and $x_\varepsilon$. Parseval's identity and
independence, which makes the source--noise cross terms vanish, yield
$$
\begin{aligned}
W_2^2\!\left(\operatorname{Law}(x_{\varepsilon,M}),\mu_\varepsilon\right)
&\leq\mathbb E\|x_\varepsilon-x_{\varepsilon,M}\|_H^2\\
&=\sum_{k>M}\mathbb E|\langle x,e_k\rangle_H|^2
  +\varepsilon^2\sum_{k>M}\lambda_k
  \longrightarrow0.
\end{aligned}
$$
The source tail vanishes since $\mathbb E\|x\|_H^2<\infty$, and the noise tail vanishes by the trace-class assumption. 
This estimate holds also for zero eigenvalues; positivity of the retained eigenvalues is used for the density statement above. These estimates quantify source-law truncation using exact spectral coefficients. 
Reconstruction of coefficients from discrete observations and convergence of learned transport maps require additional numerical and optimization analysis.

\subsection{Examples of Orthonormal Bases}

As described in the \textbf{Practical Smoothing via Spectral Augmentation} section, our method projects data onto an orthonormal basis $\{e_k\}$ to inject noise. In our experiments, we specifically utilized the Fourier basis. We detail this basis below, along with other illustrative examples for $H=L^2(-1,1)$ that are compatible with our framework.

\paragraph{1. Fourier Basis.}
On the interval $(-1,1)$, the orthonormal basis is defined as:
$$
e_k(t) = \begin{cases} 
    \frac{1}{\sqrt{2}} & \text{if } k=0, \\
    \cos(m \pi t) & \text{if } k=2m,\ m\geq1, \\
    \sin(m \pi t) & \text{if } k=2m-1,\ m\geq1.
\end{cases}
$$

\paragraph{2. Legendre Polynomials.}
Let $P_n(t)$ be the Legendre polynomials satisfying the recurrence relation $(n+1)P_{n+1}(t)=(2n+1)tP_n(t)-nP_{n-1}(t)$ with $P_0(t)=1, P_1(t)=t$.
To ensure orthonormality on $L^2(-1,1)$, we apply the normalization factor $\sqrt{\frac{2k+1}{2}}$:
$$
e_k(t) = \sqrt{\frac{2k+1}{2}} P_{k}(t), \quad k \in \mathbb{N}_0.
$$
This basis is particularly useful for capturing global trends without enforcing periodicity at the boundaries.

\paragraph{3. Haar Wavelets.}
Let $h_0(t)=\frac{1}{\sqrt{2}}\mathbf{1}_{(-1,1)}(t)$. For scale $j \in \mathbb{N}_0$ and shift $k=0,\dots,2^j-1$, define:
$$
h_{j,k}(t):=2^{j/2}\frac{1}{\sqrt{2}}\left(\mathbf{1}_{\left[ -1 + \frac{2k}{2^j}, -1 + \frac{2(k+0.5)}{2^j} \right)} - \mathbf{1}_{\left[ -1 + \frac{2(k+0.5)}{2^j}, -1 + \frac{2(k+1)}{2^j} \right)}\right).
$$
The collection $\{h_0\} \cup \{h_{j,k}\}_{j,k}$ forms a complete orthonormal system for $L^2(-1,1)$.

\clearpage
\subsection{Algorithm}
\begin{algorithm}[H]
\caption{Training algorithm of HiSNOT with Gaussian Annealing}
\begin{algorithmic}[1]
\Require Source measure $\mu$, target measure $\nu$; Orthonormal basis $\{e_k\}$ and summable sequence $\{\lambda_k\}$; Transport operator $T_\theta$ and potential $V_\phi$; batch size $m$; outer iterations $N$, retained modes $M$, inner-loop steps $K_T$; annealing schedule $\{\varepsilon_k\}_{k=0}^{N-1}$.
\For{$k = 0, 1, \dots, N-1$}
    \State Sample batches $\{x_i\}_{i=1}^m \sim \mu$ and $\{y_i\}_{i=1}^m \sim \nu$.
    \State For each sample $i$, generate $G_i=\sum_{\ell=1}^{M}\sqrt{\lambda_\ell}\xi_{i\ell}e_\ell$, with independent $\xi_{i\ell}\sim\mathcal N(0,1)$.
    \State Apply Spectral Smoothing
    \State $\tilde{x}_i \leftarrow x_i + \varepsilon_k G_i$ for each $i$.
    \State Update $\phi$ to maximize:
    \State \quad $\mathcal{L}_{\phi} = \frac{1}{m} \sum_{i=1}^m \left[ - V_\phi \left( T_\theta(\tilde{x}_i) \right) + V_\phi(y_i) \right]$.
    
    \For{$j = 1, \dots, K_T$}
        \State Sample a batch $\{x_i\}_{i=1}^m \sim \mu$ and generate fresh independent noises $\{G_i\}_{i=1}^m$ as above.
        \State $\tilde{x}_i \leftarrow x_i + \varepsilon_k G_i$ for each $i$.
        \State Update $\theta$ to minimize:
        \State \quad $\mathcal{L}_{\theta} = \frac{1}{m} \sum_{i=1}^m \left[ \frac{1}{2} \| \tilde{x}_i - T_\theta(\tilde{x}_i) \|_H^2 - V_\phi \left( T_\theta(\tilde{x}_i) \right) \right]$.
    \EndFor
\EndFor
\end{algorithmic}
\label{alg:hiSNOT}
\end{algorithm}

\subsection{Synthetic Data Experiments}\label{app:data-2d}
In this section, we describe the experimental setup and the implementation details for the synthetic data experiments
\paragraph{Dataset Description} Throughout this paragraph, let $e_1(t) := \sin(\pi t)$ and $e_2(t) := \sin(2\pi t)$ be the basis functions for $t \in [0, 1]$. We represent each signal as a linear combination of these bases, where $\mathbf{c} = (c_1, c_2) \in \mathbb{R}^2$ and $\mathbf{d} = (d_1, d_2) \in \mathbb{R}^2$ denote the coefficient vectors for the source signal $x(t) = \sum_{i=1}^2 c_i e_i(t)$ and the target signal $y(t) = \sum_{i=1}^2 d_i e_i(t)$, respectively. 
Here $c_i,d_i$ are expansion coefficients in the displayed sine functions; these functions are not normalized in $L^2(0,1)$, so their coefficients need not equal the inner products used for an orthonormal basis. 
We consider the following four scenarios:
\begin{itemize}
    \item \textbf{Perpendicular}: We generate the source coefficients $\mathbf{c} \sim \mu$ as follows: $c_1 \sim U([-1, 1])$ and $c_2 \equiv 0$. Similarly, we sample the target coefficients $\mathbf{d} \sim \nu$ by $d_1 \equiv 0$ and $d_2 \sim U([-1, 1])$.
    \item \textbf{Parallel}: We generate $\mathbf{c} \sim \mu$ as follows: $c_1 \sim U([-1, 1])$, and $c_2 \equiv 0$. Similarly, we sample $\mathbf{d} \sim \nu$ by $d_1 \sim U([-1, 1])$ and $d_2 \equiv 0.5$.
    \item \textbf{One-to-Many}: We generate $\mathbf{c} \sim \mu$ as follows: $c_1 \sim U([-1, 1])$, and $c_2 \equiv 0$. Similarly, we sample $\mathbf{d} \sim \nu$ by $d_1 \sim U([-1, 1])$ and $d_2 \sim \text{Cat}(\{-0.5, 0.5\}, (0.5, 0.5))$.
    \item \textbf{Grid}:  Let $\mathbb{P} := \text{Cat}(\{-0.75, -0.25, 0.25, 0.75\}, (0.25, 0.25, 0.25, 0.25))$. We generate $\mathbf{c} \sim \mu$ by sampling $c_1 \sim U([-1, 1])$, and $c_2 \sim \mathbb{P}$. Similarly, we sample $\mathbf{d} \sim \nu$ by $d_1 \sim \mathbb{P}$ and $d_2 \sim U([-1, 1])$.
\end{itemize}

\paragraph{Training Details}

Across all experiments, we use the same network architecture and training hyperparameters unless otherwise stated.
Both the potential function $V_\phi$ and the transport map $T_\theta$ are parameterized by neural networks composed of two Fourier layers \citep{li2020fourier}.
In particular, the potential network $V_\phi$ additionally includes a final linear integral functional layer \citep{rahman2022generative}.
We train the models with a batch size of $256$ for $5000$ epochs.
We set the OT cost parameter to $\tau = 0.2$ for the Multi-Perpendicular case, and $\tau = 0.001$ for all other cases.
To improve the optimization of the inner-loop in the dual formulation of equation \ref{eq:snot}, we perform $K_T = 5$ updates of $T_\theta$ for every single update of $V_\phi$.
Regarding the regularization, we set $\lambda = 0$ for the Orthogonal case and $\lambda = 0.01$ otherwise.
Here $\lambda$ is a regularization weight, distinct from the covariance eigenvalues $\lambda_k$.

For our experiments, we generate perturbed inputs $\widetilde{x}(t)$ as follows:
$\widetilde{x}(t) = x(t) + \sigma G(t)$,
where $x(t)$ is a clean source signal and $G$ is a finite Fourier-basis noise.
The experimental schedules below use $\sigma$ for the smoothing amplitude denoted
by $\varepsilon$ in the main text. Using $e_k$ for the Fourier basis functions,
the noise expansion is
\begin{equation}
G(t)
=
\sum_{k=1}^{M} \sqrt{\lambda_k}\,\xi_k\,e_k(t),
\qquad
\xi_k \stackrel{\text{i.i.d.}}{\sim} \mathcal{N}(0,1),
\qquad
\lambda_k = \frac{1}{k^{2}},
\label{eq:fourier_basis_noise}
\end{equation}
where we use $M=16$ Fourier modes in all experiments.

We employ a linear noise schedule from the initial noise $\sigma_{\max}=0.5$ to the terminal noise $\sigma_{\min}=0.06$.
The noise level is updated over the first $80\%$ of training epochs and kept fixed afterwards.
At epoch $e$, the noise level is
\begin{equation}
\sigma
=
\max\!\left(\sigma_{\min},\; (1-t_e)\sigma_{\max} + t_e \sigma_{\min}\right),
\qquad
t_e = \frac{e}{0.8E},\;\;
\sigma_{\max}=0.5,\;\;\sigma_{\min}=0.06,
\label{eq:noise_schedule}
\end{equation}
where $E$ is the total number of epochs.

\subsection{Time Series Imputation}\label{app:TSI}
In this section, we describe the experimental setup and implementation details for the time series imputation task.
We conduct experiments on ETTm1, ETTm2, ETTh1, ETTh2, and Exchange datasets \citep{wu2021autoformer}.
For each dataset, we consider missing ratios of $0.5$ and $0.7$, and generate missing patterns following the same setup as \citet{wang2025optimal}, and split the data into training, validation, and test sets with a ratio of $7{:}1{:}2$.

Both the potential function $V_\phi$ and the transport map $T_\theta$ are parameterized using the Neural Fourier Model (NFM) \citep{kim2024neural}.
Unless otherwise specified, we follow \citet{kim2024neural}. We additionally adopt channel-independent modeling as in NFM and use a batch size of $1792$.
We train the model for $40$ epochs with learning rate $2\times 10^{-4}$, dropout $0.1$, regularization weight $\lambda = 0.01$, and $\tau = 50$.

We follow the same input perturbation scheme as in the synthetic time-series experiments \ref{app:data-2d}.
We employ a linear noise schedule from the initial noise $\sigma_{\max}=0.2$ to the terminal noise $\sigma_{\min}=0.01$.
The stated clipped linear schedule is
\begin{equation}
\sigma_i
=
\max\!\left(\sigma_{\min},\; \sigma_{\max}\left(1-\frac{i}{0.8 I }\right)\right),
\qquad
\sigma_{\max} = 0.2,\;\; \sigma_{\min}=0.01,\;\; I=500.
\label{eq:tsi_noise_schedule}
\end{equation}
where $i$ is the iteration index and $I$ denotes the maximum schedule iterations.
For the displayed parameters, this formula reaches $\sigma_{\min}$ at $i=0.76I=380$ and remains there.
In our implementation, $I$ is fixed to $500$ for all datasets.

\begin{sidewaystable}[htbp] 
\centering
\caption{Comparison of Time-Series Imputation Performance for Missing Ratios 0.5 and 0.7. Results are reported as mean$\pm$standard deviation, computed over 5 random seeds. All baseline results (except Ours) are taken from \citet{wang2025optimal}.}
\label{table:TSI_all}
\resizebox{\textwidth}{!}{
\begin{tabular}{l|c|cc|cc|cc|cc|cc}
\toprule
\textbf{Datasets} & & \multicolumn{2}{c|}{\textbf{ETTh1}} & \multicolumn{2}{c|}{\textbf{ETTh2}} & \multicolumn{2}{c|}{\textbf{ETTm1}} & \multicolumn{2}{c|}{\textbf{ETTm2}} & \multicolumn{2}{c}{\textbf{Exchange}} \\
\cmidrule{3-12}
\textbf{Methods} & \textbf{Ratio} & \textbf{MSE} & \textbf{MAE} & \textbf{MSE} & \textbf{MAE} & \textbf{MSE} & \textbf{MAE} & \textbf{MSE} & \textbf{MAE} & \textbf{MSE} & \textbf{MAE} \\
\midrule

% Transformer 
\multirow{2}{*}{Transformer} 
 & 0.5 & 0.249 $\pm$ 0.007 & 0.359 $\pm$ 0.005 & 0.250 $\pm$ 0.007 & 0.360 $\pm$ 0.008 & 0.069 $\pm$ 0.009 & 0.173 $\pm$ 0.005 & 0.032 $\pm$ 0.008 & 0.122 $\pm$ 0.008 & 0.261 $\pm$ 0.009 & 0.215 $\pm$ 0.003 \\
 & 0.7 & 0.426 $\pm$ 0.005 & 0.461 $\pm$ 0.002 & 0.313 $\pm$ 0.004 & 0.405 $\pm$ 0.0 & 0.100 $\pm$ 0.007 & 0.219 $\pm$ 0.008 & 0.080 $\pm$ 0.002 & 0.199 $\pm$ 0.007 & 0.378 $\pm$ 0.006 & 0.362 $\pm$ 0.0 \\
\midrule

% DLinear 
\multirow{2}{*}{DLinear} 
 & 0.5 & 0.148 $\pm$ 0.008 & 0.272 $\pm$ 0.003 & 0.117 $\pm$ 0.004 & 0.240 $\pm$ 0.007 & 0.110 $\pm$ 0.01 & 0.229 $\pm$ 0.001 & 0.089 $\pm$ 0.005 & 0.206 $\pm$ 0.002 & 0.252 $\pm$ 0.01 & 0.208 $\pm$ 0.003 \\
 & 0.7 & 0.225 $\pm$ 0.009 & 0.341 $\pm$ 0.001 & 0.159 $\pm$ 0.002 & 0.283 $\pm$ 0.008 & 0.152 $\pm$ 0.007 & 0.270 $\pm$ 0.007 & 0.125 $\pm$ 0.007 & 0.248 $\pm$ 0.003 & 0.323 $\pm$ 0.007 & 0.267 $\pm$ 0.006 \\
\midrule

% TimesNet 
\multirow{2}{*}{TimesNet} 
 & 0.5 & 0.163 $\pm$ 0.008 & 0.297 $\pm$ 0.005 & 0.126 $\pm$ 0.005 & 0.260 $\pm$ 0.0 & 0.061 $\pm$ 0.002 & 0.179 $\pm$ 0.0 & 0.072 $\pm$ 0.001 & 0.197 $\pm$ 0.01 & 0.343 $\pm$ 0.008 & 0.283 $\pm$ 0.007 \\
 & 0.7 & 0.523 $\pm$ 0.007 & 0.532 $\pm$ 0.01 & 0.168 $\pm$ 0.002 & 0.295 $\pm$ 0.005 & 0.110 $\pm$ 0.005 & 0.238 $\pm$ 0.002 & 0.114 $\pm$ 0.002 & 0.243 $\pm$ 0.005 & 0.401 $\pm$ 0.009 & 0.332 $\pm$ 0.004 \\
\midrule

% FreTS 
\multirow{2}{*}{FreTS} 
 & 0.5 & 0.229 $\pm$ 0.008 & 0.358 $\pm$ 0.003 & 0.141 $\pm$ 0.004 & 0.261 $\pm$ 0.001 & 0.048 $\pm$ 0.001 & 0.149 $\pm$ 0.0 & 0.037 $\pm$ 0.007 & 0.134 $\pm$ 0.005 & 0.256 $\pm$ 0.002 & 0.211 $\pm$ 0.002 \\
 & 0.7 & 0.236 $\pm$ 0.003 & 0.351 $\pm$ 0.006 & 0.176 $\pm$ 0.006 & 0.294 $\pm$ 0.009 & 0.074 $\pm$ 0.008 & 0.187 $\pm$ 0.007 & 0.042 $\pm$ 0.008 & 0.137 $\pm$ 0.002 & 0.201 $\pm$ 0.002 & 0.166 $\pm$ 0.009 \\
\midrule

% PatchTST
\multirow{2}{*}{PatchTST} 
 & 0.5 & 0.153 $\pm$ 0.003 & 0.284 $\pm$ 0.008 & 0.137 $\pm$ 0.003 & 0.269 $\pm$ 0.003 & \textbf{0.046} $\pm$ 0.006 & 0.139 $\pm$ 0.004 & 0.029 $\pm$ 0.006 & 0.116 $\pm$ 0.001 & 0.204 $\pm$ 0.009 & 0.169 $\pm$ 0.007 \\
 & 0.7 & 0.278 $\pm$ 0.001 & 0.392 $\pm$ 0.004 & 0.150 $\pm$ 0.005 & 0.275 $\pm$ 0.008 & \textbf{0.064} $\pm$ 0.009 & 0.166 $\pm$ 0.001 & 0.036 $\pm$ 0.005 & 0.127 $\pm$ 0.002 & 0.250 $\pm$ 0.005 & 0.206 $\pm$ 0.009 \\
\midrule

% SCINet 
\multirow{2}{*}{SCINet} 
 & 0.5 & 0.169 $\pm$ 0.001 & 0.300 $\pm$ 0.005 & 0.176 $\pm$ 0.004 & 0.285 $\pm$ 0.008 & 0.062 $\pm$ 0.001 & 0.173 $\pm$ 0.004 & 0.077 $\pm$ 0.001 & 0.206 $\pm$ 0.0 & 0.280 $\pm$ 0.001 & 0.231 $\pm$ 0.003 \\
 & 0.7 & 0.220 $\pm$ 0.001 & 0.341 $\pm$ 0.006 & 0.168 $\pm$ 0.001 & 0.290 $\pm$ 0.007 & 0.113 $\pm$ 0.005 & 0.234 $\pm$ 0.009 & 0.093 $\pm$ 0.007 & 0.216 $\pm$ 0.004 & 0.371 $\pm$ 0.007 & 0.307 $\pm$ 0.006 \\
\midrule

% iTransformer
\multirow{2}{*}{iTransformer} 
 & 0.5 & 0.168 $\pm$ 0.007 & 0.289 $\pm$ 0.009 & 0.095 $\pm$ 0.004 & 0.210 $\pm$ 0.004 & 0.054 $\pm$ 0.003 & 0.154 $\pm$ 0.002 & 0.029 $\pm$ 0.009 & 0.117 $\pm$ 0.01 & 0.058 $\pm$ 0.002 & 0.045 $\pm$ 0.007 \\
 & 0.7 & 0.244 $\pm$ 0.004 & 0.353 $\pm$ 0.007 & 0.155 $\pm$ 0.007 & 0.263 $\pm$ 0.003 & 0.082 $\pm$ 0.003 & 0.191 $\pm$ 0.008 & 0.057 $\pm$ 0.001 & 0.167 $\pm$ 0.004 & 0.116 $\pm$ 0.008 & 0.095 $\pm$ 0.005 \\
\midrule

% SAITS
\multirow{2}{*}{SAITS} 
 & 0.5 & 0.229 $\pm$ 0.009 & 0.321 $\pm$ 0.008 & 0.184 $\pm$ 0.004 & 0.264 $\pm$ 0.01 & 0.064 $\pm$ 0.001 & 0.169 $\pm$ 0.008 & 0.057 $\pm$ 0.001 & 0.156 $\pm$ 0.008 & 1.007 $\pm$ 0.004 & 0.831 $\pm$ 0.004 \\
 & 0.7 & 0.343 $\pm$ 0.005 & 0.395 $\pm$ 0.004 & 0.287 $\pm$ 0.001 & 0.341 $\pm$ 0.004 & 0.094 $\pm$ 0.002 & 0.207 $\pm$ 0.002 & 0.061 $\pm$ 0.009 & 0.164 $\pm$ 0.009 & 1.006 $\pm$ 0.01 & 0.833 $\pm$ 0.0 \\
\midrule

% CSDI 
\multirow{2}{*}{CSDI} 
 & 0.5 & 0.161 $\pm$ 0.005 & 0.293 $\pm$ 0.001 & 0.096 $\pm$ 0.001 & 0.269 $\pm$ 0.0 & 0.059 $\pm$ 0.009 & 0.139 $\pm$ 0.009 & 0.186 $\pm$ 0.002 & 0.126 $\pm$ 0.008 & 0.125 $\pm$ 0.0 & 0.139 $\pm$ 0.008 \\
 & 0.7 & 0.213 $\pm$ 0.002 & 0.312 $\pm$ 0.01 & 0.117 $\pm$ 0.004 & 0.273 $\pm$ 0.002 & 0.252 $\pm$ 0.004 & 0.273 $\pm$ 0.006 & 0.136 $\pm$ 0.002 & 0.174 $\pm$ 0.005 & 0.143 $\pm$ 0.005 & 0.161 $\pm$ 0.01 \\
\midrule

% Sinkhorn
\multirow{2}{*}{Sinkhorn} 
 & 0.5 & 0.910 $\pm$ 0.008 & 0.669 $\pm$ 0.007 & 0.820 $\pm$ 0.002 & 0.614 $\pm$ 0.01 & 0.950 $\pm$ 0.005 & 0.708 $\pm$ 0.003 & 0.905 $\pm$ 0.001 & 0.678 $\pm$ 0.008 & 0.740 $\pm$ 0.005 & 0.611 $\pm$ 0.001 \\
 & 0.7 & 0.962 $\pm$ 0.009 & 0.701 $\pm$ 0.004 & 0.934 $\pm$ 0.008 & 0.673 $\pm$ 0.001 & 0.979 $\pm$ 0.005 & 0.725 $\pm$ 0.01 & 0.948 $\pm$ 0.002 & 0.725 $\pm$ 0.003 & 0.826 $\pm$ 0.002 & 0.684 $\pm$ 0.005 \\
\midrule

% TDM 
\multirow{2}{*}{TDM} 
 & 0.5 & 0.993 $\pm$ 0.004 & 0.749 $\pm$ 0.008 & 0.978 $\pm$ 0.007 & 0.738 $\pm$ 0.009 & 1.002 $\pm$ 0.003 & 0.754 $\pm$ 0.005 & 0.995 $\pm$ 0.0 & 0.746 $\pm$ 0.004 & 0.957 $\pm$ 0.005 & 0.790 $\pm$ 0.006 \\
 & 0.7 & 0.989 $\pm$ 0.008 & 0.749 $\pm$ 0.008 & 1.017 $\pm$ 0.007 & 0.751 $\pm$ 0.007 & 1.004 $\pm$ 0.003 & 0.757 $\pm$ 0.006 & 1.001 $\pm$ 0.008 & 0.749 $\pm$ 0.0 & 0.981 $\pm$ 0.001 & 0.812 $\pm$ 0.009 \\
\midrule

% PWS-I 
\multirow{2}{*}{PWS-I} 
 & 0.5 & \textbf{0.125} $\pm$ 0.003 & \textbf{0.234} $\pm$ 0.009 & 0.047 $\pm$ 0.006 & 0.145 $\pm$ 0.003 & 0.048 $\pm$ 0.0 & \textbf{0.133} $\pm$ 0.008 & 0.021 $\pm$ 0.008 & 0.096 $\pm$ 0.007 & 0.032 $\pm$ 0.005 & \textbf{0.026} $\pm$ 0.005 \\
 & 0.7 & \textbf{0.194} $\pm$ 0.004 & \textbf{0.289} $\pm$ 0.006 & 0.059 $\pm$ 0.001 & 0.165 $\pm$ 0.003 & 0.066 $\pm$ 0.003 & \textbf{0.158} $\pm$ 0.009 & 0.029 $\pm$ 0.006 & 0.110 $\pm$ 0.002 & 0.039 $\pm$ 0.009 & \textbf{0.033} $\pm$ 0.004 \\
\midrule
\midrule

% our
\multirow{2}{*}{\textbf{Ours}} 
 & 0.5 & 0.166 $\pm$ 0.004 & 0.260 $\pm$ 0.006 & \textbf{0.045} $\pm$ 0.001 & \textbf{0.141} $\pm$ 0.002 & 0.057 $\pm$ 0.002 & 0.141 $\pm$ 0.004 & \textbf{0.019} $\pm$ 0.0 & \textbf{0.086} $\pm$ 0.001 & \textbf{0.005} $\pm$ 0.0 & 0.048 $\pm$ 0.003 \\
 & 0.7 & 0.288 $\pm$ 0.009 & 0.333 $\pm$ 0.006 & \textbf{0.054} $\pm$ 0.001 & \textbf{0.158} $\pm$ 0.001 & 0.076 $\pm$ 0.001 & 0.163 $\pm$ 0.003 & \textbf{0.025} $\pm$ 0.0 & \textbf{0.098} $\pm$ 0.001 & \textbf{0.006} $\pm$ 0.001 & 0.049 $\pm$ 0.004 \\
\bottomrule
\end{tabular}
}
\end{sidewaystable}

\section{Interpretation, Computation, and Scope}
\label{app:revision_context}

\subsection{What Hilbert geometry supplies}
\label{app:hilbert_geometry}
The Hilbert assumption supports the arguments used here, beyond providing an ambient space for functional observations. For $c(x,y)=\frac12\|x-y\|_H^2$, the Riesz identification $H^*\simeq H$ yields $D_xc(x,y)=x-y$. At differentiability points of a source-side truncated transform, this identity determines its contacting target point; Appendix~\ref{26.05.04.01.17} gives the argument without assuming that an untruncated potential is locally Lipschitz on its whole domain. A Hilbert-valued Gaussian also has a positive self-adjoint trace-class covariance, an orthonormal eigenbasis, and canonical projections onto $\ker Q$ and its orthogonal complement. However, its Cameron--Martin space $\operatorname{Ran}(Q^{1/2})$ need not be closed in $H$. This distinction is essential to the qualified smoothing result in Theorem~\ref{25.12.09.14.29}.

Gaussian laws exist on Banach spaces too, but their covariance naturally maps the dual space into the original space; the same orthogonal decomposition and gradient identification are not generally available. Our results apply to structured objects represented in a separable Hilbert space, including paths, fields, signals on a fixed graph, and Hilbert--Schmidt operators. They do not establish an intrinsic theory for arbitrary manifolds, varying graphs, or general Banach, Fr\'echet, or metric spaces.

\subsection{Interpreting source regularity}
\label{app:regularity_examples}
Regularity is a population property. In finite dimensions it is equivalent to absolute continuity with respect to Lebesgue measure. A nondegenerate Gaussian on $H$ is regular by definition, as is its convolution with any probability measure. By contrast, for a finite-mode model $X=\sum_{j=1}^r A_j\phi_j$ in an infinite-dimensional $H$, the source lies in a proper finite-dimensional subspace even if $(A_1,\ldots,A_r)$ has a density. Every proper closed affine subspace is Gaussian null: it is contained in a level set of a nonzero continuous linear functional, whose distribution under any nondegenerate Gaussian has no atoms. Thus this source is non-regular in the ambient $H$, although it may be regular in its intrinsic finite-dimensional space.

Similarly, if $A:H\to Y$ is a nonzero bounded linear operator into a normed space and $Au=f$ has a solution $u_0$, the solution set $u_0+\ker A$ is a proper closed affine subspace. Any law concentrated there is non-regular. This covers bounded linear observation constraints and appropriately posed linear physical constraints; an unbounded PDE operator requires a suitable domain and topology before this reasoning applies. Exact low-rank models provide a rigorous example of singularity. A general nonlinear manifold hypothesis alone is not a proof that a particular observed population is non-regular.

A finite empirical measure is atomic and hence non-regular on a nontrivial $H$, regardless of the unknown population law. Finite observations do not certify population regularity. Smoothing therefore offers a constructive regularization strategy rather than a statistical test for that property.

\subsection{Choosing a covariance}
\label{app:covariance_choice}
One distribution-independent choice is an injective trace-class covariance
\[
 Q_0=\sum_{j\geq1}\tau_j(e_j\otimes e_j),\qquad
 \tau_j>0,\qquad \sum_{j\geq1}\tau_j<\infty.
\]
By Theorem~\ref{25.12.09.14.29}, it regularizes every source. To allocate more variance to target-relevant directions, finite empirical principal subspaces $\widehat U_\mu,\widehat U_\nu$ can guide an optional refinement:
\[
 \widehat W=(\widehat U_\mu+\widehat U_\nu)\cap\widehat U_\mu^\perp,
 \qquad Q=\sigma_{\rm adapt}^2 \pi_{\widehat W}+\sigma_{\rm floor}^2Q_0,
 \quad\sigma_{\rm floor}>0.
\]
Here $\pi_{\widehat W}$ is the orthogonal projection onto $\widehat W$. The finite-rank first term emphasizes selected directions; the injective floor supplies the population guarantee. FPCA neither certifies regularity nor consistently recovers the exact singular geometry under the present analysis. The source determines validity, while the target guides allocation. This is a practical design suggestion, not an additional evaluated algorithm. A finite-mode isotropic floor is nondegenerate only in the retained space $H_M$; it is not an injective covariance on the entire infinite-dimensional $H$.

\subsection{Computational overhead and perturbation size}
\label{app:smoothing_cost}
Let $m$ denote batch size, $M$ the number of retained spectral modes, and $K_T$ the number of transport updates per potential update, as in Algorithm~\ref{alg:hiSNOT}. Sampling and adding one perturbation takes $O(mM)$ coefficient operations and $O(mM)$ temporary storage. With freshly sampled perturbations for one potential update and $K_T$ transport updates, the additional cost per outer iteration is $O((K_T+1)mM)$. This count assumes independent spectral coefficients in the chosen basis. Explicit synthesis on an observation grid adds the cost of the selected basis transform. No additional network, trainable parameter, or optimization loop is introduced, and the number of neural forward/backward passes is unchanged. These are operation counts; no separate wall-clock overhead measurement is reported.

The coupling $(X+\varepsilon G,X)$ gives
\[
 W_2^2(\mu*\mathcal N(0,\varepsilon^2Q),\mu)
 \leq\varepsilon^2\operatorname{Tr}(Q),\qquad\mu\in\mathcal P_2(H).
\]
Larger noise can aid regularization while increasing finite-$\varepsilon$ source perturbation. This bound concerns the source distribution, not learned-map error or downstream prediction accuracy. The ablations in Appendix~\ref{app:smoothing_ablations} assess only the tested smoothing supports and endpoint values. We do not derive a finite-sample optimal covariance, magnitude, or annealing schedule.

\subsection{Static transport, dynamic transport, and generation}
\label{app:dynamic_scope}
Dynamic OT describes paths of measures and velocity fields subject to a continuity equation \citep{benamou2000computational}. Wasserstein gradient flows describe energy-driven evolution, often approximated with successive JKO steps \citep{ambrosio2005gradient,mokrov2021large}. Neural methods can approximate those steps \citep{mokrov2021large}, but the present analysis concerns one static problem between fixed endpoint laws. It proves neither continuity-equation estimates nor energy dissipation or convergence of a JKO scheme. Using HiSNOT as a component of such a method would require additional analysis along the evolving marginals.

For functional generation, a nondegenerate Hilbert-valued Gaussian reference $\gamma$ is regular, and a learned operator aims to satisfy $T_\#\gamma\approx\nu$; samples are obtained as $T(X)$ for $X\sim\gamma$. Identifiability applies to the ideal problem with an admissible optimal potential. 
Appendix~\ref{app:exchange_generation} reports the additional Exchange experiment; it does not establish arbitrary-resolution generation or finite-sample optimization guarantees.

\section{Additional Function-Level Generation Experiment}
\label{app:exchange_generation}

We evaluate function-level generation on the Exchange dataset using a Gaussian
source distribution and Exchange time-series windows as the target distribution.
Training samples are constructed by moving a fixed-length window over the target
series with the specified stride. The window size determines the temporal length
of a sample, while the stride determines the displacement between successive
windows. At a fixed window size, a larger stride reduces the number of training
windows and their overlap.

For evaluation, we generate $1{,}000$ windows from each trained model and randomly
sample $1{,}000$ windows from the target dataset. We report the $2$-Wasserstein
distance $W_2$ between the generated and target samples; lower values indicate a
closer match under this evaluation. Table~\ref{tab:exchange_generation} compares
HiSNOT with the vector SNOT baseline for six window-size/stride settings.

\begin{table}[!htbp]
\centering
\caption{\textbf{Function-level generation on Exchange.}
Evaluation uses $1{,}000$ generated and $1{,}000$ target windows.
The $\pm$ terms are reported dispersion values; their definition and replication
count are unavailable in the provided experiment record.
Bold marks the lower reported central $W_2$ value within each row.}
\label{tab:exchange_generation}
\small
\setlength{\tabcolsep}{7pt}
\begin{tabular}{rrcc}
\toprule
Window size & Stride & HiSNOT ($W_2\downarrow$) & SNOT ($W_2\downarrow$) \\
\midrule
720 & 1 & $9.632 \pm 0.201$ & $\mathbf{9.341} \pm 0.341$ \\
360 & 1 & $\mathbf{5.009} \pm 0.161$ & $5.133 \pm 0.120$ \\
240 & 1 & $\mathbf{3.543} \pm 0.031$ & $3.817 \pm 0.010$ \\
240 & 2 & $\mathbf{3.687} \pm 0.232$ & $4.426 \pm 0.080$ \\
240 & 3 & $\mathbf{3.681} \pm 0.143$ & $4.783 \pm 0.120$ \\
120 & 4 & $\mathbf{2.055} \pm 0.126$ & $2.383 \pm 0.085$ \\
\bottomrule
\end{tabular}
\end{table}

HiSNOT has a lower reported central $W_2$ value in five of the six settings; SNOT
has the lower value for window size $720$ and stride $1$. At fixed window size
$240$, the difference in the reported central values favors HiSNOT at all three
tested strides. These results provide an additional finite-observation functional
generation example. They do not establish statistical significance, an advantage
in every setting, or resolution-independent convergence of the numerical method.

\section{Additional Smoothing Ablations}
\label{app:smoothing_ablations}

We report three controlled ablations of smoothing support and schedule endpoints.
The target-distribution error $D_{\mathrm{target}}$ is the metric defined in
Section~\ref{sec:exp_synthetic_data}; lower values are better.

\subsection{Smoothing Support in a Finite-Dimensional Subspace}
\label{app:smoothing_support_ablation}

Consider
\[
H_5=\operatorname{span}\{e_1,\ldots,e_5\},
\qquad e_k(t)=\sin(k\pi t).
\]
The source has the form $X=A_3e_3+A_4e_4+A_5e_5$, where
$(A_3,A_4,A_5)$ has a density on $\mathbb R^3$, and the target is supported on
$\operatorname{span}\{e_1,e_2\}$. For the diagonal covariances tested here,
smoothing must include both $e_1$ and $e_2$ to make this source absolutely
continuous within $H_5$. This is equivalent to absolute continuity of the source
projection onto the unsmoothed kernel in that kernel's finite-dimensional space.
All regularity labels in Table~\ref{tab:smoothing_support_ablation} refer to
$H_5$, not to a larger infinite-dimensional ambient space.

\begin{table}[!htbp]
\centering
\caption{\textbf{Smoothing-support ablation in $H_5$.}
The source occupies modes $e_3,e_4,e_5$ with a joint coefficient density.
The condition column records whether its unsmoothed projection is absolutely
continuous in the finite-dimensional covariance kernel.
Replication counts and dispersion estimates are unavailable in the provided
experiment record.}
\label{tab:smoothing_support_ablation}
\small
\begin{tabular}{lcc}
\toprule
Active smoothing modes & Kernel condition in $H_5$ & $D_{\mathrm{target}}\downarrow$ \\
\midrule
$e_1$ & Violated & $0.2063$ \\
$e_1,e_2$ & Satisfied & $0.0276$ \\
$e_1,e_2,e_3$ & Satisfied & $0.0300$ \\
$e_1,e_2,e_3,e_4$ & Satisfied & $0.0249$ \\
\bottomrule
\end{tabular}
\end{table}

Omitting $e_2$ gives an error of $0.2063$, whereas the three tested supports
containing both required directions give errors between $0.0249$ and $0.0300$.
Thus, in this controlled finite-dimensional example, the exact minimal support
does not have to be recovered to obtain comparable errors. Additional directions
can nevertheless increase $\operatorname{Tr}(Q)$ and the source perturbation at
positive smoothing amplitude; these results do not imply that arbitrary extra
smoothing is harmless.

\subsection{Initial Smoothing Level}
\label{app:initial_smoothing_ablation}

In the One-to-Many experiment, we vary the initial smoothing level
$\varepsilon_{\max}$ while keeping the remaining annealing protocol fixed.
The setting $\varepsilon_{\max}=0$ denotes the unsmoothed baseline.
Table~\ref{tab:initial_smoothing_ablation} reports all seven tested values.

\begin{table}[!htbp]
\centering
\caption{\textbf{Initial smoothing-level ablation on One-to-Many.}
The remaining annealing protocol is fixed.
Replication counts and dispersion estimates are unavailable in the provided
experiment record.}
\label{tab:initial_smoothing_ablation}
\small
\begin{tabular}{rc}
\toprule
$\varepsilon_{\max}$ & $D_{\mathrm{target}}\downarrow$ \\
\midrule
$0.0$ & $0.288$ \\
$0.2$ & $0.004$ \\
$0.5$ & $0.004$ \\
$0.6$ & $0.003$ \\
$0.8$ & $0.004$ \\
$1.0$ & $0.004$ \\
$2.0$ & $0.009$ \\
\bottomrule
\end{tabular}
\end{table}

All tested positive initial levels yield $D_{\mathrm{target}}<0.01$, compared
with $0.288$ without smoothing. This supports robustness to the initial amplitude
over the tested range in this experiment.

\subsection{Terminal Smoothing Level}
\label{app:terminal_smoothing_ablation}

We next fix $\varepsilon_{\max}=0.5$ and vary the terminal level
$\varepsilon_{\min}$ at which the annealing schedule ends.
Table~\ref{tab:terminal_smoothing_ablation} reports all six tested values.

\begin{table}[!htbp]
\centering
\caption{\textbf{Terminal smoothing-level ablation.}
The initial level is fixed at $\varepsilon_{\max}=0.5$.
Replication counts and dispersion estimates are unavailable in the provided
experiment record.}
\label{tab:terminal_smoothing_ablation}
\small
\begin{tabular}{rc}
\toprule
$\varepsilon_{\min}$ & $D_{\mathrm{target}}\downarrow$ \\
\midrule
$0.02$ & $0.008$ \\
$0.06$ & $0.003$ \\
$0.10$ & $0.004$ \\
$0.14$ & $0.004$ \\
$0.18$ & $0.004$ \\
$0.20$ & $0.004$ \\
\bottomrule
\end{tabular}
\end{table}

All tested terminal levels yield target-distribution error below $0.01$.
The endpoint ablations support robustness within the tested ranges; they
neither identify an optimal finite-sample schedule nor eliminate tuning on
other datasets. The population bound in Appendix~\ref{app:smoothing_cost}
controls source perturbation, not learned-map or downstream prediction error.

%%%%%%%%%%%%%%%%%%%%%%%%%%%%%%%%%%%%%%%%%%%%%%%%%%%%%%%%%%%%

\clearpage
\section*{NeurIPS Paper Checklist}

\begin{enumerate}

\item {\bf Claims}
    \item[] Question: Do the main claims made in the abstract and introduction accurately reflect the paper's contributions and scope?
    \item[] Answer: \answerYes{} % Replace by \answerYes{}, \answerNo{}, or \answerNA{}.
    \item[] Justification: The abstract and introduction precisely state the theoretical contributions that correspond to Theorems \ref{thm:uniqueness}, \ref{25.12.09.14.29}, and \ref{thm:convergence} and experimental claims align with the results in Sections \ref{sec:experiments}.

\item {\bf Limitations}
    \item[] Question: Does the paper discuss the limitations of the work performed by the authors?
    \item[] Answer: \answerYes{} % Replace by \answerYes{}, \answerNo{}, or \answerNA{}.
    \item[] Justification: The conclusion explicitly notes that consistency is established only at the plan level (not map level) and that learning the smoothing covariance operator $Q$  from finite samples remains open.

\item {\bf Theory assumptions and proofs}
    \item[] Question: For each theoretical result, does the paper provide the full set of assumptions and a complete (and correct) proof?
    \item[] Answer: \answerYes{} % Replace by \answerYes{}, \answerNo{}, or \answerNA{}.
    \item[] Justification: All theorems (\ref{thm:uniqueness}, \ref{25.12.09.14.29}, \ref{thm:convergence}) and supporting propositions are stated with explicit assumptions  by complete proofs in Appendix \ref{26.05.04.01.17},  \ref{26.05.04.00.24}, and \ref{26.01.29.13.52}. 

    \item {\bf Experimental result reproducibility}
    \item[] Question: Does the paper fully disclose all the information needed to reproduce the main experimental results of the paper to the extent that it affects the main claims and/or conclusions of the paper (regardless of whether the code and data are provided or not)?
    \item[] Answer: \answerYes{} % Replace by \answerYes{}, \answerNo{}, or \answerNA{}.
    \item[] Justification: Appendix \ref{app:implementation_detail} provides full implementation details including dataset, network architectures, batch sizes, learning rates, noise schedules, and hyperparameters.

\item {\bf Open access to data and code}
    \item[] Question: Does the paper provide open access to the data and code, with sufficient instructions to faithfully reproduce the main experimental results, as described in supplemental material?
    \item[] Answer: \answerYes{} % Replace by \answerYes{}, \answerNo{}, or \answerNA{}.
    \item[] Justification: Anonymized code is provided as supplementary material. 

\item {\bf Experimental setting/details}
    \item[] Question: Does the paper specify all the training and test details (e.g., data splits, hyperparameters, how they were chosen, type of optimizer) necessary to understand the results?
    \item[] Answer: \answerYes{} % Replace by \answerYes{}, \answerNo{}, or \answerNA{}.
    \item[] Justification: Training details, optimizer settings, annealing schedules (⁡$\sigma_{\max}$, $\sigma_{\min}$), data splits (7:1:2), and missing ratios (0.5, 0.7) are fully specified in Appendix \ref{app:implementation_detail}.

\item {\bf Experiment statistical significance}
    \item[] Question: Does the paper report error bars suitably and correctly defined or other appropriate information about the statistical significance of the experiments?
    \item[] Answer:  \answerYes{} % Replace by \answerYes{}, \answerNo{}, or \answerNA{}.
    \item[] Justification: Yes. As reported in Table~\ref{table:TSI_all}, all results are reported as mean$\pm$standard deviation over 5 random seeds.

\item {\bf Experiments compute resources}
    \item[] Question: For each experiment, does the paper provide sufficient information on the computer resources (type of compute workers, memory, time of execution) needed to reproduce the experiments?
    \item[] Answer: \answerYes{} % Replace by \answerYes{}, \answerNo{}, or \answerNA{}.
    \item[] Justification:  All experiments were conducted on a single NVIDIA A100 80GB GPU, as noted in Appendix~\ref{app:implementation_detail}.
    
\item {\bf Code of ethics}
    \item[] Question: Does the research conducted in the paper conform, in every respect, with the NeurIPS Code of Ethics \url{https://neurips.cc/public/EthicsGuidelines}?
    \item[] Answer: \answerYes{}  % Replace by \answerYes{}, \answerNo{}, or \answerNA{}.
    \item[] Justification: The research conforms with the NeurIPS Code of Ethics. It involves no human subjects, uses only publicly available time-series benchmarks (ETT, Exchange) \citep{wu2021autoformer}.

\item {\bf Broader impacts}
    \item[] Question: Does the paper discuss both potential positive societal impacts and negative societal impacts of the work performed?
    \item[] Answer: \answerNA{} % Replace by \answerYes{}, \answerNo{}, or \answerNA{}.
    \item[] Justification: This work is foundational theoretical and methodological research on optimal transport plan in Hilbert spaces with no direct path to specific harmful applications. Guidelines:
    \begin{itemize}
        \item The answer \answerNA{} means that there is no societal impact of the work performed.
        \item If the authors answer \answerNA{} or \answerNo, they should explain why their work has no societal impact or why the paper does not address societal impact.
        \item Examples of negative societal impacts include potential malicious or unintended uses (e.g., disinformation, generating fake profiles, surveillance), fairness considerations (e.g., deployment of technologies that could make decisions that unfairly impact specific groups), privacy considerations, and security considerations.
        \item The conference expects that many papers will be foundational research and not tied to particular applications, let alone deployments. However, if there is a direct path to any negative applications, the authors should point it out. For example, it is legitimate to point out that an improvement in the quality of generative models could be used to generate Deepfakes for disinformation. On the other hand, it is not needed to point out that a generic algorithm for optimizing neural networks could enable people to train models that generate Deepfakes faster.
        \item The authors should consider possible harms that could arise when the technology is being used as intended and functioning correctly, harms that could arise when the technology is being used as intended but gives incorrect results, and harms following from (intentional or unintentional) misuse of the technology.
        \item If there are negative societal impacts, the authors could also discuss possible mitigation strategies (e.g., gated release of models, providing defenses in addition to attacks, mechanisms for monitoring misuse, mechanisms to monitor how a system learns from feedback over time, improving the efficiency and accessibility of ML).
    \end{itemize}
    
\item {\bf Safeguards}
    \item[] Question: Does the paper describe safeguards that have been put in place for responsible release of data or models that have a high risk for misuse (e.g., pre-trained language models, image generators, or scraped datasets)?
    \item[] Answer: \answerNA{}  % Replace by \answerYes{}, \answerNo{}, or \answerNA{}.
    \item[] Justification:  The paper does not release pretrained models, generative systems, or scraped datasets that would pose misuse risks.

\item {\bf Licenses for existing assets}
    \item[] Question: Are the creators or original owners of assets (e.g., code, data, models), used in the paper, properly credited and are the license and terms of use explicitly mentioned and properly respected?
    \item[] Answer: \answerYes{}  % Replace by \answerYes{}, \answerNo{}, or \answerNA{}.
    \item[] Justification: All datasets (ETT,Exchange \citep{wu2021autoformer}) and  baseline results from  \citet{wang2025optimal} are properly cited. The neural network architectures used in our experiments are also properly cited in Appendix \ref{app:implementation_detail}. These include Fourier Neural Operator \cite{li2020fourier}, the integral functional layer from \cite{rahman2022generative}, and Neural Fourier Model \cite{kim2024neural}.

\item {\bf New assets}
    \item[] Question: Are new assets introduced in the paper well documented and is the documentation provided alongside the assets?
    \item[] Answer: \answerYes{}  % Replace by \answerYes{}, \answerNo{}, or \answerNA{}.
    \item[] Justification: An anonymized implementation of HiSNOT is provided as supplementary material, including a README with dependencies, training commands.

\item {\bf Crowdsourcing and research with human subjects}
    \item[] Question: For crowdsourcing experiments and research with human subjects, does the paper include the full text of instructions given to participants and screenshots, if applicable, as well as details about compensation (if any)? 
    \item[] Answer: \answerNA{}  % Replace by \answerYes{}, \answerNo{}, or \answerNA{}.
    \item[] Justification: The paper does not involve crowdsourcing or human subjects research. 

\item {\bf Institutional review board (IRB) approvals or equivalent for research with human subjects}
    \item[] Question: Does the paper describe potential risks incurred by study participants, whether such risks were disclosed to the subjects, and whether Institutional Review Board (IRB) approvals (or an equivalent approval/review based on the requirements of your country or institution) were obtained?
    \item[] Answer: \answerNA{} % Replace by \answerYes{}, \answerNo{}, or \answerNA{}.
    \item[] Justification: The paper does not involve human subjects research and therefore does not require IRB approval.

\item {\bf Declaration of LLM usage}
    \item[] Question: Does the paper describe the usage of LLMs if it is an important, original, or non-standard component of the core methods in this research? Note that if the LLM is used only for writing, editing, or formatting purposes and does \emph{not} impact the core methodology, scientific rigor, or originality of the research, declaration is not required.
    %this research? 
    \item[] Answer: \answerNA{} % Replace by \answerYes{}, \answerNo{}, or \answerNA{}.
    \item[] Justification: LLMs were not used as a component of the core methodology, theoretical analysis, or experimental design.

\end{enumerate}

\end{document}